%% file: neurips_2022.tex
\documentclass{article}

\usepackage[final]{neurips_2022}

\usepackage[utf8]{inputenc} 
\usepackage[T1]{fontenc}    
\usepackage{hyperref}       
\usepackage{url}            
\usepackage{booktabs}       
\usepackage{amsfonts}       
\usepackage{nicefrac}       
\usepackage{microtype}      
\usepackage{xcolor}         

\title{Symmetry Teleportation for Accelerated Optimization}

\author{
  Bo Zhao \\
  University of California, San Diego \\
  \texttt{bozhao@ucsd.edu} \\
  \And
  Nima Dehmamy \\
  IBM Research\\
  \texttt{nima.dehmamy@ibm.com} \\
  \And
  Robin Walters \\
  Northeastern University \\
  \texttt{r.walters@northeastern.edu} \\
  \And
  Rose Yu \\
  University of California, San Diego \\
  \texttt{roseyu@ucsd.edu} \\
}

\input{math_commands}

\begin{document}

\maketitle

\input{secs/abstract}
\input{secs/intro}
\input{secs/relate}
\input{secs/algorithm}

\input{secs/group_action}
\input{secs/theory}


\input{secs/experiments}
\input{secs/conclusion}
\input{secs/acknowledgement}

\bibliography{references}
\bibliographystyle{plainnat}

\out{
\section*{Checklist}


\begin{enumerate}

\item For all authors...
\begin{enumerate}
  \item Do the main claims made in the abstract and introduction accurately reflect the paper's contributions and scope?
    \answerYes{}
  \item Did you describe the limitations of your work?
    \answerYes{}
  \item Did you discuss any potential negative societal impacts of your work?
    \answerNo{}
  \item Have you read the ethics review guidelines and ensured that your paper conforms to them?
    \answerYes{}
\end{enumerate}

\item If you are including theoretical results...
\begin{enumerate}
  \item Did you state the full set of assumptions of all theoretical results?
    \answerYes{}
        \item Did you include complete proofs of all theoretical results?
    \answerYes{}
\end{enumerate}

\item If you ran experiments...
\begin{enumerate}
  \item Did you include the code, data, and instructions needed to reproduce the main experimental results (either in the supplemental material or as a URL)?
    \answerYes{}
  \item Did you specify all the training details (e.g., data splits, hyperparameters, how they were chosen)?
    \answerYes{}
        \item Did you report error bars (e.g., with respect to the random seed after running experiments multiple times)?
    \answerYes{}
        \item Did you include the total amount of compute and the type of resources used (e.g., type of GPUs, internal cluster, or cloud provider)?
    \answerYes{}
\end{enumerate}

\item If you are using existing assets (e.g., code, data, models) or curating/releasing new assets...
\begin{enumerate}
  \item If your work uses existing assets, did you cite the creators?
    \answerNA{}
  \item Did you mention the license of the assets?
    \answerNA{}
  \item Did you include any new assets either in the supplemental material or as a URL?
    \answerNA{}
  \item Did you discuss whether and how consent was obtained from people whose data you're using/curating?
    \answerNA{}
  \item Did you discuss whether the data you are using/curating contains personally identifiable information or offensive content?
    \answerNA{}
\end{enumerate}

\item If you used crowdsourcing or conducted research with human subjects...
\begin{enumerate}
  \item Did you include the full text of instructions given to participants and screenshots, if applicable?
    \answerNA{}
  \item Did you describe any potential participant risks, with links to Institutional Review Board (IRB) approvals, if applicable?
    \answerNA{}
  \item Did you include the estimated hourly wage paid to participants and the total amount spent on participant compensation?
    \answerNA{}
\end{enumerate}

\end{enumerate} }


\appendix

\input{secs/appendix}

\end{document}

%% file: math_commands.tex
\usepackage{amsthm}
\usepackage{subfigure}

\usepackage{amsmath,amssymb,amsfonts,amscd}
\usepackage{mathrsfs}  

\usepackage[all,cmtip]{xy}
\usepackage{enumerate}
\usepackage{latexsym}
\usepackage{natbib}
\setcitestyle{authoryear,open={(},close={)}}
\usepackage{xcolor}
\usepackage{fullpage}
\usepackage{url}
\usepackage{hyperref}
\usepackage{marvosym}
\usepackage{wrapfig}
\usepackage{graphicx}
\usepackage{bm}

\newcommand{\R}{{\mathbb R}}   

\newcommand{\eat}[1]{}
\newcommand{\out}{\eat}

\newtheorem{theorem}{Theorem}[section]
\newtheorem{proposition}[theorem]{Proposition}
\newtheorem{assumption}[theorem]{Assumption}

\newtheorem{lemma}[theorem]{Lemma}

\newtheorem{corollary}[theorem]{Corollary}

\newcommand{\ryedit}[1]{{\color{magenta} #1}}
\newcommand{\ry}[1]{\ryedit{[RY: #1]}}
\newcommand{\nd}[1]{{\color{orange}[ND: #1]} }
\newcommand{\bz}[1]{{\color{blue} BZ: #1}}

\usepackage[ruled, lined, linesnumbered, commentsnumbered, longend]{algorithm2e}

\makeatletter

\newcommand{\Rmnum}[1]{\expandafter\@slowromancap\romannumeral #1@}
\makeatother

\def\vv{{\bm{v}}}
\def\vw{{\bm{w}}}

\def\eps{{\varepsilon}}
\DeclareMathOperator{\Tr}{Tr}
\newcommand{\ro}{\partial}

\renewcommand{\L}{\mathcal{L}}
\newcommand{\pa}[1]{\left(#1\right)}
\newcommand{\br}[1]{\left[#1\right]}
\newcommand{\bk}[1]{\left<#1\right>}
\newcommand{\grad}{\nabla}

\newtheorem{manualtheoreminner}{Proposition}
\newenvironment{manualproposition}[1]{%
  \renewcommand\themanualtheoreminner{#1}%
  \manualtheoreminner
}{\endmanualtheoreminner}

\DeclareMathOperator*{\argmax}{arg\,max}
\DeclareMathOperator*{\argmin}{arg\,min}

%% file: secs/abstract.tex
\begin{abstract}
Existing gradient-based optimization methods update parameters locally, in a direction that minimizes the loss function. 
We study a different approach, symmetry teleportation, that allows parameters to travel a large distance on the loss level set, in order to improve the convergence speed in subsequent steps. 
Teleportation exploits symmetries in the loss landscape of optimization problems. 
We derive loss-invariant group actions for test functions in optimization and multi-layer neural networks, and prove a necessary condition for teleportation to improve convergence rate. 
We also show that our algorithm is closely related to second order methods. 
Experimentally, we show that teleportation improves the convergence speed of gradient descent and AdaGrad for several optimization problems including test functions, multi-layer regressions, and MNIST classification. 
Our code is available at \url{https://github.com/Rose-STL-Lab/Symmetry-Teleportation}.



\end{abstract}

%% file: secs/intro.tex
\section{Introduction}
Consider the optimization problem of finding $\text{argmin}_\vw \L(\vw)$, where $\L$ is a loss function and $\vw $ 
are the parameters. 
While finding global minima of $\L(\vw)$ is hard for non-convex problems, we can use gradient descent (GD) to find local minima. 
In GD we apply the following update rule at every step:
\begin{align}
    \vw_{t+1} \leftarrow \vw_t - \eta \grad \L,
\end{align}
where $\eta$ is the learning rate. Gradient descent is a first-order method that uses only gradient information. 
It is easy to compute but suffers from slow convergence. 
Second-order methods such as Newton's method use the second derivative to account for the landscape geometry. 
These methods enjoy faster convergence, but calculating the second derivative (Hessian) can be computationally expensive over high dimensional spaces \citep{hazan2019lecture}. 


\begin{figure}[ht!]
\centering
\subfigure{\label{fig:GD-first}\includegraphics[width=0.32\textwidth]{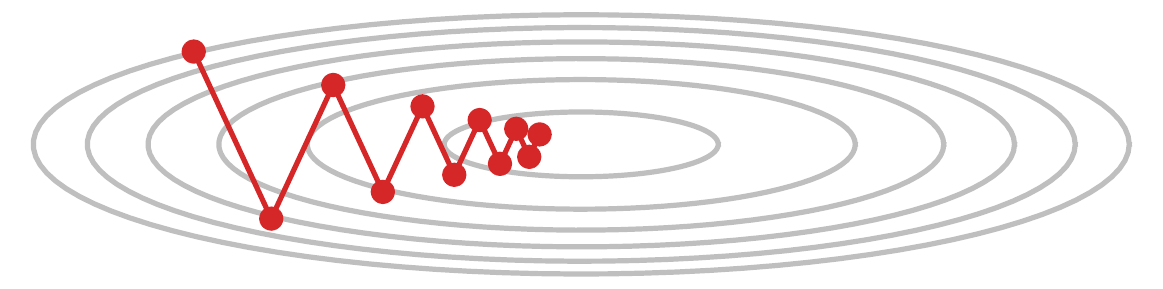}}
\subfigure{\label{fig:GD-second}\includegraphics[width=0.32\textwidth]{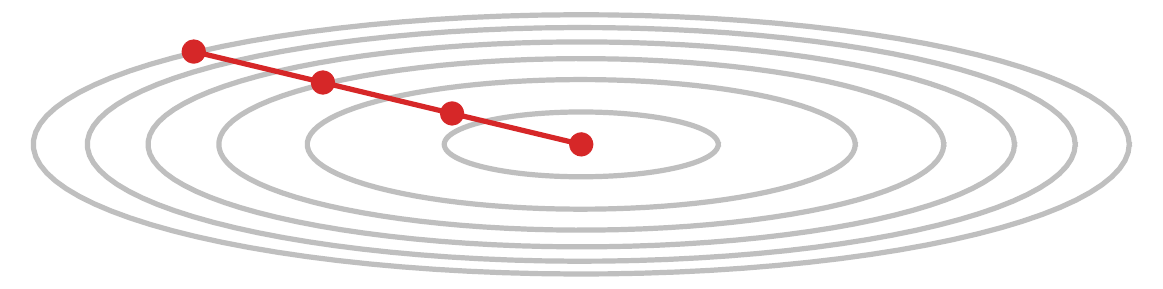}}
\subfigure{\label{fig:GD-proposed}\includegraphics[width=0.32\textwidth]{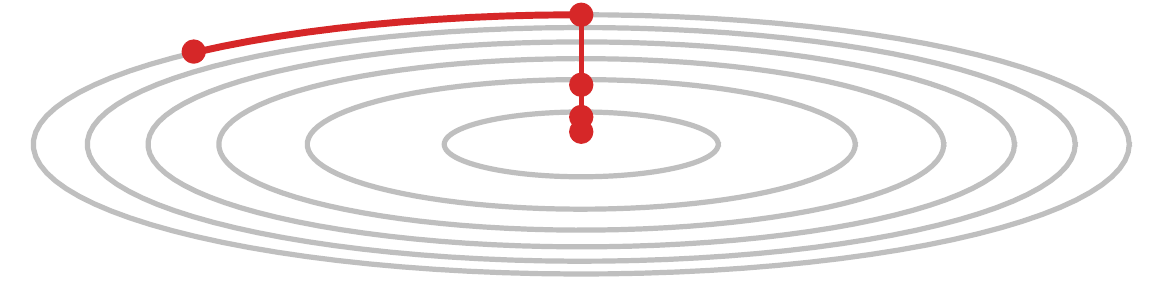}}
\caption{Left to right: gradient descent, second-order methods, proposed method.}
\label{fig:sales_pic}
\end{figure}

In this paper, we propose a new optimization method based on parameter space symmetries, which are group actions on the parameter space that leave the loss unchanged. Our algorithm takes advantage of higher-order landscape geometry but uses only gradient information in most steps, thus avoiding the computational cost of second-order methods.

Specifically, we look beyond local optimization and ask: \textit{what if we allow the parameters to make a big jump once in a while?} 
As shown in Figure \ref{fig:sales_pic}, during optimization, we teleport the parameters to another point with steeper gradients using a loss-invariant symmetry transformation.
After teleportation, the loss stays the same, but the gradient and hence the rate of loss decay changes. The increased magnitude of the gradient can reduce the number of steps required to converge, leading to acceleration of the gradient descent.


The locality of gradient descent is reflected in its formulation in terms of a proximal term; our method circumvents this locality by teleporting to new locations with the same loss.
A step of GD is 
equivalent to the following proximal mapping \citep{combettes2011proximal}. Let $\bk{\cdot, \cdot}$ denote the inner product, we have
\begin{align}
    \vw_{t+1} = \text{argmin}_{\vw} \left\{\eta \langle (\grad \L)|_{\vw_t}, \vw \rangle + \frac{1}{2}\|\vw - \vw_t\|_2^2 \right\}.
\end{align}
The term $\frac{1}{2}\|\vw - \vw_t\|^2$ is the proximal term that keeps $\vw_{t+1}$ close to $\vw_t$. 
Adaptive gradient methods define the proximal term using the Mahalanobis distance $\|\vw - \vw_t\|_{G^{-1/2}}^2$ to account for landscape geometry. 
For example, in AdaGrad, $G$ is the sum of the outer product of gradients \citep{duchi2011adaptive}. 
Our proposed teleportation technique relaxes the requirement from the proximal term. 
We teleport to points on the same level set of $\L(\vw_t)$, but allow $\vw$  to be far from $\vw_t$ in Euclidean distance. 
\out{
For example, the update rule in AdaGrad replaces the Euclidean distance by the Mahalanobis distance $\|\vw - \vw_t\|_{G^{-1/2}}^2$ \cite{duchi2011adaptive}.
In our case,  we define the  $NS$ as the quotient space \ry{$NS =V/S$ ?? } 
\nd{This requires more careful definition of $S$ and whether $NS$ is a metric space. It may be too much to discuss here.}
\begin{align}
    \vw_{t+1} = \text{argmin}_{\vw \in NS } \left\{\eta \langle (\grad \L)|_{\vw_t}, \vw \rangle + \frac{1}{2}\|\vw - \vw_t\|_2^2 \right\}.
\end{align}
}

In summary, our main contributions are:
\begin{itemize}
    \item We propose \textit{symmetry teleportation}, an accelerated gradient-based optimization algorithm that exploits symmetries in the loss landscape.
    \item We derive loss-invariant group actions of test functions and multi-layer neural networks.
    \item We provide necessary conditions and examples of when teleportation improves convergence rate of the current and subsequent steps.
    \item We show that our algorithm is closely related to second-order methods. 
    \item Experimentally, we show that teleportation improves the convergence speed of gradient descent and AdaGrad in various optimization problems. 
\end{itemize}

%% file: secs/relate.tex
\section{Related Work}

Continuous parameter space  symmetries have been identified in neural networks with homogeneous \citep{badrinarayanan2015symmetry,du2018algorithmic} and radial activation functions \citep{ganev2021qr}.
The effect of symmetry transformations on gradients has been examined for translation, scale, and rescale symmetries in \cite{kunin2021neural}.  We contribute to this line of work by deriving loss-invariant group actions in multi-layer neural networks with invertible activation functions.

Several works exploit symmetry to facilitate optimization. For example, Path-SGD \citep{neyshabur2015path-sgd} improves optimization in ReLU networks by path regularization. $\mathcal{G}$-SGD \citep{meng2019mathcal} performs weight updates in a scale-invariant space, which also exploits the symmetry in ReLU networks.
\cite{van2017l2} analyzes the effect of L2 regularization in batch normalization which gives scale-invariant functions. 
\citet{bamler2018improving} transforms parameters in the symmetry orbits to address slow movement along directions of weakly broken symmetry and find the optimal $g \in G$ that minimizes $\L(g \cdot \vw)$. In comparison, we search within $G$-orbits for  points which maximize $|d\L(g\cdot \vw)/dt| = \|\nabla \L (g\cdot\vw)\|^2 $. 

Natural gradient \citep{amari1998natural}, adaptive gradient methods \citep{duchi2011adaptive, kingma2015adam}, and their approximations \citep{martens2018kfac, gupta2018shampoo} improve the direction of parameter updates instead of directly transforming parameters. If the group acts transitively on the level set, and we teleport to the point with maximum gradient, then our update direction is the same as that in Newton's method. 
We prove that our algorithm is connected to second-order optimization methods and show that teleportation can be used to improve these methods empirically. 

The concept of neural teleportation was first explored using quiver representation theory \citep{armenta2021representation, armenta2020neural}. These works provide a way to explore level curves of the loss of neural networks and show that random teleportation speeds up gradient descent experimentally and theoretically. Our algorithm improves neural teleportation by searching for teleportation destinations that lead to the largest improvement in the magnitude of gradient. 

Several works study how parameter initialization affects convergence rate \citep{saxe2014exact, tarmoun2021understanding, min2021explicit}. If we apply a group transformation only at initialization, our method is similar to that of \cite{tarmoun2021understanding}. We do not guarantee that the transformed parameters lead to faster convergence rate throughout the entire training. However, we accelerate convergence at least for a short time after initialization. Additionally, our method is not restricted to initialization and can be applied at any time during training.

Most contemporary neural networks are overparametrized.  While this has been shown to improve generalization \citep{belkin2019reconciling}, an important question is how overparametrization affects optimization. A series of works starting from \cite{arora2018optimization} show that overparametrization resulting from the depth of a neural network accelerates convergence. Another view is that the symmetry created by overparametrization poses constraints on trajectories in the form of conserved quantities \citep{gluch2021noether}. Additionally, the symmetry generates additional trajectories. When the trajectories created by overparametrization are equivalent, model compression by removing symmetry reduces training time \citep{ganev2021qr}. However, when trajectories are not equivalent, gradient flows on some paths are faster than others. We search for the better paths created by overparametrization. 



%% file: secs/algorithm.tex
\section{Symmetry Teleportation}
We propose \textit{symmetry teleportation}, an accelerated gradient-based optimization algorithm that exploits symmetries in the loss landscape. Symmetry teleportation searches for the best gradient descent trajectory by teleporting parameters to a point with larger gradients using a group action. The resulting algorithm requires only  gradient computations but is able to account for the global landscape geometry, leading to faster loss decay. 

Let the group $G$ be a set of symmetries which leave the loss function $\L$ invariant: $\L(g \cdot(\vw, X)) = \L(\vw, X)$, where $g \in G$. We perform gradient descent for $t_{max}$ steps. We define an index set $K \subseteq \{0, 1, 2, ..., t_{max}-1\}$ as a teleportation schedule. At epochs that are in the schedule, we transform parameters using group element $g \in G$ to the location where the gradient is largest, then continue with gradient descent. Algorithm \ref{alg:teleport} describes the details of this procedure.  Note that the loss does not change after teleportation (Line 2-5) since $\L$ is $G$-invariant. 

\vspace{-3mm}
\begin{algorithm}
    \SetKwInput{Input}{Input}
    \SetKwInput{Output}{Output}

    \Input{Loss function $\L(\vw)$, learning rate $\eta$, number of epochs $t_{max}$, initialized parameters $\vw_0$, symmetry group $G$, teleportation schedule $K$.}
    \Output{$\vw_{t_{max}}$.}

    \For{$t \leftarrow 0$ \KwTo $t_{max}-1$}{
        \If{$t \in K$}{
            $g \leftarrow \text{argmax}_{g \in G} \| (\grad \L)|_{g \cdot \vw_t} \|^2$ \\
            $\vw_t \leftarrow g \cdot \vw_t$
        }
        $\vw_{t+1} \leftarrow \vw_t - \eta (\grad \L)|_{\vw_t}$
    }

    \KwRet{$\vw_{t_{max}}$}
    \caption{Symmetry Teleportation}
    \label{alg:teleport}
\end{algorithm}
\vspace{-3mm}
Algorithm \ref{alg:teleport} can be generalized to apply teleportation to stochastic gradient descent (Appendix \ref{appendix:algorithm-sgd}). We discuss some design choices below and provide detailed analysis in the next two sections. 

When the action of $G$ is continuous, teleportation can be implemented by parameterizing and performing gradient ascent on $g$. 
For example, the $\mathrm{SO}(2)$ group can be parameterized by the rotation angle $\theta$. 
Small transformations $g\in\mathrm{GL}_d(\R)$ (general linear group) can be parameterized as $g \approx I+\eps T$ where $\eps\ll 1$ and $T$ are arbitrary $d\times d$ matrices. 
For discrete groups, search algorithms or random sampling can be used to find a group element that improves the magnitude of the gradient.

Although $g \cdot (\vw, X)$ does not change $X$ in the cases we consider in this paper, Algorithm \ref{alg:teleport} can be extended to allow transformations on both parameters and data. The group actions on data during teleportations can be precomposed and applied to the input data at inference time. The  algorithm extension that allows transformation on the input can be found in Appendix \ref{appendix:algorithm-data-transform}.

The teleportation schedule $K$ is a hyperparameter that determines when to perform teleportation. We define $K$ as a set to allow flexible teleportation schedules, such as with non-fixed frequencies or teleporting only at the earlier epochs.

%% file: secs/group_action.tex
\section{Symmetry Groups of Certain Optimization Problems}
We give a few practical examples to demonstrate how teleportation can be used to accelerate optimization. 
Specifically, we first consider two test functions 
which are often used to evaluate optimization algorithms \citep{back1996evolutionary}. 
We then derive the symmetries of   multi-layer  neural networks. 

\subsection{Test functions}
\paragraph{Rosenbrock function.}
The Rosenbrock function originally introduced by  \cite{rosenbrock1960automatic} has a characteristic global minimum that is inside a long, narrow, parabolicly-shaped flat valley. Finding the valley is easy but reaching the minimum is difficult. On a 2-dimensional space, the Rosenbrock function has the following form: 
\begin{align}
    \L_{r}(x_1, x_2) = 100(x_1^2 - x_2)^2 + (x_1 - 1)^2.
\label{eqn:rosenbrock}
\end{align}

\paragraph{Booth function.}
The Booth function \citep{jamil2013literature} is also defined on $\mathbb{R}^2$ and has one global minimum at $(1,3)$ where the function evaluates to $0$:
\begin{align}
    \L_{b}(x_1, x_2) = (x_1 + 2x_2 - 7)^2 + (2x_1 + x_2 - 5)^2.
\label{eqn:booth}
\end{align}
The following proposition identifies the symmetry of these two test functions. 
\begin{proposition}
The Rosenbrock and Booth functions have rotational symmetry. In other words, there exist action maps $a_{r}, a_{b}: \mathrm{SO}(2) \times \R^2 \xrightarrow[]{} \R^2$, such that for all $g \in \mathrm{SO}(2)$, 
\[\L_{r} (x_1, x_2) = \L_{r} (a_{r}(g, [x_1, x_2])) \quad \text{and} \quad  \L_{b} (x_1, x_2) = \L_{b} (a_{b}(g, [x_1, x_2])).\]
\end{proposition}
The exact forms of the action maps are deferred to Appendix \ref{app:rosenbrock_symmetry} and \ref{app:booth_symmetry}.
During the teleportation step, our goal is to maximize the gradient within a level set of the loss:
$
    \max_{ g \in \mathrm{SO}(2)} \| \frac{d \L(g \cdot (x_1, x_2))}{dt} \| 
$.

\out{
\subsection{Two-layer neural network}
\bz{To be removed.}

Given a two-layer neural network with Leaky ReLU as the activation function.  $X \in \R^{d \times n}, Y \in \R^{d \times n}$ are the input and output. The objective function for regression is
\begin{align}
    \L(X, Y, U, V) = \| Y - U \sigma(VX) \|_2^2
    \label{eq:two-layer-obj}
\end{align}
where $U \in \R^{d \times k}, V \in \R^{k\times d}$, and $\sigma: \R^{k\times d} \xrightarrow[]{} \R^{k\times d}$ is the element-wise Leaky ReLU. As shown in section 2 in this document, $\L$ has a $GL_8(\R)$ symmetry. For $g \in G$, the following action leaves $\L$ unchanged:
\begin{align}
    g \cdot (U, V) = (g \cdot U, g \cdot V) = (Ug^{-1}, \sigma^{-1} (g \sigma (VX)) X^{-1})
    \label{eq:two-layer-group-action}
\end{align}

During the symmetry transform step, our goal is to maximize the learning rate within a level set of the loss:
$
    \max_{g \in GL(\R)} \frac{d \L(g \cdot (U, V))}{dt} 
    \label{eq:two-layer-group-obj}
$
One approach is to find an element $x$ in the Lie algebra of $G$, such that transforming $U,V$ by the group element $g=\exp(x)$ improves the gradient $\frac{d\L}{dt}$. 
}


\subsection{Multi-layer Neural Networks}
Next, we consider feed-forward neural networks.
Denote the output of the $m$th layer by $h_m \in \R^{d_m \times n}$, where $d_m$ is the hidden dimension and $n$ is the number of samples. 
Denote the input by $h_0 = X \in \R^{d_0\times n}$. 
In terms of the previous layer output $\tilde{h}_m = W_m h_{m-1}$ where $W_m\in \R^{d_{m}\times d_{m-1}}$ (we absorb biases into $W_m$ by adding an extra row of ones in $\tilde{h}_{m-1}$), the output $h_m$ is 
\begin{align}
    h_m &= \sigma(\tilde{h}_m) = \sigma \pa{W_m h_{m-1} }.
\end{align}

We assume the activation functions $\sigma: \R\to \R $ are bijections. 
For instance, Leaky-ReLU is bijective. 
For other activation, such Sigmoid or Tanh, we analytically extend them to  bijective functions (e.g. $\tanh(x) + e^{x-N} - e^{-x+N}$ for $N\gg 1$). 
For linear activations, we have linear symmetries: 

\begin{proposition}
\label{thm:mult-layer}
A linear network is invariant under all groups $G_m \equiv \mathrm{GL}_{d_m}$ acting as 
\begin{align}
    g \cdot (W_{m}, W_{m-1}) &= (W_{m}g^{-1}, g W_{m-1}), &
    g \cdot W_{k} = W_k, \quad \forall k \notin \{m,m-1\}. 
\end{align}
\end{proposition}

Similar, in the nonlinear case, we want to find $g\cdot W_m$ that keep the outputs $h_k$ for $k\ne m-1$ invariant. 
With nonlinearity, $\tilde{h}_m = W_m \sigma(W_{m-1} h_{m-2})$. 
This network has a
$\mathrm{GL}_{d_{m-1}}$ symmetry 
given by 
\begin{align}
    g \cdot (W_{m}, h_{m-1}) = (W_{m}g^{-1}, g h_{m-1}), \quad
    g\cdot h_m = W_{m}g^{-1} g h_{m-1} = h_m.
    \label{eq:g-act-W-h-0}
\end{align}
\out{
First, we note that \eqref{eq:g-act-W-h-0} yields a nonlinear action on $\tilde{h}_{m-1}=W_{m-1}h_{m-2}$
\begin{align}
    g \cdot \{W_{m}, \tilde{h}_{m-1}\}&= \left\{W_{m}g^{-1}, \sigma^{-1}\pa{g \sigma( \tilde{h}_{m-1})}\right\}.
\end{align}
}

Thus, we have the following proposition regarding symmetries of nonlinear networks: 
%
\begin{proposition}
\label{thm:mult-layer-nonlinear}
Assume that $h_{m-2}$ is invertible.
A multi-layer network with bijective activation $\sigma$ has a $\mathrm{GL}_{d_{m-1}}$ symmetry.
For $g_m \in G_m = \mathrm{GL}_{d_{m-1}}(\R)$ the following group action keeps $h_p$ with $p\geq m$ invariant 
\begin{align}
    g_m \cdot W_k = \left\{
    \begin{array}{lc}
        W_m g_m^{-1} & k = m \\
        \sigma^{-1}\pa{g_m \sigma\pa{W_{m-1} h_{m-2}} } h_{m-2}^{-1} & k = m-1 \\
        W_k & k \not\in \{m, m-1\}
    \end{array}
    \right.
    \label{eq:Wm-g-action-main}
\end{align}
\end{proposition}

Proofs can be found in Appendix \ref{appendix:mult-layer}. Note that this group action depends on the input to the network as well as the current weights of all the lower layers. 
Yet, since the action keeps the output of upper and lower layers invariant, multiple $G_m$ for different $m$ applied at the same time still keep the network output invariant. 
The proposition assumes that $h_{m-2}$ is square and full-rank.
When $n < d_{m-2}$ and $h_{m-2}$ has rank $n$, \eqref{eq:Wm-g-action-main} (with left inverse of $h_{m-2}$) keeps the loss invariant but does not satisfy the identity axiom of a group action. 

%% file: secs/theory.tex
\section{Theoretical Analysis}
\subsection{What symmetries help accelerate optimization}
We now discuss the conditions that need to be satisfied for teleportation to accelerate GD. 
For brevity, we denote all trainable parameters (e.g. $ \left\{W_1,\cdots W_p \right\}$ for the $p$-layer neural network) by a single flattened vector $\vw \in \R^n$. 
Consider a symmetry $G$ of the loss function $\L(\vw)$, meaning 
for all $ g \in G$, $\L(\vw) = \L(g \cdot \vw)$. 
We quantify how teleportation by $G$ changes the rate of loss decay, given by 
\begin{align}
    {d\L(\vw)\over dt}& = \bk{{\ro \L \over \ro \vw },{d\vw \over dt} } = - \br{\grad \L}^T\eta \grad\L = - \|\grad\L\|^2_\eta,
    \label{eq:dLdt-0}
\end{align}
where $\eta$ is the matrix of learning rates (which must be positive semi-definite) and $\|v \|_\eta^2=v^T\eta v$ is the Mahalanobis norm. A constant learning rate means $\eta=I$.

The following proposition provides the condition a symmetry needs to satisfy to accelerate optimization. 
(Proofs can be found in Appendix \ref{appendix:speedup-cond-prop}.)

\begin{proposition}
\label{prop:speedup-cond}
Let $\vw'=g\cdot \vw$ be a point we teleport to. 
Let $J= \ro \vw' /\ro \vw $ be the Jacobian. 
Symmetry teleportation using $g$ accelerates the rate of decay in $\L$ if it satisfies
\begin{align}
    \left\|\br{J^{-1}}^T\grad \L(\vw)\right\|^2_\eta > \left\|\grad \L(\vw)\right\|^2_\eta.
    \label{eq:dL-w-cond-0}
\end{align}

\end{proposition}

If the action of the symmetry group $G\subset GL_n$ is linear we have $\vw' = g \cdot \vw = g \vw$ and $J=g$.
It follows that if $G$ is a subgroup of the orthogonal group, $d\L/dt$ will be invariant: 
\begin{corollary}
\label{prop:corollary-cond}
Let $O_\eta $ denote the orthogonal group of invariances of the inverse of the learning rate, $\eta^{-1}$, meaning for $g\in O_\eta$, $ g^T\eta^{-1} g= \eta^{-1}$. 
Then
\begin{align}
    \forall g &\in O_\eta, \qquad {d\L (g\cdot \vw)\over dt } = {\L(\vw)\over dt } .
\end{align}
\end{corollary}

In the simple case where the learning rate is a constant, $O_\eta = O(n)$ becomes the classic orthogonal group (e.g. rotations)
with $g^Tg = I $. 
In general, when $g$ preserves the norm of the gradient, symmetry teleportation has no effect on the convergence rate. 


\out{
\begin{proposition}
Let   $ M&\equiv {\ro \L\over \ro \vw } \br{{\ro \L\over \ro \vw }}^T$, and $J = \partial (g\cdot \vw) / \partial \vw$ and $J^{-1} = \partial \vw / \partial (g\cdots \vw)$, then symmetry teleportation accelerates optimization if it satisfies
\[ \Tr\br{ M [J^{-1}]^TJ^{-1} } > \Tr[M] \]
\bz{equivalently, \[\left\| J^{-1} \frac{\ro \L(\vw)}{\ro \vw} \right\|^2 > \left\| \frac{\ro \L(\vw)}{\ro \vw} \right\|^2\]
which means the Jacobian of the group action increases the gradient norm.} \ry{simplify and explain}
\end{proposition}

\begin{proof}
Let $\vw' = g \cdot \vw$. 
Denote the Jacobian as $J$, where $J_i^j = \ro \vw_i' / \ro \vw_j$. 
Then the inverse of $J$ has entries $[J^{-1}]_i^j = \ro \vw_i / \ro \vw_j'$.
The gradient at $\vw'$ is
\begin{align}
    \frac{\ro \L(\vw')}{\ro \vw'_i} 
    &= \frac{\ro \L(\vw)}{\ro \vw'_i} 
    = {\ro \vw_j \over \vw'_i} \frac{\ro \L(\vw)}{\ro \vw_j} \cr 
    & = [J^{-1}]_j^i\frac{\ro \L(\vw)}{\ro \vw_j} 
    = \br{J^{-1} \frac{\ro \L(\vw)}{\ro \vw} }^i
\end{align}
The rate of change of $\L$ during gradient flow $d\vw /dt = -\eta \ro \L/\ro \vw $ is 
\begin{align}
    \frac{d \L(\vw')}{dt} 
    &= \sum_i \frac{\ro \L(\vw')}{\ro \vw_i'} \frac{d \vw_i'}{d t}
    = -\eta \sum_i \frac{\ro \L(\vw')}{\ro \vw_i'} \frac{\ro \L(\vw')}{\ro \vw_i'} \cr 
    &= -\eta \left\| J^{-1} \frac{\ro \L(\vw)}{\ro \vw} \right\|^2 
    = -\eta \Tr\br{ M [J^{-1}]^TJ^{-1} }
    \cr
    M&\equiv {\ro \L\over \ro \vw } \br{{\ro \L\over \ro \vw }}^T
\end{align}
\nd{ derive the full AdaGrad version }
Thus we will have a speedup if 
\begin{align}
    \mbox{if }\Tr\br{ M [J^{-1}]^TJ^{-1} } > \Tr[M]
    \label{eq:MJ-sppedup-cond}
\end{align}
\end{proof}
}


From Theorem 3.2 in \cite{hazan2019lecture}, assuming that $\L$ is $\beta$-smooth and is bounded by $|\L| \leq M$, the gradient norm in gradient descent converges as
$\frac{1}{2\beta} \sum_{t=1}^T \| (\grad \L)|_{\vw_t} \| \leq 2M$. Teleportation increases $(\grad \L)|_{\vw_t}$, therefore requiring less time to reach convergence.

\subsection{Improvement of subsequent steps }
Since teleportation moves the parameters to a point with a larger  gradient, and subsequent GD steps are local, we would expect that teleportation improves the magnitude of the gradient for a few future steps as well. The following results formalize this intuition (proofs in Appendix \ref{appendix:lipschitz}.).
\begin{assumption}[Lipschitz Continuity]
\label{thm:lipschitz}
The l2 norm of the gradient is Lipschitz continuous with constant $L \in \R^{\geq 0}$ , which is
\begin{align}
    \left| \left\| \frac{\partial \L}{\partial \vw_{1}} \right\|_2 - \left\| \frac{\partial \L}{\partial \vw_{2}} \right\|_2 \right| \leq L \| \vw_1 - \vw_2 \|_2,
\end{align}
where $\vw_1, \vw_2$ are two points in the parameter space and $L$ is the Lipschitz constant.
\end{assumption}
\begin{proposition}
\label{prop:lipschitz}
Consider the gradient descent with a $G$-invariant loss $\L(\vw)$ and learning rate $\eta \in \R^+$. Let $\vw_t$ be the parameter at time $t$ and $\vw_t' = g \cdot \vw_t$ the parameter teleported by $g\in G$. Let $\vw_{t+T}$ and $\vw_{t+T}'$ be the parameters after $T$ more steps of gradient descent from $\vw_t$ and $\vw_t'$ respectively. Under Assumption \ref{thm:lipschitz}, if $\eta L < 1$, and 
\begin{align}
    \frac{\left\| \partial \L / \partial \vw_{t}' \right\|_2}{\left\| \partial \L / \partial \vw_{t} \right\|_2} \geq \frac{(1 + \eta L)^T}{(1 - \eta L)^T},
\end{align}
then 
\begin{align}
    \left\| \frac{\partial \L}{\partial \vw_{t+T}'} \right\|_2 \geq \left\| \frac{\partial \L}{\partial \vw_{t+T}} \right\|_2.
\end{align}
\end{proposition}
The proposition provides a sufficient condition for teleportation to improve $T$ future steps. The condition is met when $L$ is small, $\eta$ is small, or the initial improvement, $\left\| \frac{\partial \L}{\partial \vw_{t}'} \right\|_2 / \left\| \frac{\partial \L}{\partial \vw_{t}} \right\|_2$, is large.

\subsection{Convergence analysis for convex quadratic functions}
Teleportation improves the magnitude of gradient for the current step. We have also shown that the magnitude of gradient stays large for a few subsequent steps (Proposition \ref{prop:lipschitz}). In this section, we analyze a class of functions where teleporting once guarantees optimality at all future times. 
We consider a trajectory optimal if for every point on the trajectory, the magnitude of gradient is at a local maximum in the loss level set that contains the point. 

Consider a positive definite quadratic form $\L_A(\vw) = \vw^T A \vw$, where $\vw \in \mathbb{R}^n$ is the parameter and $A \in \mathbb{R}^{n \times n}$ is a positive definite matrix. The gradient of $\L_A$ is $\grad \L_A = 2A \vw$, and the Hessian of $\L_A$ is $H = 2A$. Since $A$ is defined to be positive definite, $\L_A$ is convex. The function $\L_A$ has global minimum at a single point $\vw^* = 0$.


Let $\rho$ be a representation of $O(n)$ acting on $\R^n$. For $g \in O(n)$, we define the following group action:
\begin{align}
    g \cdot \vw = A^{-\frac{1}{2}} \rho(g) A^{\frac{1}{2}} \vw.
\end{align}
Then it can be shown that $\L_A(\vw)$ admits a $O(n)$ symmetry:
\begin{align}
    \L_A(g \cdot \vw) = \vw^T {A^{\frac{1}{2}}}^T \rho(g)^T {A^{-\frac{1}{2}}}^T A A^{-\frac{1}{2}} \rho(g) A^{\frac{1}{2}} \vw = \vw^T A \vw = \L_A(\vw).
\end{align}

Let $S_c = \{\vw: \L_A(\vw) = c\}$ be a level set of $\L_A$. 
We show that after a teleportation, every point on the gradient flow trajectory is optimal in its level set (with full proof in Appendix \ref{appendix:convex-quadratic}). 
\begin{proposition}
\label{prop:convex-quadratic-global-opt}
If at point $\vw$, $\| \grad \L_A|_\vw \|^2$ is at a maximum in $S_{\L_A(\vw)}$, then for any point $\vw'$ on the gradient flow trajectory starting from $\vw$, $\| \grad \L_A|_{\vw'} \|^2$ is at a maximum in $S_{\L_A(\vw')}$.
\end{proposition}

Finally, we observe that for $\L_A$, teleportation moves the parameters closer to the global minimum in Euclidean distance. In other words, maximizing the magnitude of gradient minimizes the distance to $\vw^*$ in a loss level set.
\begin{proposition}
\label{prop:convex-quadratic-dist-minimum}
Consider a point $\vw$ in the parameter space. Let $g = \argmax\nolimits_{g \in G} \|\grad \L_A |_{g \cdot \vw}\|^2_2.$ Then $g \cdot \vw = \argmin\nolimits_{\vw' \in S_{\L_A(\vw)}} \|\vw' - \vw^*\|^2_2$. 
\end{proposition}

\subsection{Relation to second-order optimization methods}
Since our algorithm involves optimizing the gradients, symmetry teleportation is closely related to second-order optimization methods. At the target point to teleport to, gradient descent becomes equivalent to Newton's method (proof in Appendix \ref{appendix:newton-direction-prop}).

\begin{proposition}
Let $S_{\L_0} = \{\vw: \L(\vw) = \L_0\}$ be a level set of $\L$. 
If at a particular $\vw \in S_{\L_0}$ we have $ \|\grad\L(\vw)\|_2\geq  \|\grad\L(\vw')\|_2 $  for all $\vw'$ in a small neighborhood of $\vw$ in $S_{\L_0}$, then the gradient $\grad \L(\vw)$ has the same direction as the direction from Newton's method, $H^{-1} \grad\L(\vw) $.
\label{prop:newton-direction}
\end{proposition}


Proposition \ref{prop:newton-direction} provides an alternative way to interpret teleportation. Instead of computing the Newton's direction, we search within the loss level set for a point where the gradient has the same direction as Newton's direction. 
However, symmetry teleportation does not require computing the full Hessian matrix. The second derivative required for optimizing over continuous groups is obtained by taking derivatives with respect to parameters and then with respect to the group element, as opposed to taking the derivative with respect to parameters twice. This makes the computation significantly more feasible than Newton's method on neural networks with large number of parameters.

In practice, however, the group action used for teleportation is usually not transitive. Additionally, we do not teleport using the optimal group element since it can be unbounded. We also do not apply teleportation in every gradient descent step. Therefore, Proposition \ref{prop:newton-direction} serves as an intuition instead of an exact formulation of how teleportation works. We provide empirical evidence in Section 6 that these approximations do not erase the benefits of teleportation completely, and leave theoretical investigations of the connection to second-order methods under approximations as future work. 



%% file: secs/experiments.tex
\section{Experiments}

\subsection{Acceleration through symmetry teleportation}
\label{sec:experiment}
We examine the effect of symmetry teleportation on optimization. 
We illustrate teleportation in the parameter space on two test functions and show a speedup in regression and classification problems using multilayer neural networks. 
For test functions, we compared with GD for illustration purposes.  
For multi-layer neural networks, we also include AdaGrad as a more competitive baseline.
\begin{figure}[t!]
\centering
a\hfill b \hfill c\hfill d \hfill ~ \\
\includegraphics[width=0.24\textwidth, trim=35pt 20pt 0 0, clip]{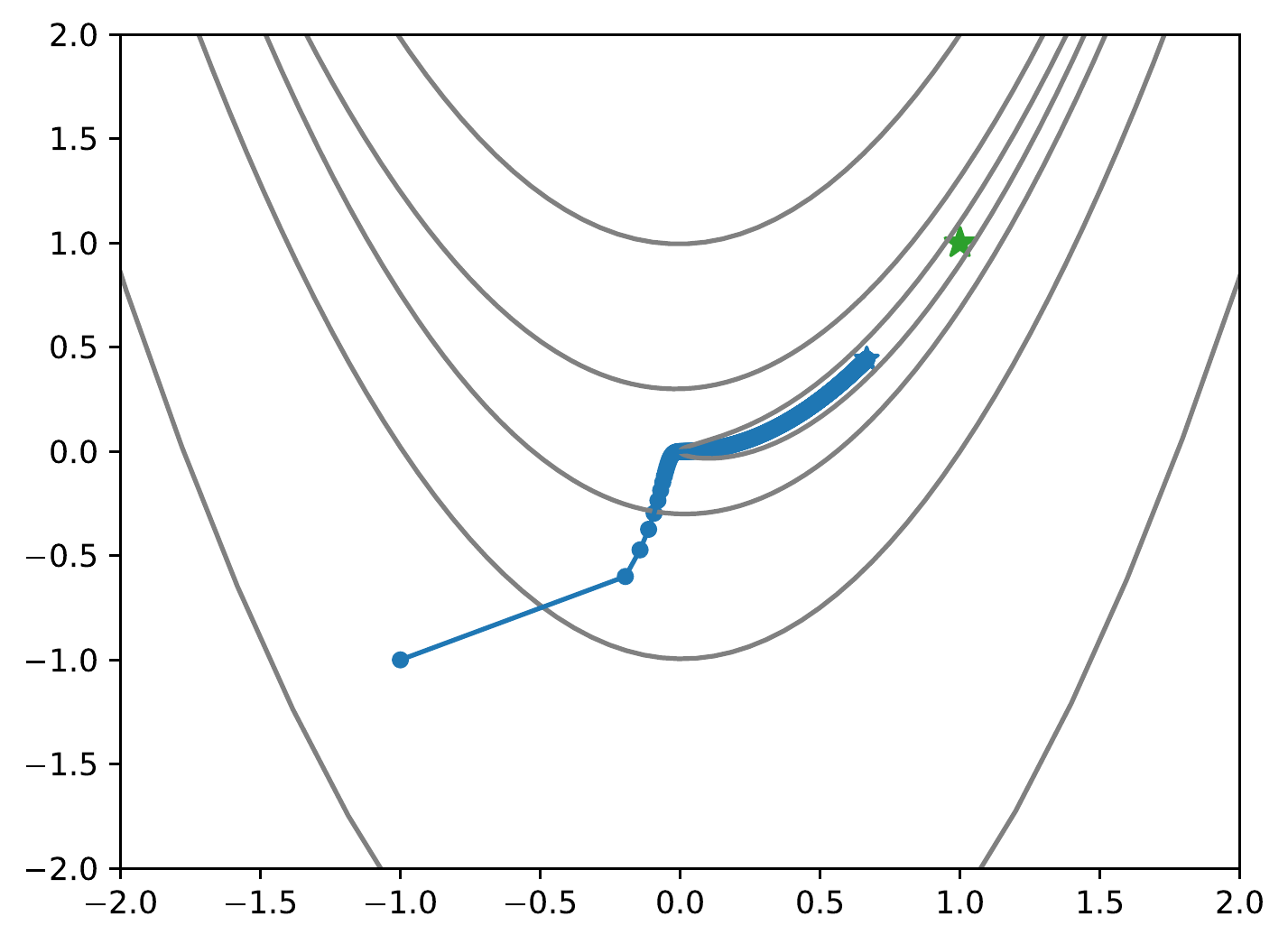}
\includegraphics[width=0.24\textwidth, trim=35pt 20pt 0 0, clip]{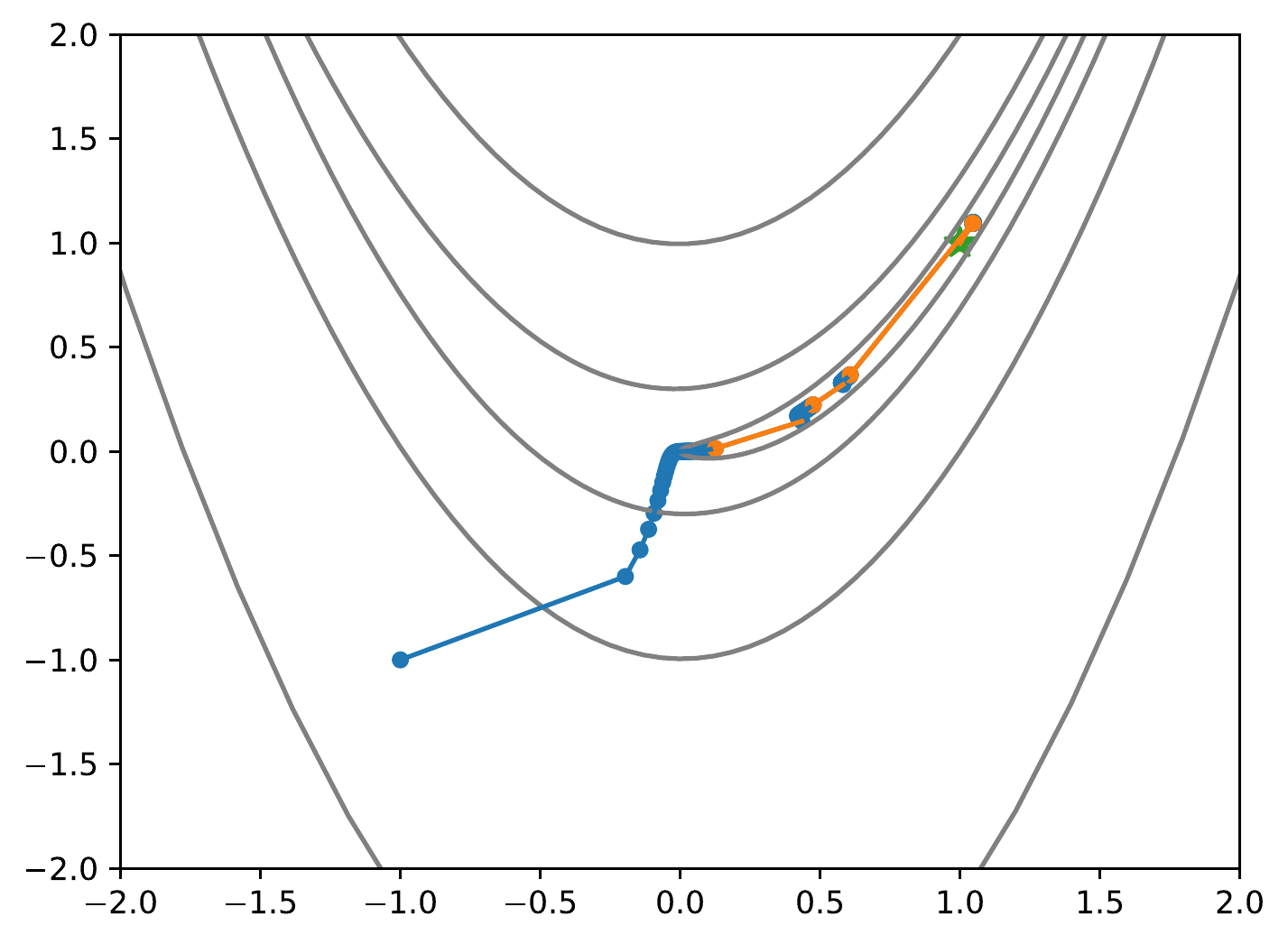}
\includegraphics[width=0.24\textwidth]{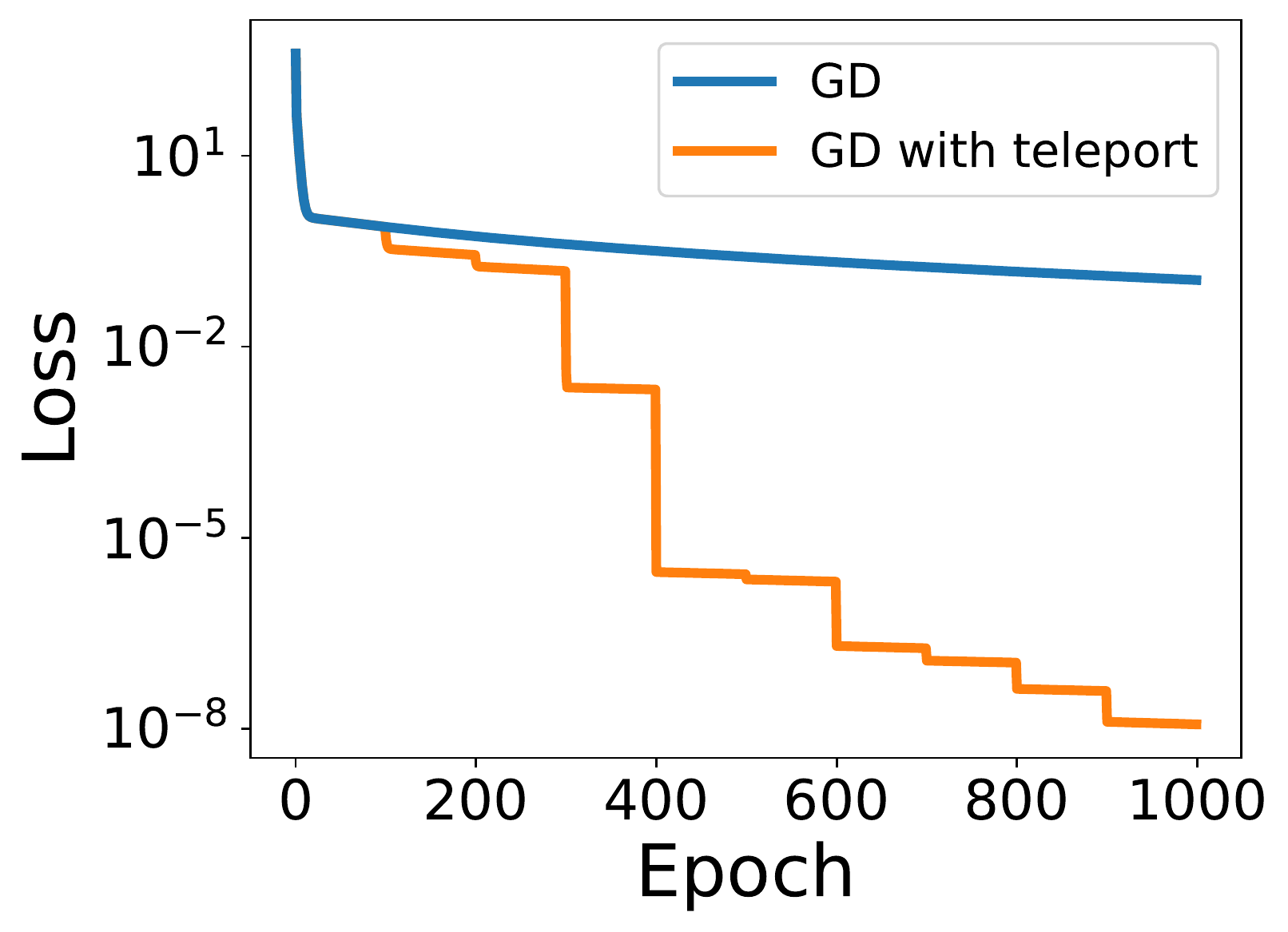}
\includegraphics[width=0.24\textwidth]{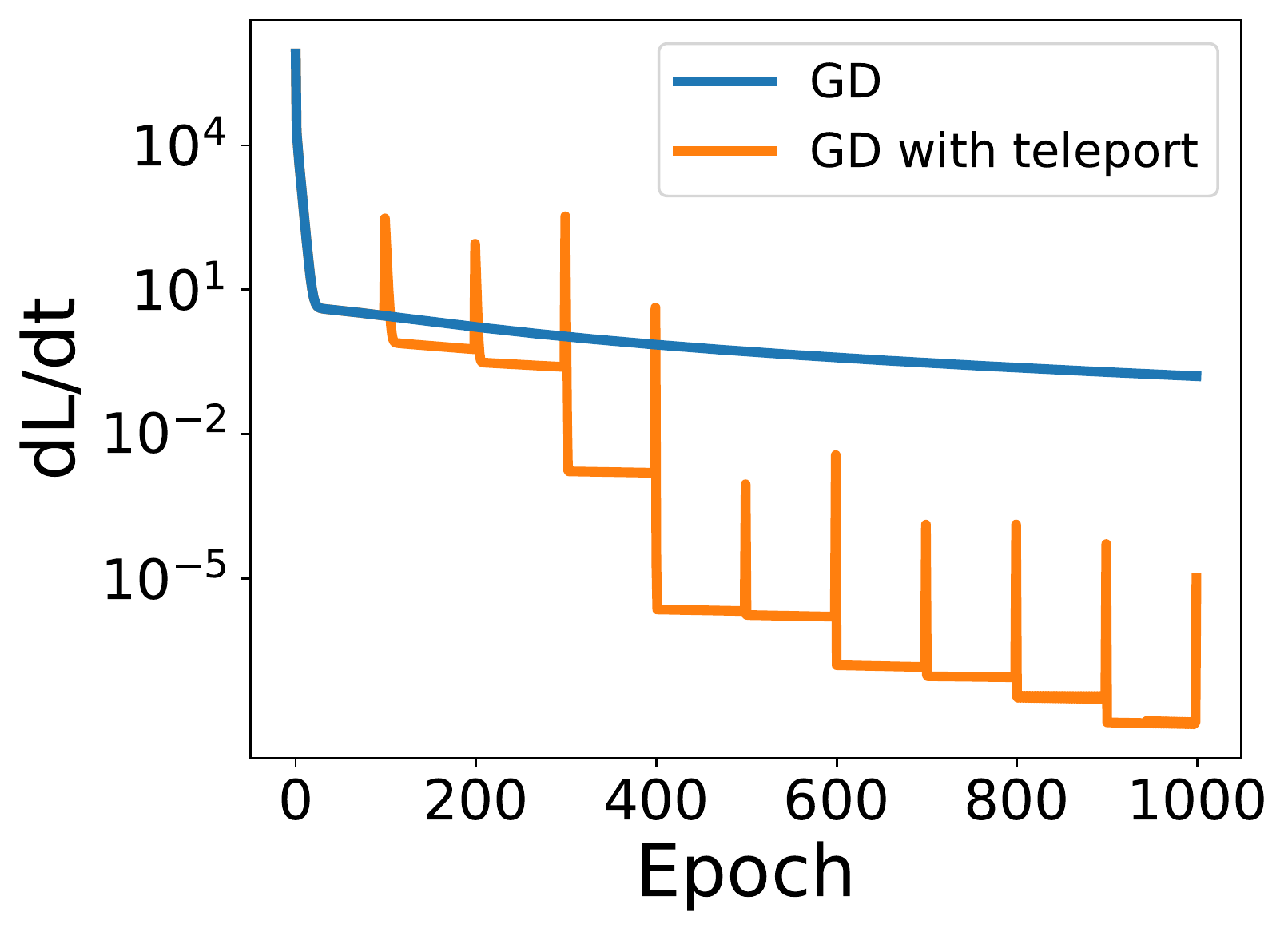}\\
e\hfill f \hfill g\hfill h \hfill ~ \\
\includegraphics[width=0.24\textwidth, trim=40pt 20pt 0 0, clip]{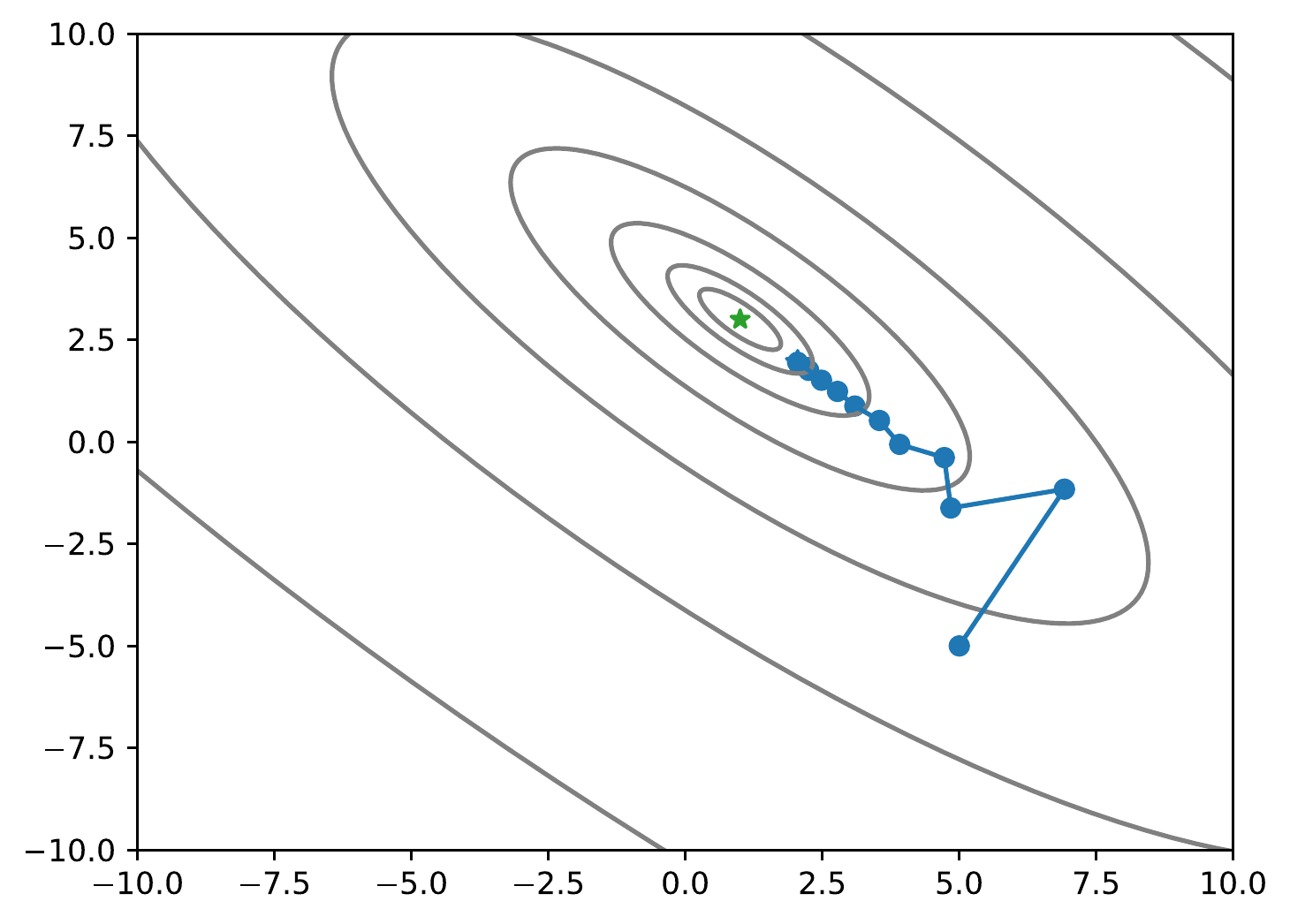}
\includegraphics[width=0.24\textwidth, trim=40pt 20pt 0 0, clip]{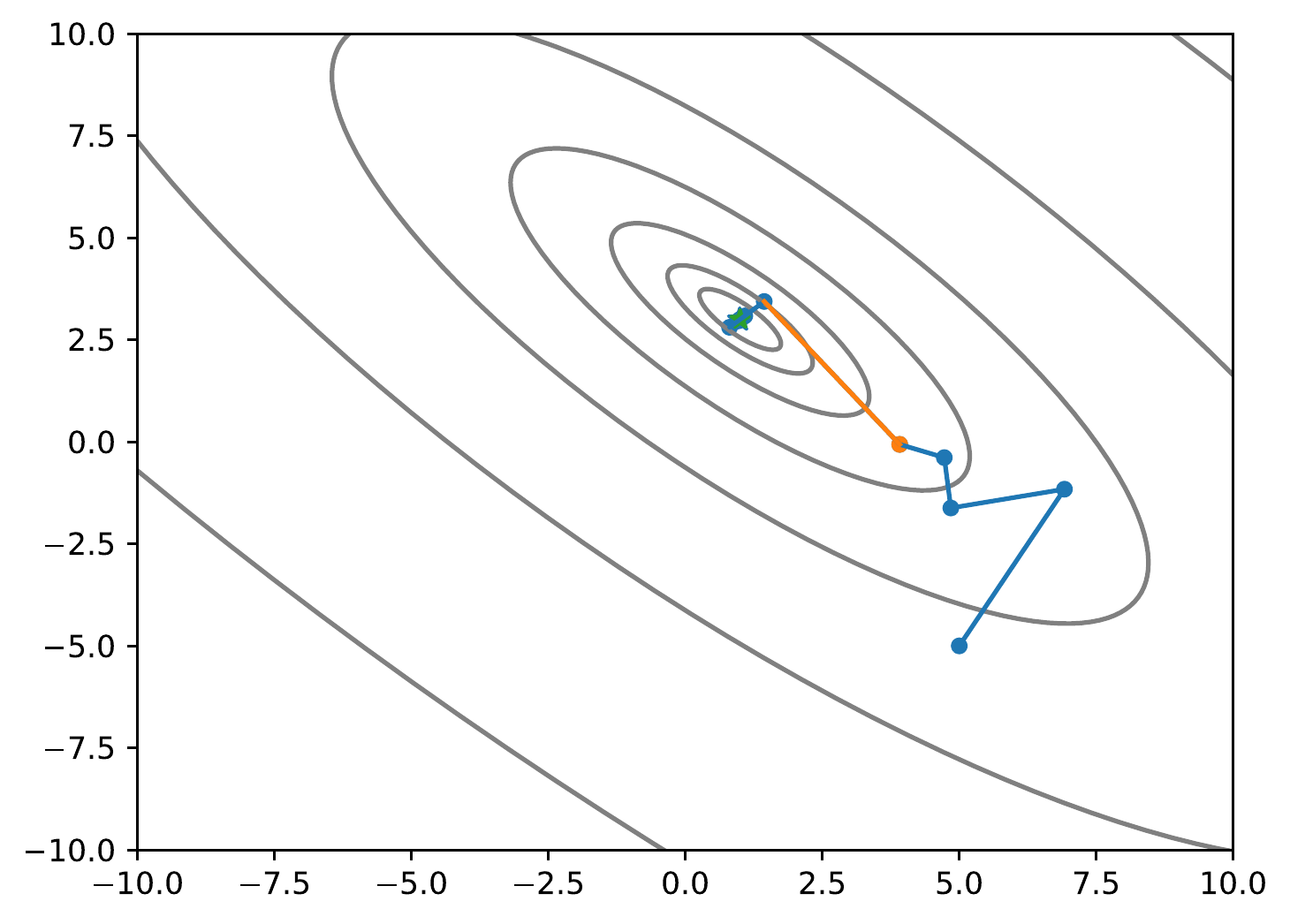}
\includegraphics[width=0.24\textwidth]{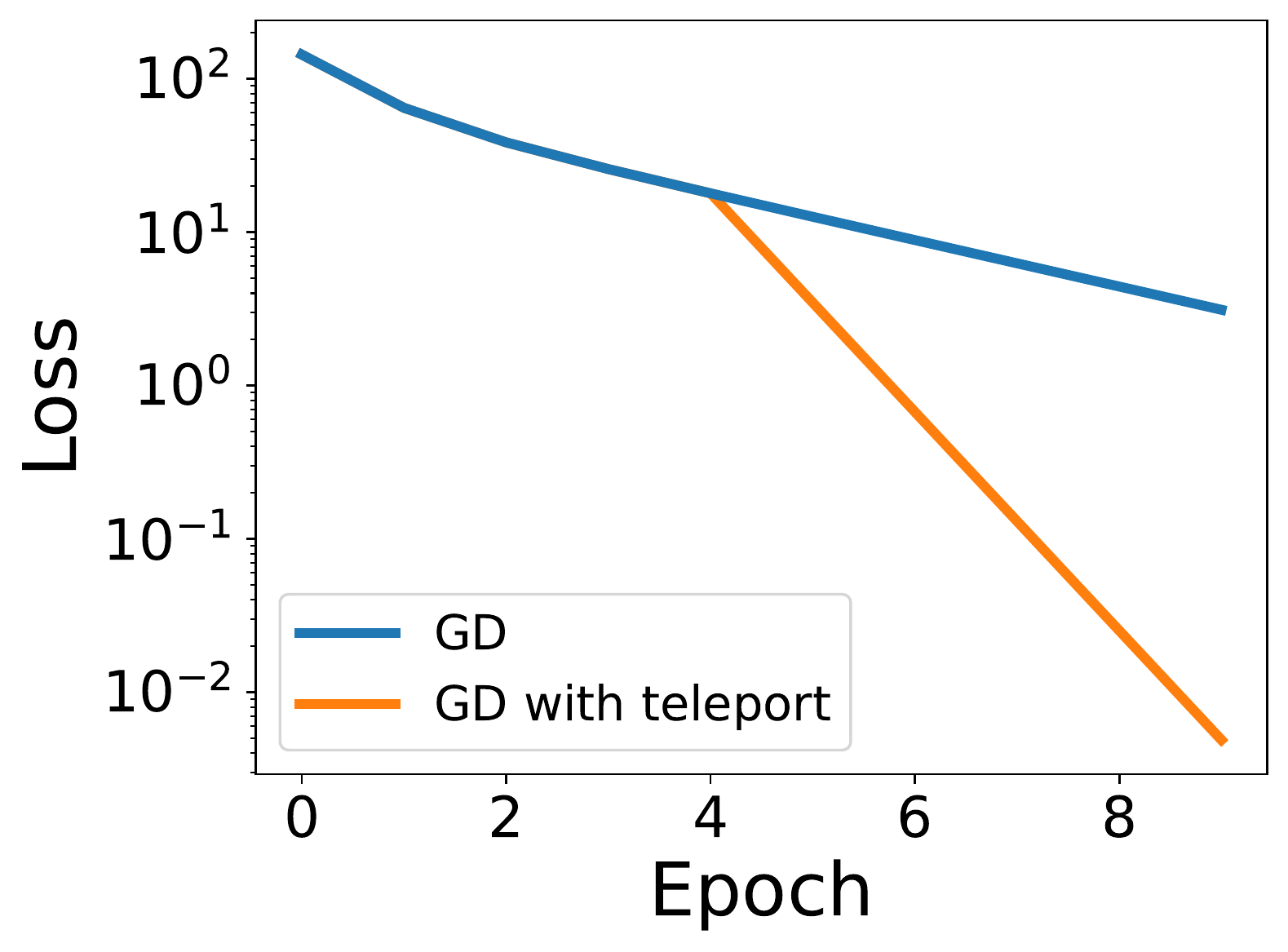}
\includegraphics[width=0.24\textwidth]{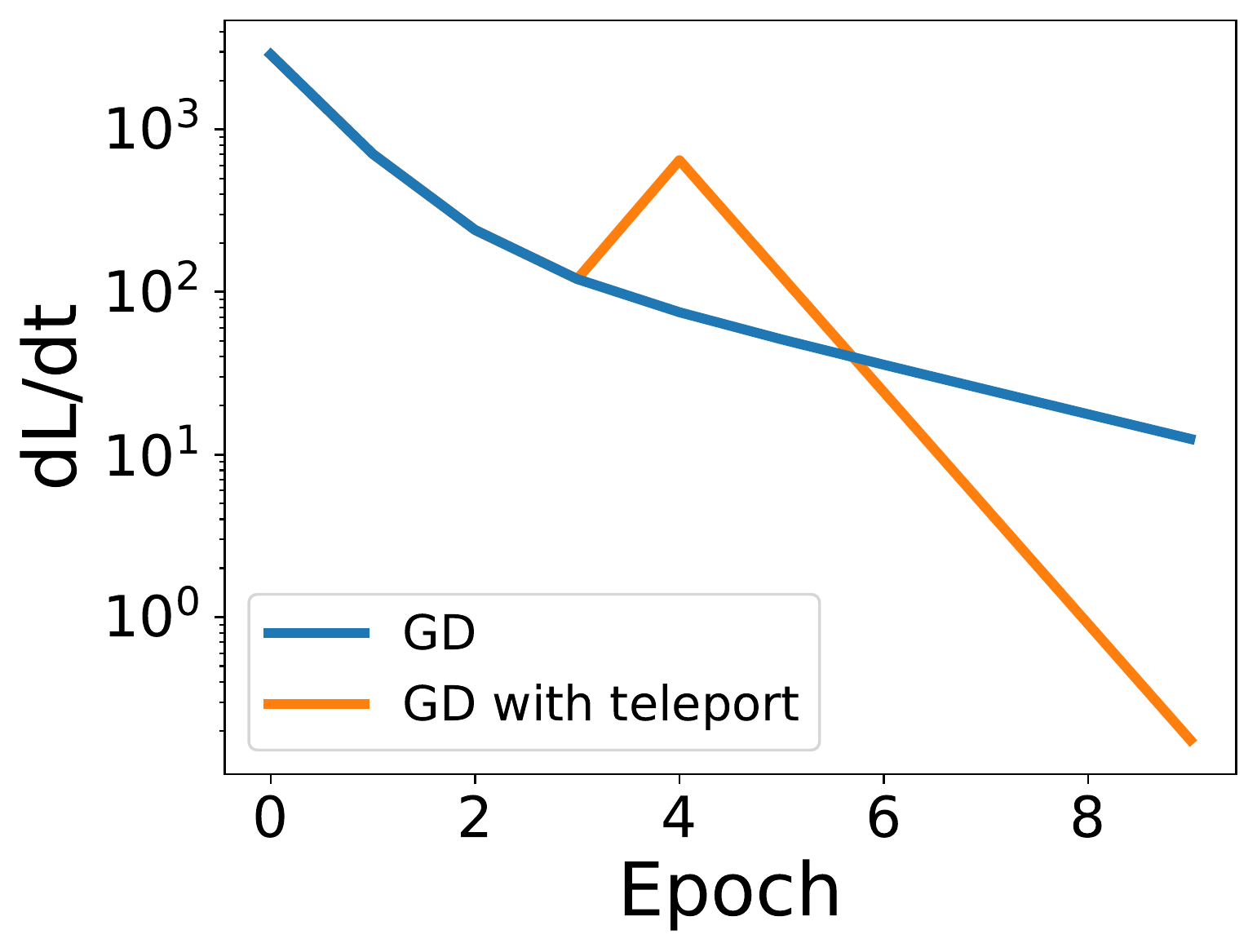}
\caption{Optimization of the Rosenbrock function (top row) and Booth function (bottom row) using (a) gradient descent and (b) the proposed algorithm. Contours represent the level sets of the loss function. Loss $\L$ and convergence rate $d\L/dt$ are shown in (c) and (d). Teleportation helps move the parameters towards the target. }
\label{fig:test_functions}
\end{figure}
\out{
\begin{figure}[t!]
\centering
\subfigure[]{\includegraphics[width=0.24\textwidth]{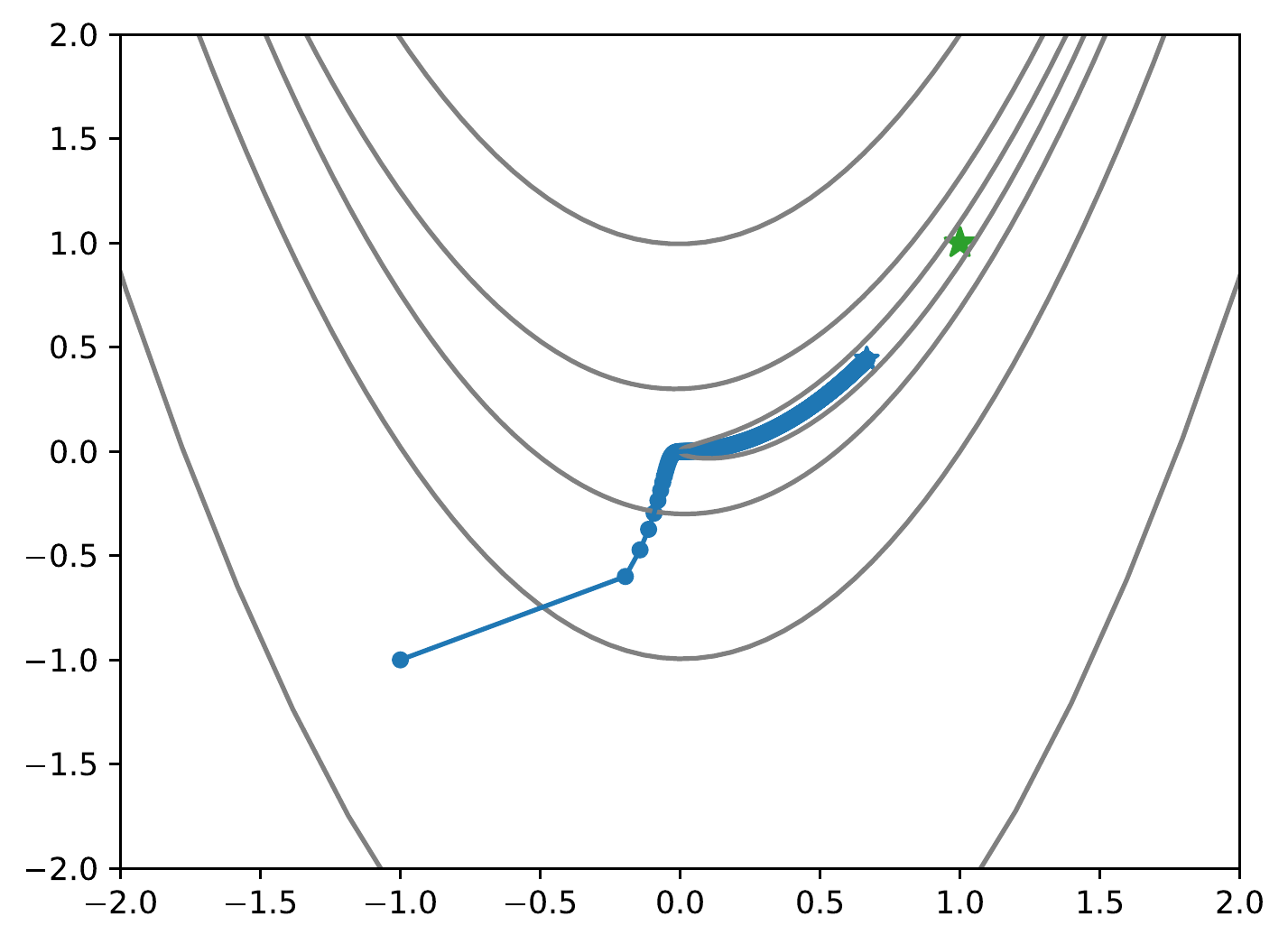}}
\subfigure[]{\includegraphics[width=0.24\textwidth]{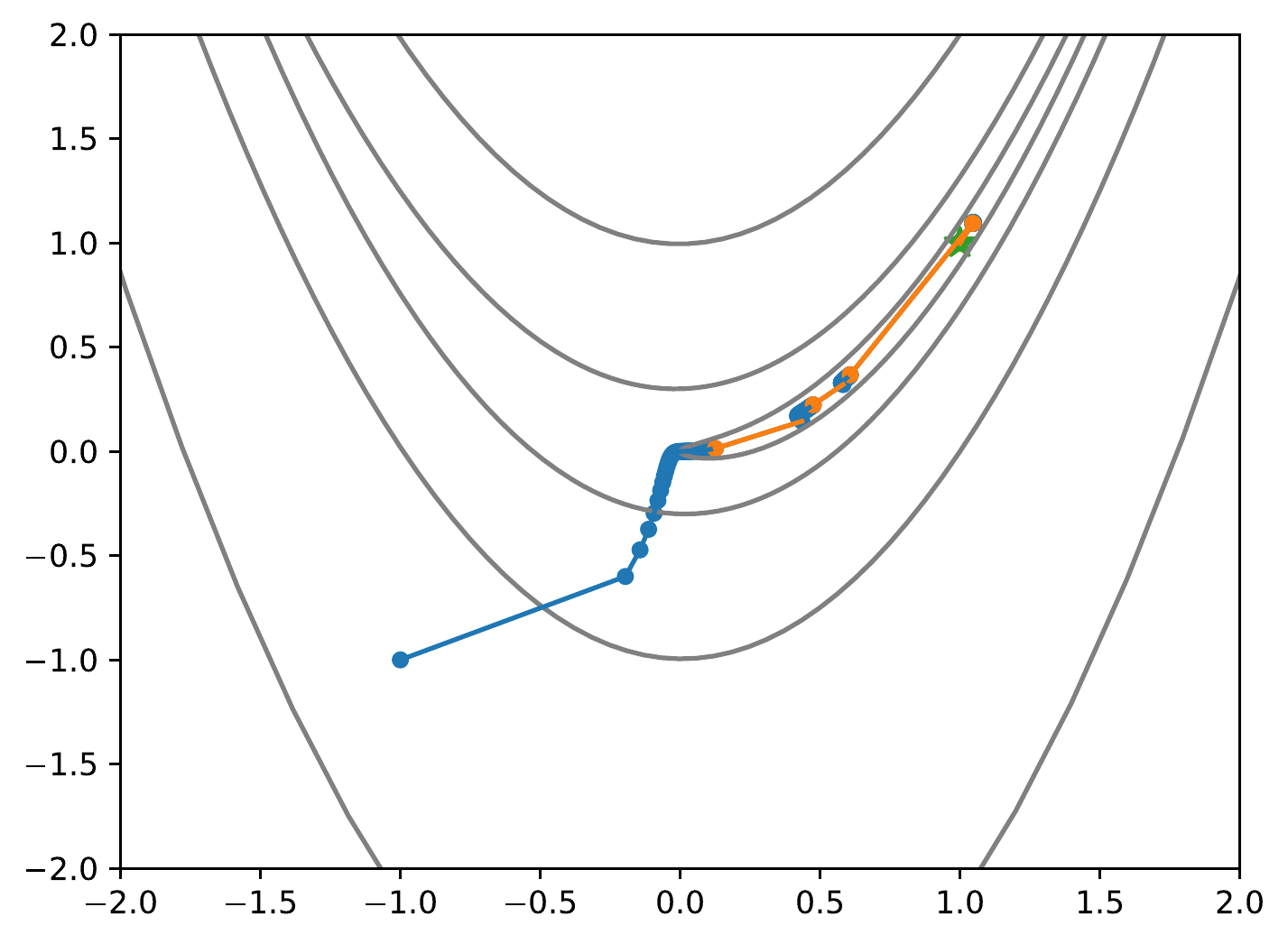}}
\subfigure[]{\includegraphics[width=0.24\textwidth]{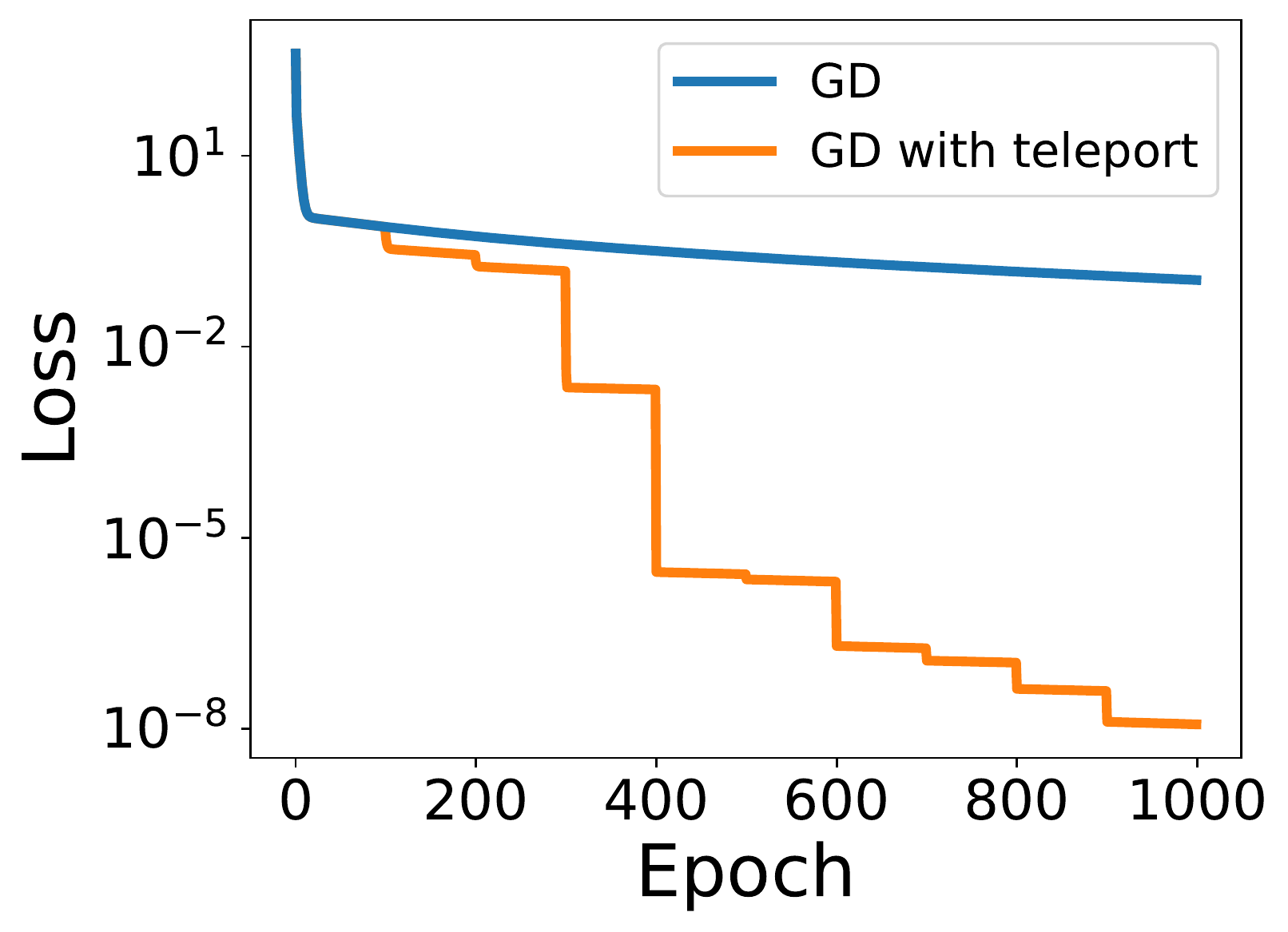}}
\subfigure[]{\includegraphics[width=0.24\textwidth]{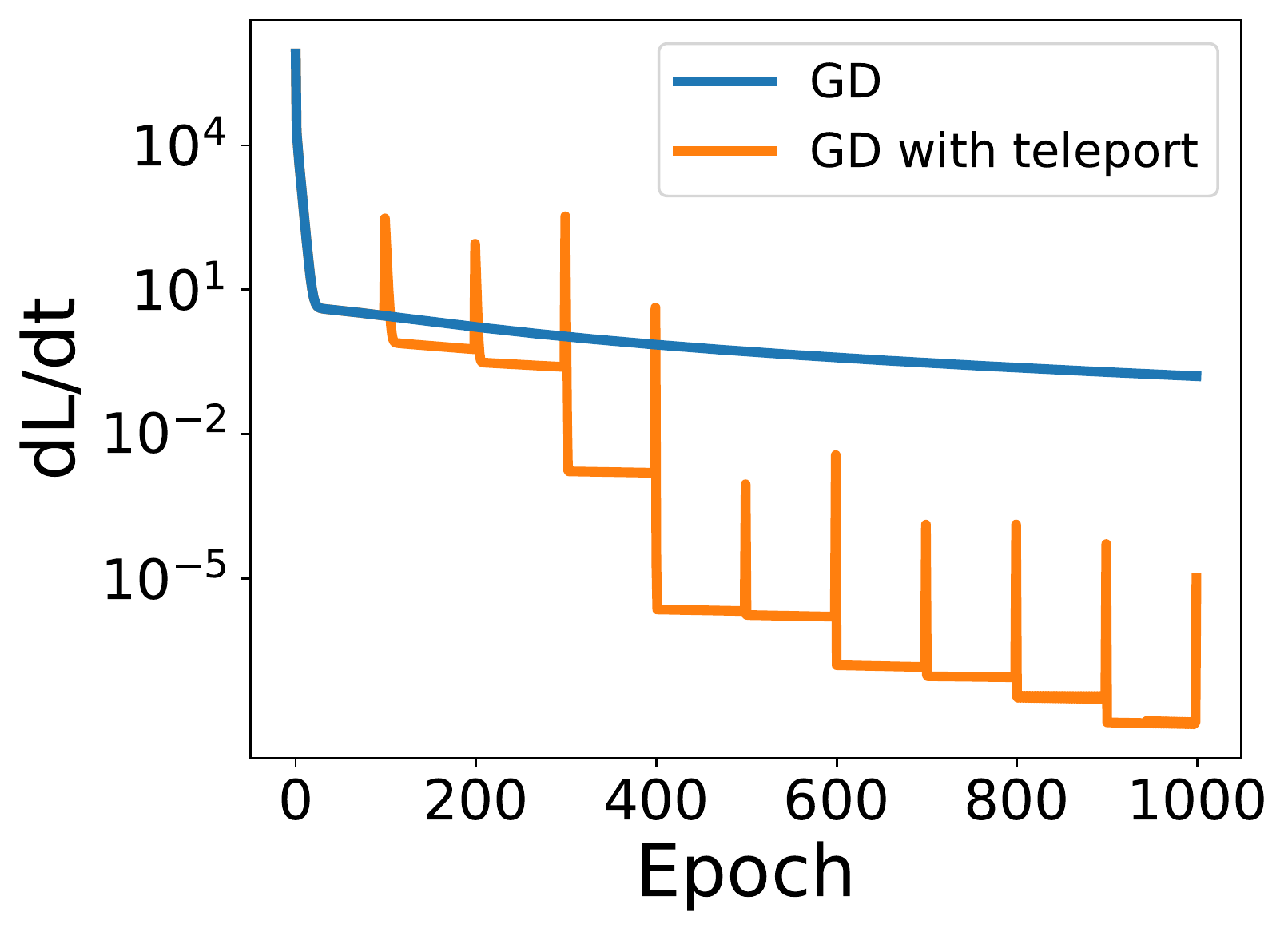}}
\subfigure[]{\includegraphics[width=0.24\textwidth]{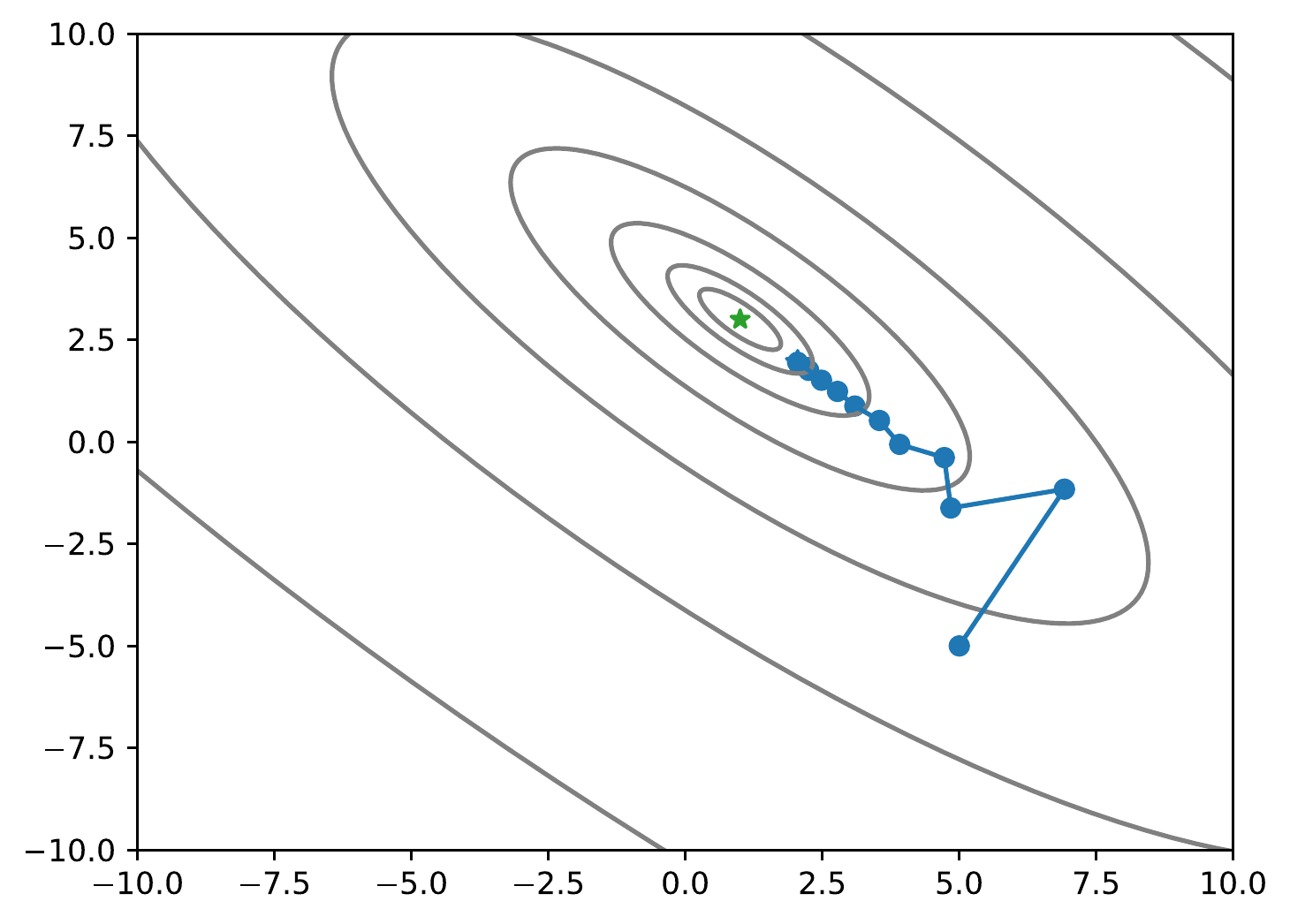}}
\subfigure[]{\includegraphics[width=0.24\textwidth]{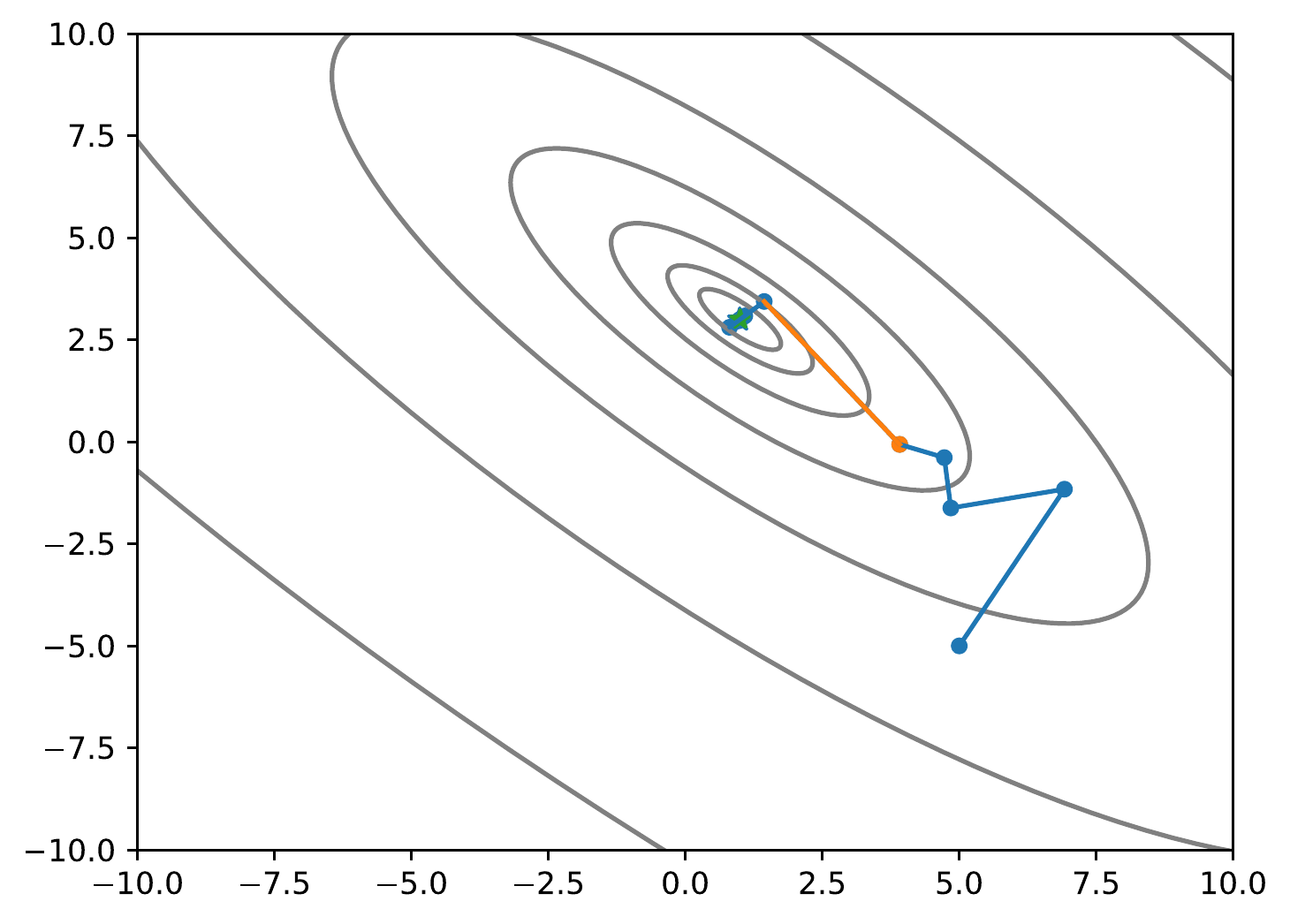}}
\subfigure[]{\includegraphics[width=0.24\textwidth]{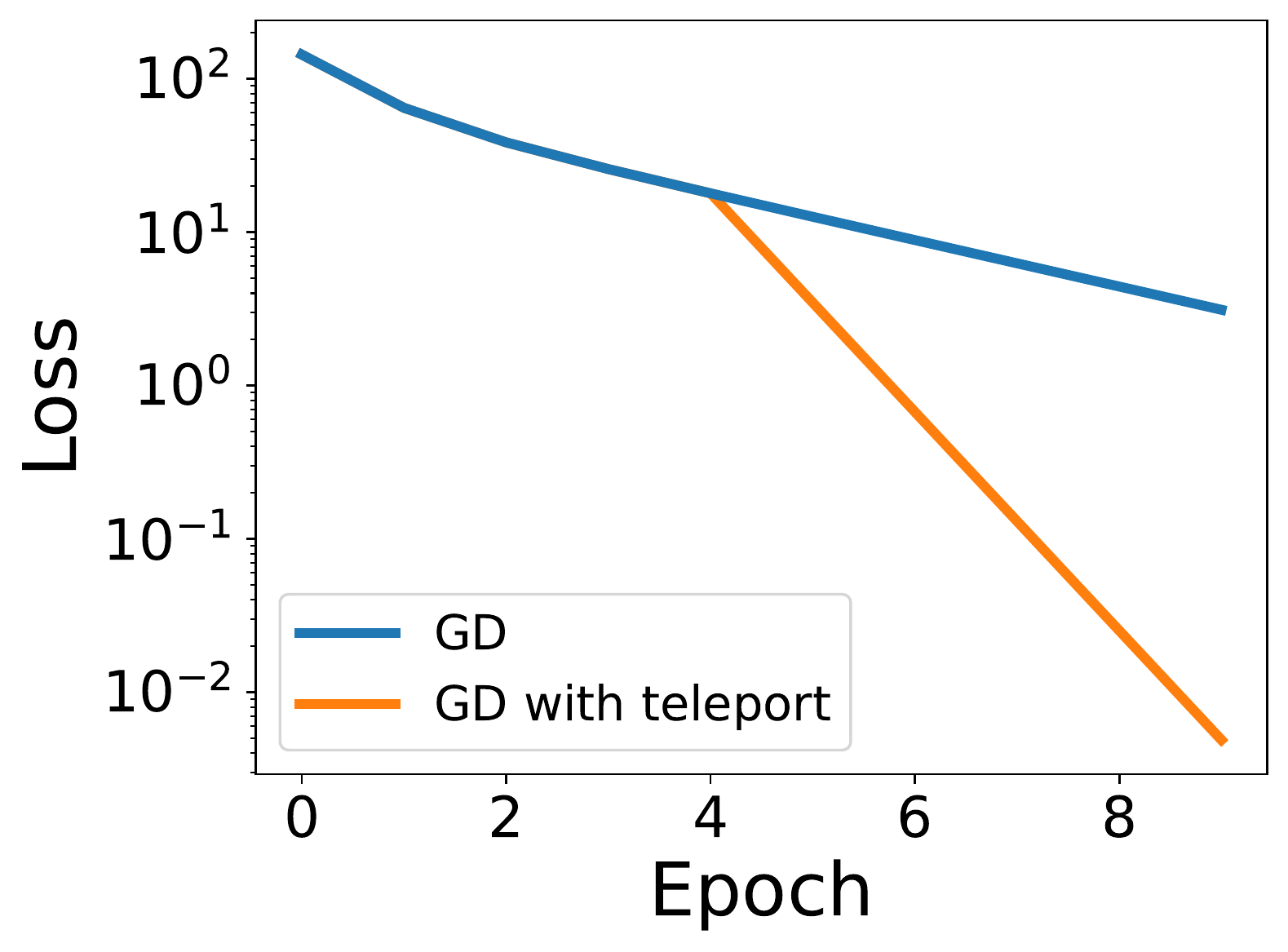}}
\subfigure[]{\includegraphics[width=0.24\textwidth]{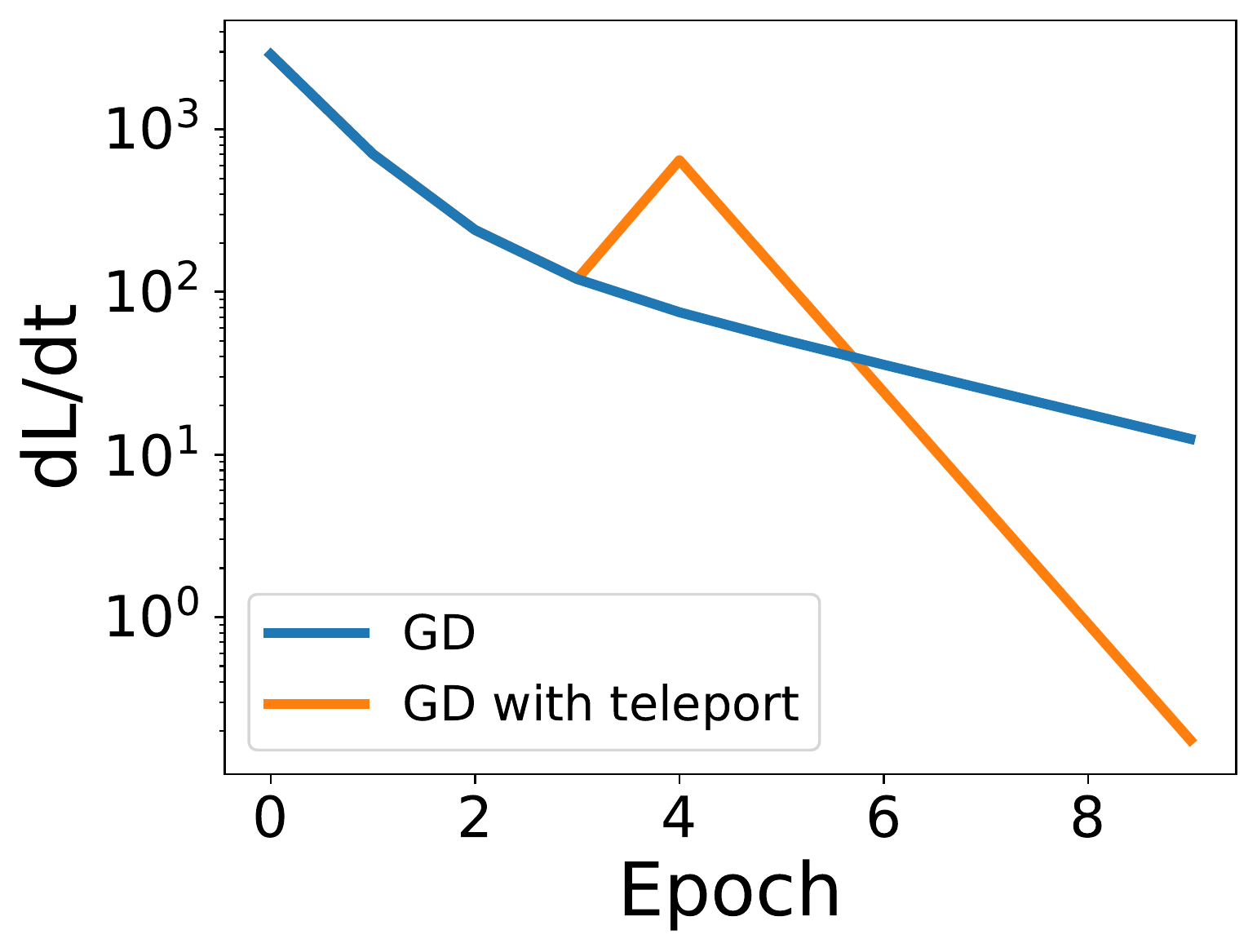}}
\caption{Optimization of the Rosenbrock function (top row) and Booth function (bottom row). using (a) gradient descent and (b) the proposed algorithm. Contours represent the level sets of the loss function. Loss $\L$ and convergence rate $d\L/dt$ are shown in (c) and (d). Teleportation helps move the parameters towards the target. }
\label{fig:test_functions}
\vspace{-3mm}
\end{figure}
}
\paragraph{Rosenbrock function.} We apply symmetry teleportation to optimize the 2-variable Rosenbrock function \eqref{eqn:rosenbrock}. The parameters $x_1, x_2$ are initialized to $(-1, -1)$. Each algorithm is run 1000 steps with learning rate $10^{-3}$. We teleport the parameters every 100 steps. The group elements are found by gradient ascent on $\theta$, the parameter for the $\mathrm{SO}(2)$ group, for 10 steps with learning rate $10^{-1}$. 

The trajectory of parameters and the loss level sets are plotted in Figure \ref{fig:test_functions}a,b. The blue star denotes the final position of parameters, the green star denotes the target, and orange dots are the positions from which symmetry transforms start. While gradient descent is not able to reach the target in 1000 steps, teleportation allows large steps and reaches the target much earlier. Figure \ref{fig:test_functions}d shows that teleportation improves the norm of gradients in the following step, and \ref{fig:test_functions}c shows its effect on the loss value. 
Teleportations clearly reduce the number of steps needed for convergence. 

\paragraph{Booth function.}
We also test symmetry teleportation on the Booth function defined in  Eqn. \eqref{eqn:booth}. We initialize the parameters $x_1, x_2$  to $(5, -5)$. Each algorithm is run 10 steps with learning rate $0.08$. We perform symmetry teleportation on the parameters once, before epoch 5. The group elements are found by gradient ascent on $\theta$ for 10 steps with learning rate $0.001$. $\theta$ is initialized uniformly at random over $[0, \pi)$. Similar to the Rosenbrock function, teleportation moves the parameters to a trajectory with a larger convergence rate (Figure \ref{fig:test_functions} bottom row). 

\paragraph{Multilayer neural network regression.}
We further evaluated our method on a three-layer neural network with a  regression loss   $\min_{W_1,W_2,W_3}\| Y - W_3 \sigma(W_2 \sigma(W_1 X))\|_2$. The dimension of weight matrices are $W_3 \in \R^{8 \times 7}$, $W_2 \in \R^{7 \times 6}$, and $W_1 \in \R^{6 \times 5}$. $X \in \R^{5 \times 4}$ is the data, $Y \in \R^{8 \times 4}$ is the target, and $\sigma$ is the LeakyReLU activation with slope coefficient 0.1. Additional hyperparameter details can be found in Appendix \ref{appendix:exp-multilayer}. 

Teleportation of the $\mathrm{GL}(\R)$ group can be performed by finding an element $x$ in its Lie algebra, such that transforming $\vw$ by the group element $g=\exp(x)$ improves the gradient $\frac{d\L}{dt}$. We use the first order approximation of the exponential map. Between each pair of weight matrices, we replace $g_m$ by $I+T$ and $g_m^{-1}$ by $I-T$ in \eqref{eq:Wm-g-action-main}, where $T \in \R^{d_m \times d_m}$ is initialized to 0. Then we perform gradient ascent steps on $T$ with objective defined in Line 3 of Algorithm \ref{alg:teleport} and update the pair of weights. 

Figure \ref{fig:two-layer}a and \ref{fig:two-layer}b show the training curves plotted against epochs and time. Shaded area denotes one standard deviation from 5 runs. Since GD and AdaGrad use different learning rates, they are not directly comparable. However, the addition of teleportation clearly improves both algorithms. Figure \ref{fig:two-layer}c and \ref{fig:two-layer}d shows the squared norm of gradient from a single run. Teleportation increases the magnitude of gradient, and the trajectory with teleportation has a larger $d\L/dt$ value at the same loss values, which demonstrates that teleportation finds a better trajectory.

\begin{figure}[t!]
\centering
\ \ \ a\hfill b \hfill c\hfill d \hfill ~ \\
\includegraphics[width=0.245\textwidth]{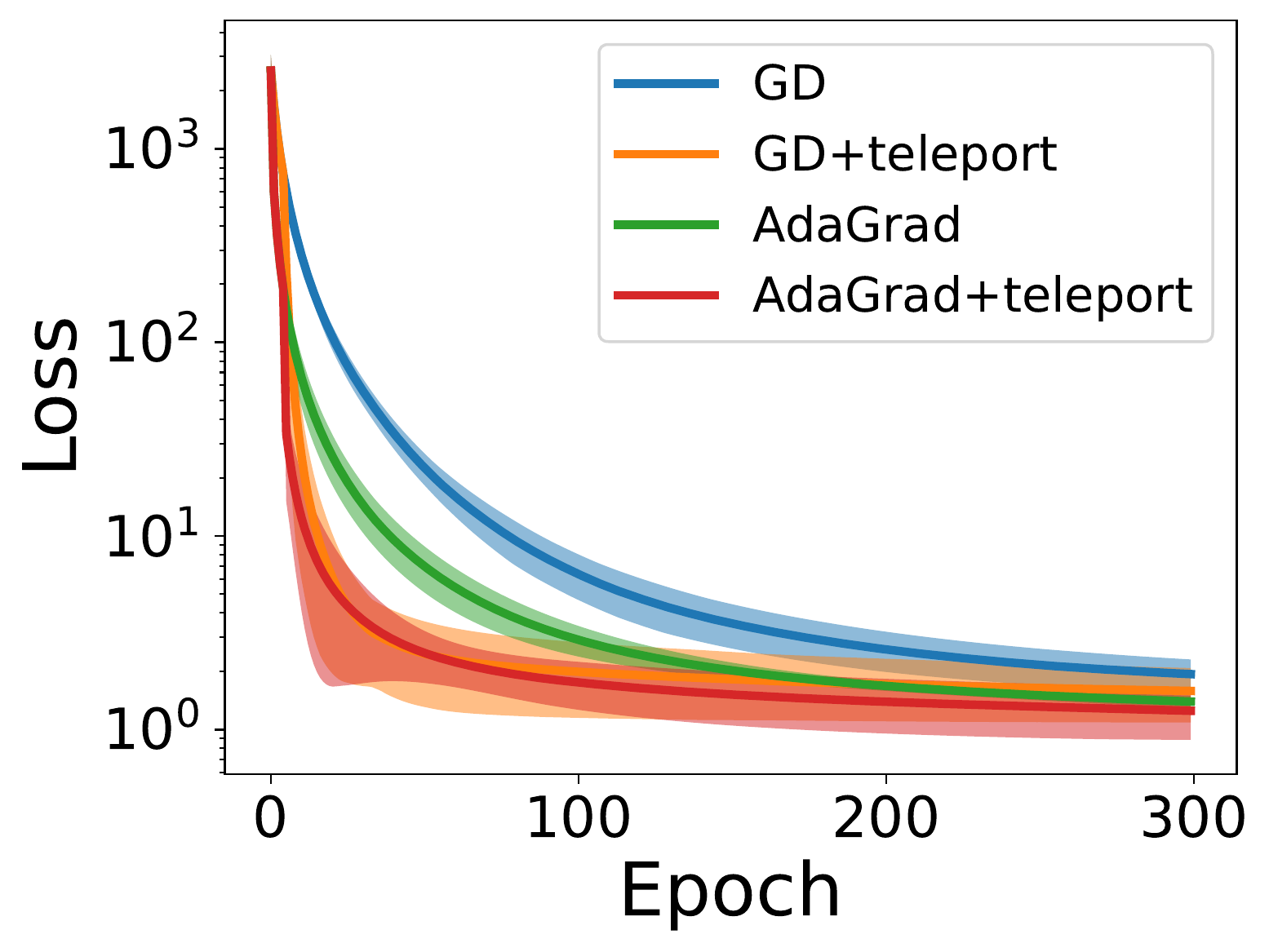}
\includegraphics[width=0.245\textwidth]{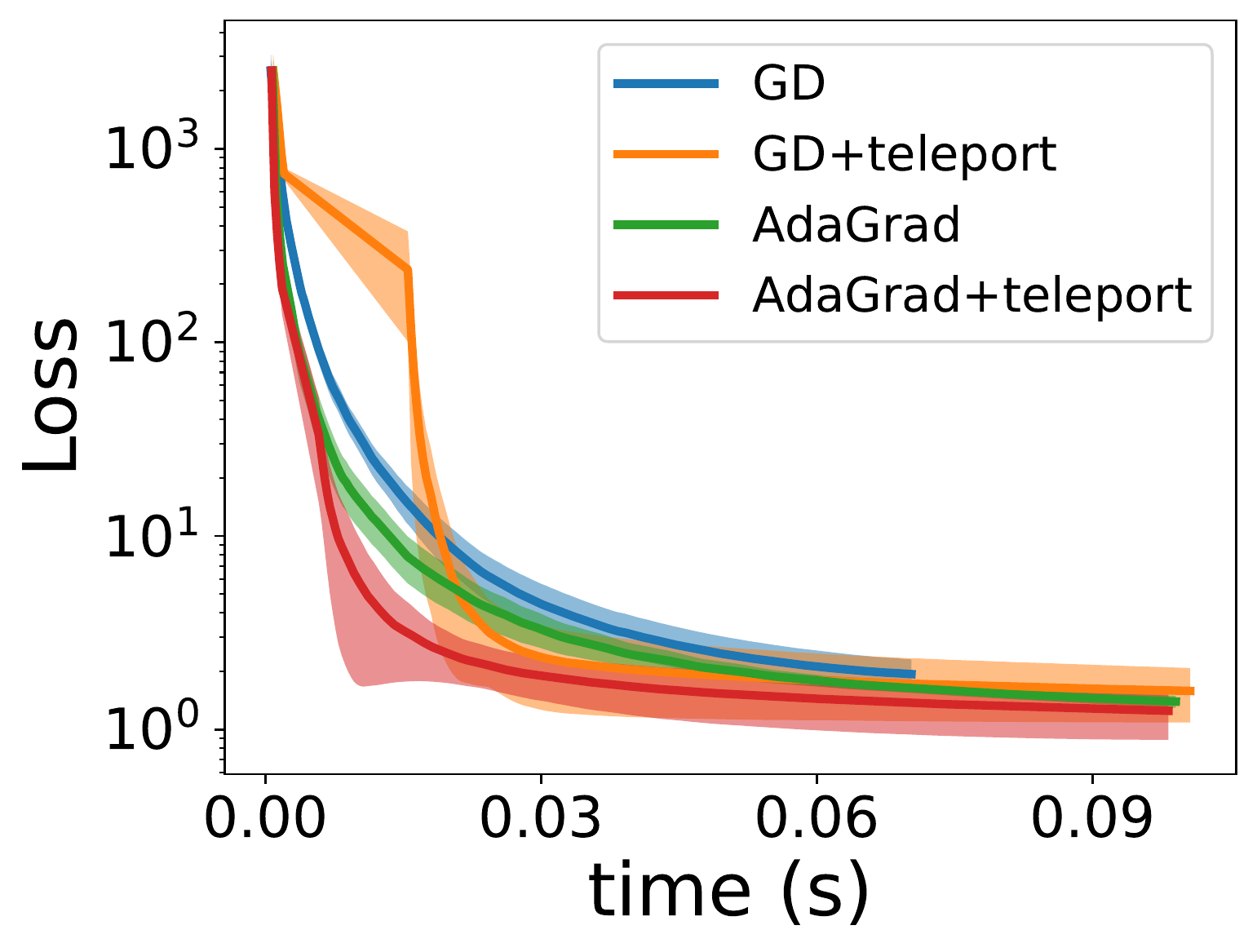}
\includegraphics[width=0.245\textwidth]{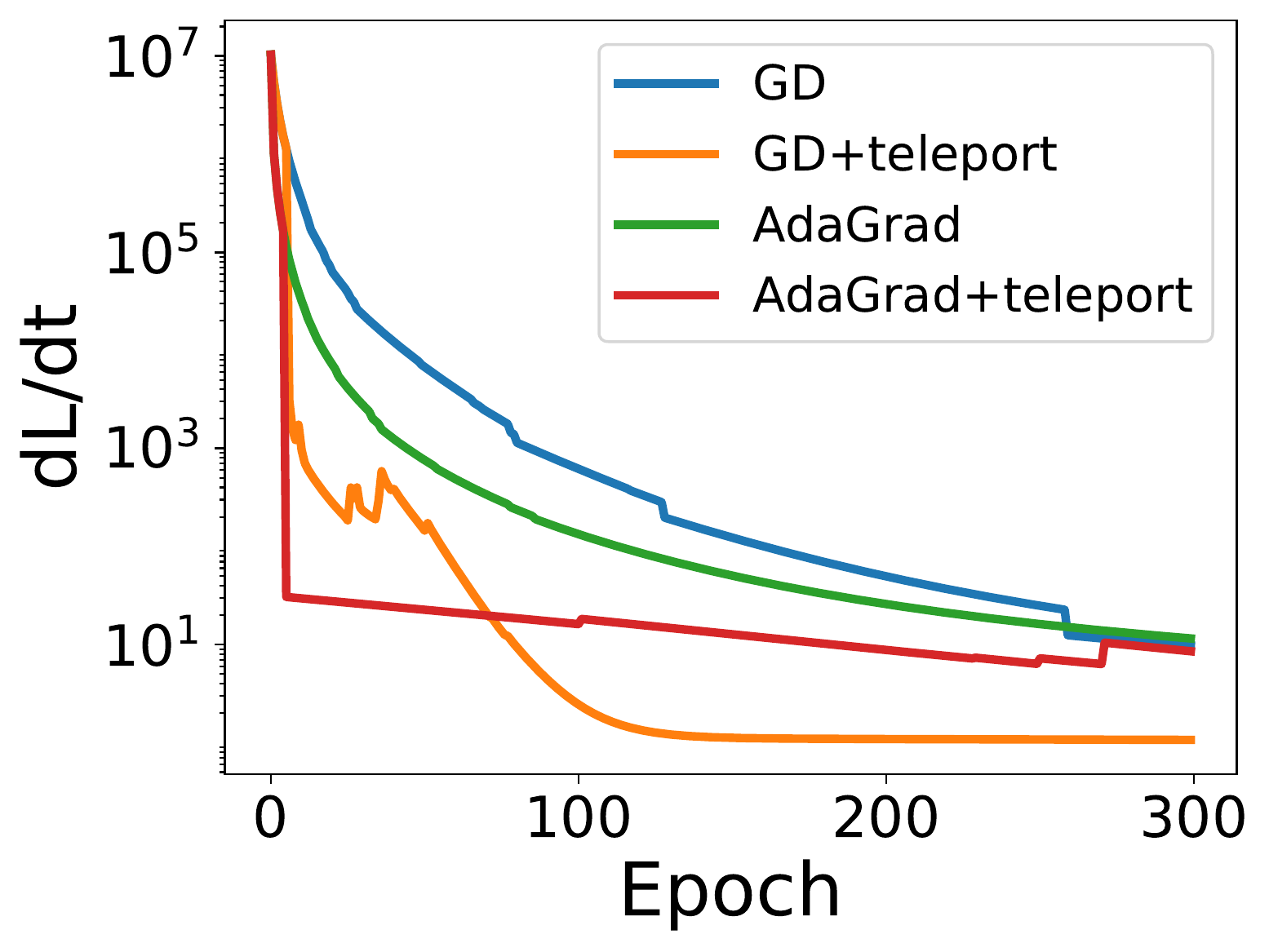}
\includegraphics[width=0.245\textwidth]{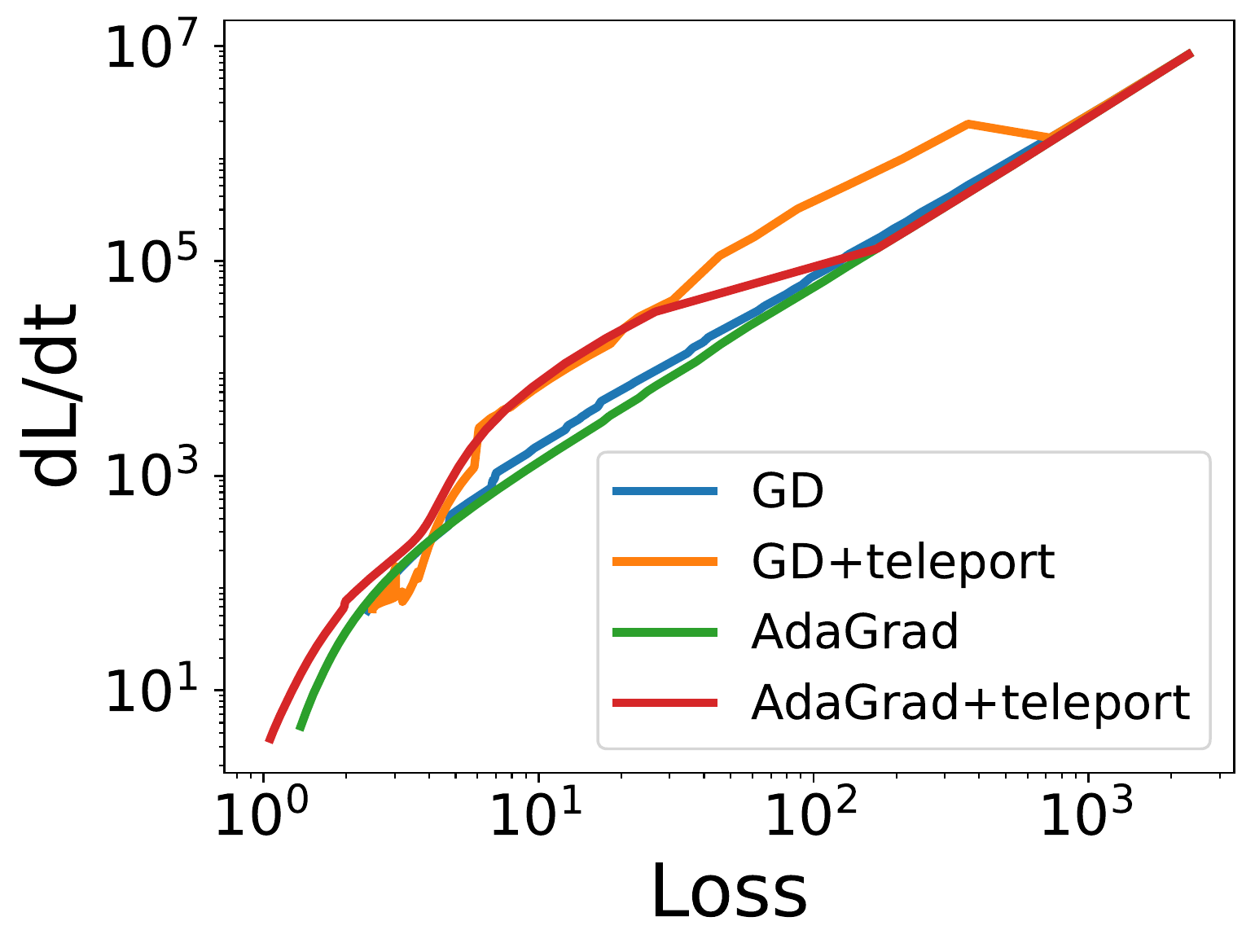}
\caption{Multi-layer network optimization  using gradient descent with and without teleportation. 
Teleportation reduces both the number of epochs and the total computational time required to reach convergence.
At the same loss values, the teleported version has a larger gradient.  }
\label{fig:two-layer}
\end{figure}

\out{
\begin{figure}[t!]
\centering
\ \ \ a\hfill b \hfill c\hfill ~ \\
\includegraphics[width=0.32\textwidth]{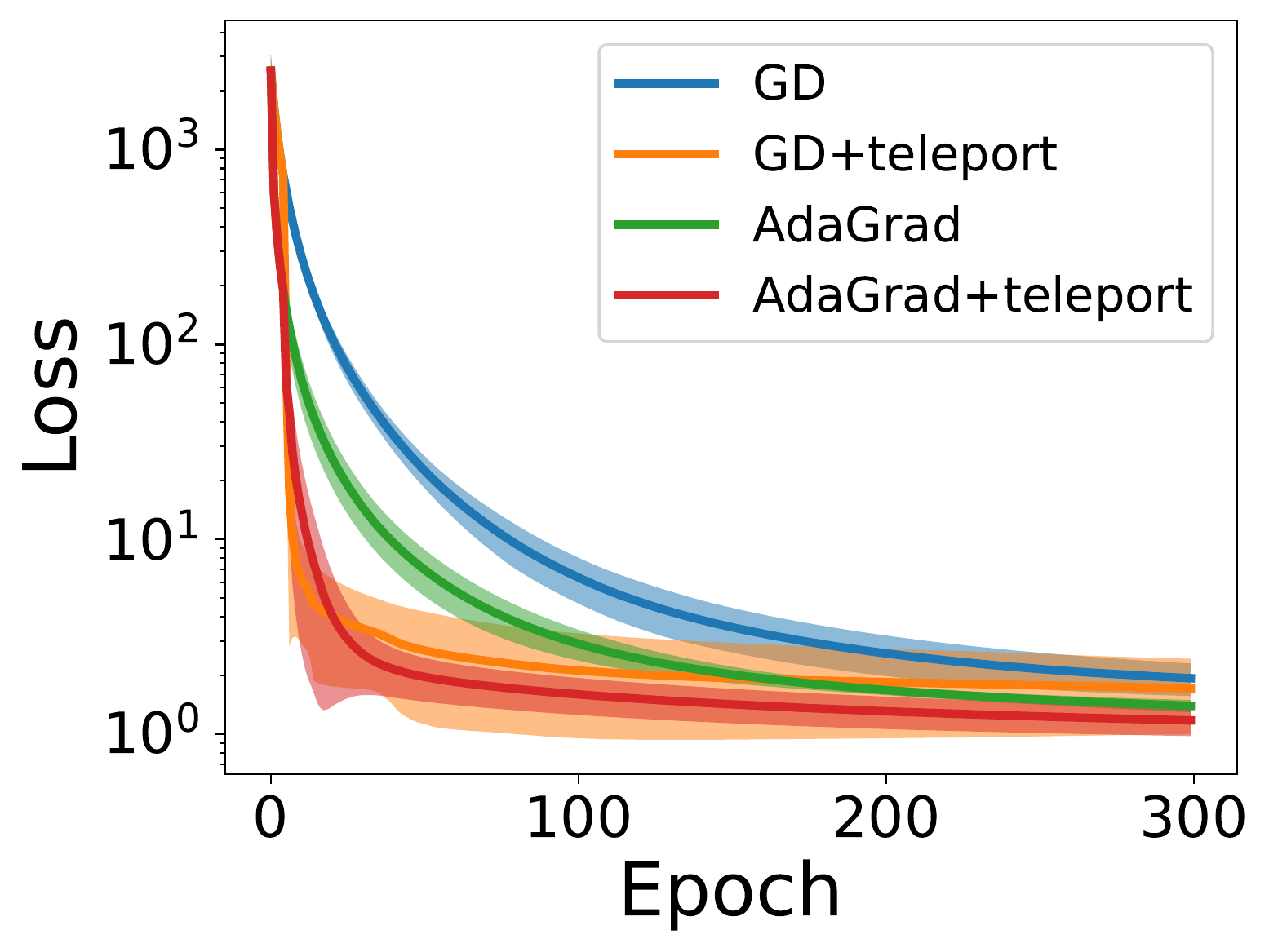}
\includegraphics[width=0.32\textwidth]{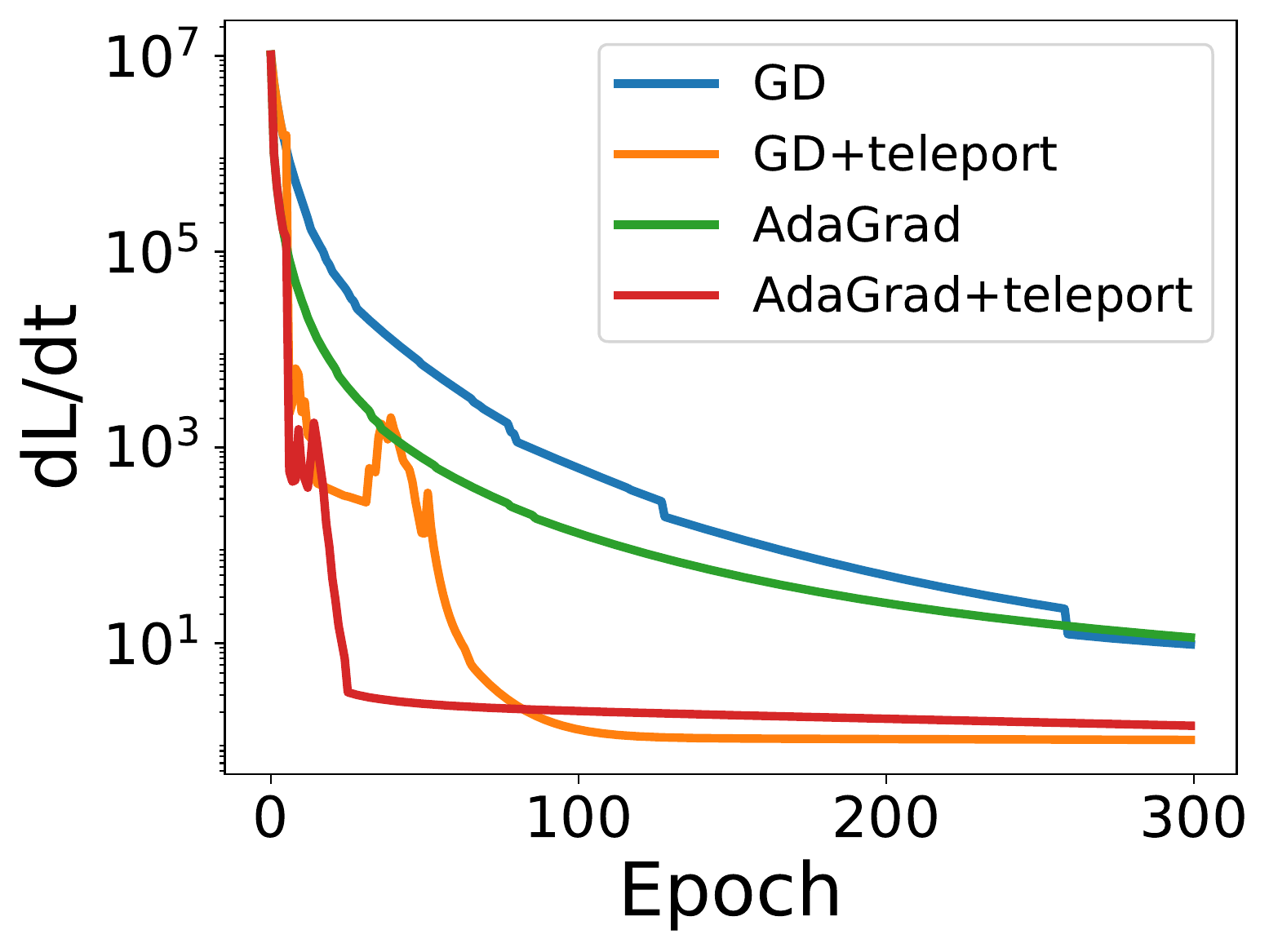}
\includegraphics[width=0.32\textwidth]{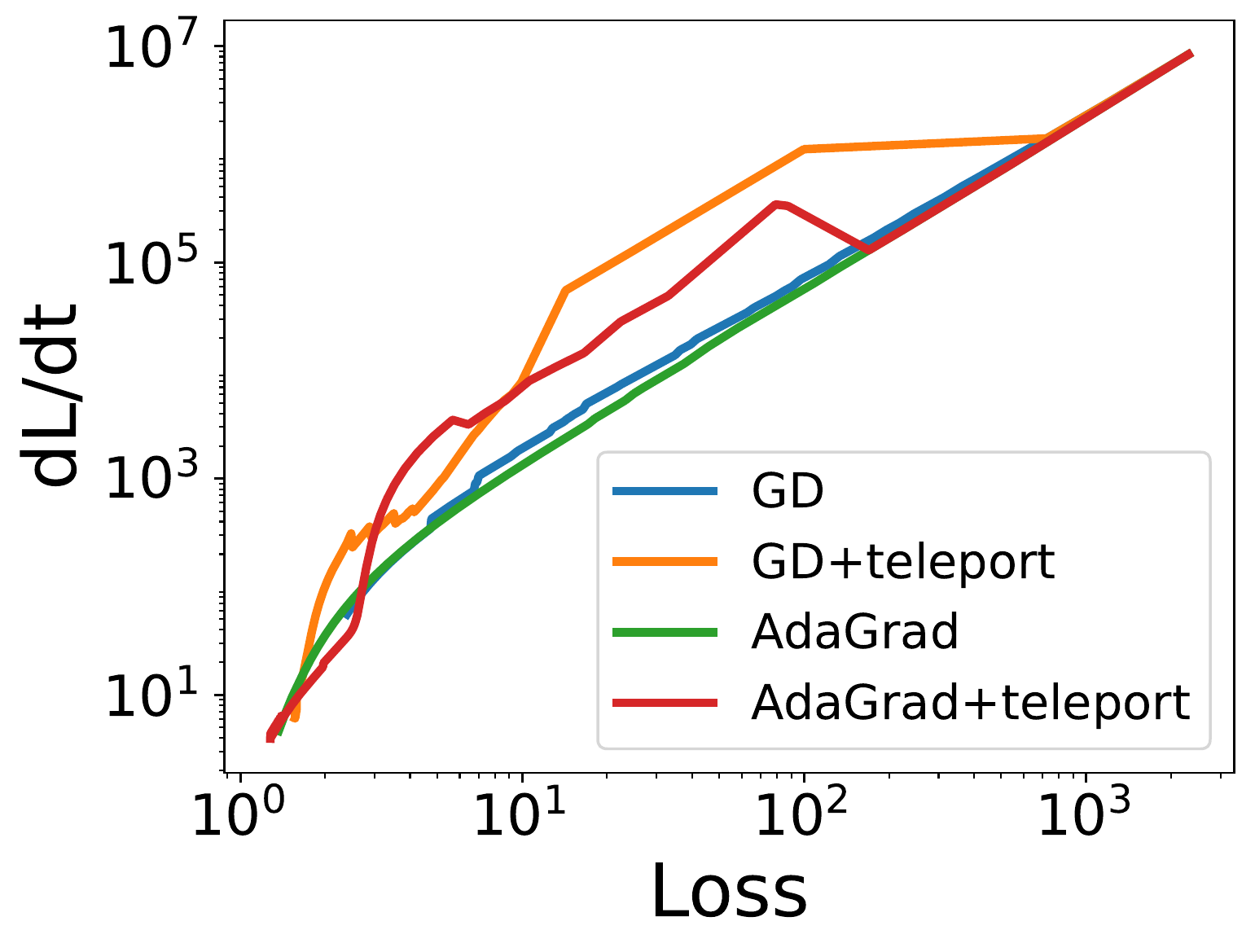}
\caption{Multi-layer network optimization  using gradient descent with and without teleportation. Loss $\L$ and convergence rate $d\L/dt$ are shown in (a) and (b). Teleportation improves the overall convergence rate. The gradient stays large after teleportation.  }
\label{fig:two-layer}
\end{figure}
}

\paragraph{MNIST classification.}
We apply symmetry teleportation on the MNIST classification task \citep{deng2012mnist}. We split the training set into 48,000 for training and 12,000 for validation. The input data has dimension $28 \times 28$ and is flattened into a vector. The output of the neural network has dimension $10$ corresponding to the 10 digit classes.  We used a three-layer neural network with hidden dimension $[512, 512]$, LeakyReLU activations, and cross-entropy loss. Learning rate is $2 \times 10^{-3}$, and learning rate for teleportation is $10^{-3}$. Each optimization algorithm is run 80 epochs with batch size of 20. Immediately after the first epoch, we apply teleportation using data from one mini-batch, and repeat for 4 different mini-batches. For each mini-batch, 10 gradient ascent steps are used to optimize $g_m$.

Figure \ref{fig:mnist}a,b shows the effect of teleportation on training and validation loss, and Figure \ref{fig:mnist}c,d shows the norm of the gradient in training. 
While teleportation significantly accelerates the decrease of the training loss in SGD, its effect on the validation loss is limited and detrimental for AdaGrad. 
Therefore, teleportation on MNIST makes training faster at the beginning but leads to earlier overfit and slightly worse validation accuracy. A possible reason is that regions with large gradients have sharp minima that do not generalize well. 

\begin{figure}[t!]
\centering
\ \ \ a\hfill b \hfill c\hfill d\hfill ~ \\
\includegraphics[width=0.24\textwidth]{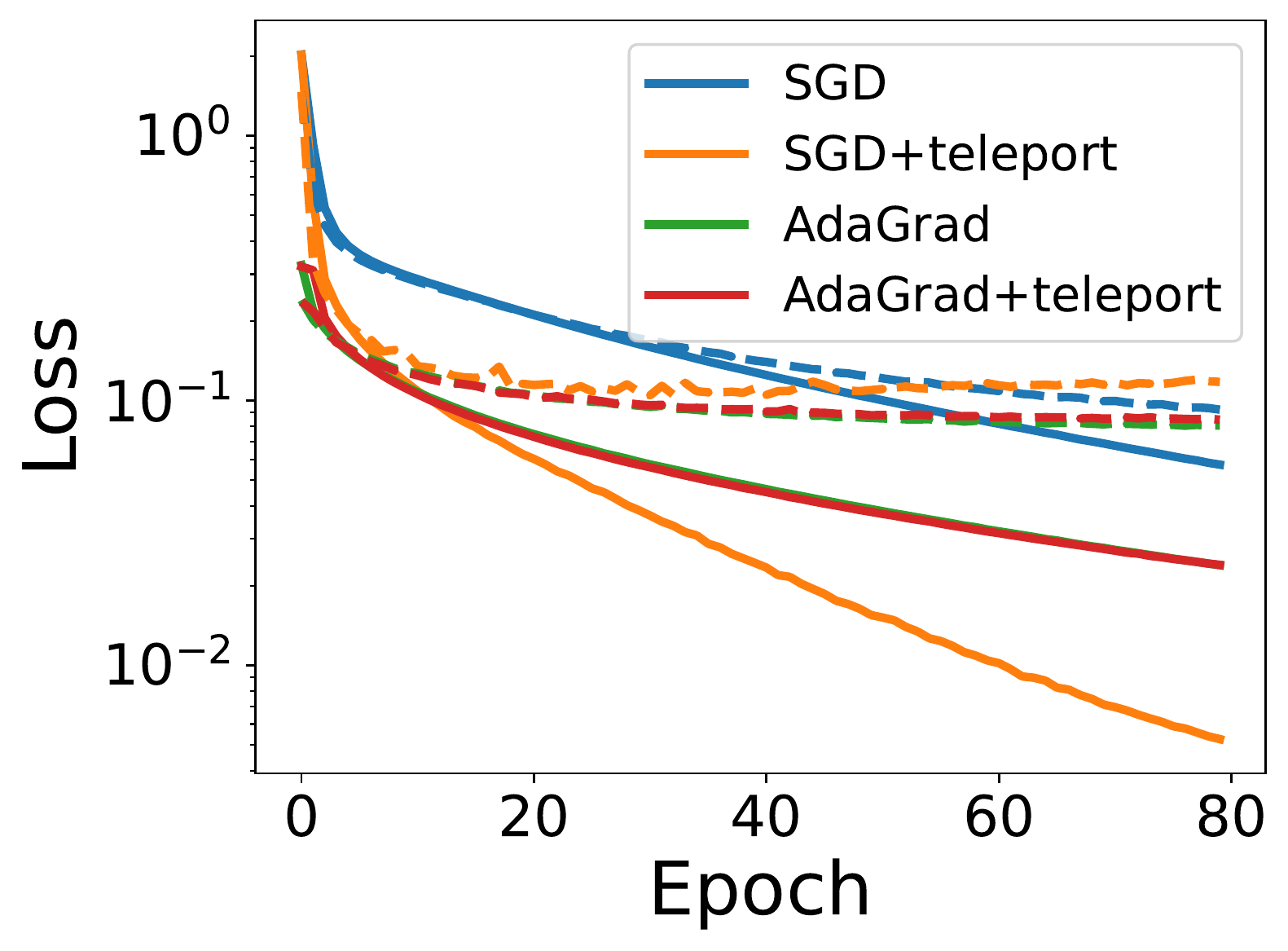}
\includegraphics[width=0.24\textwidth]{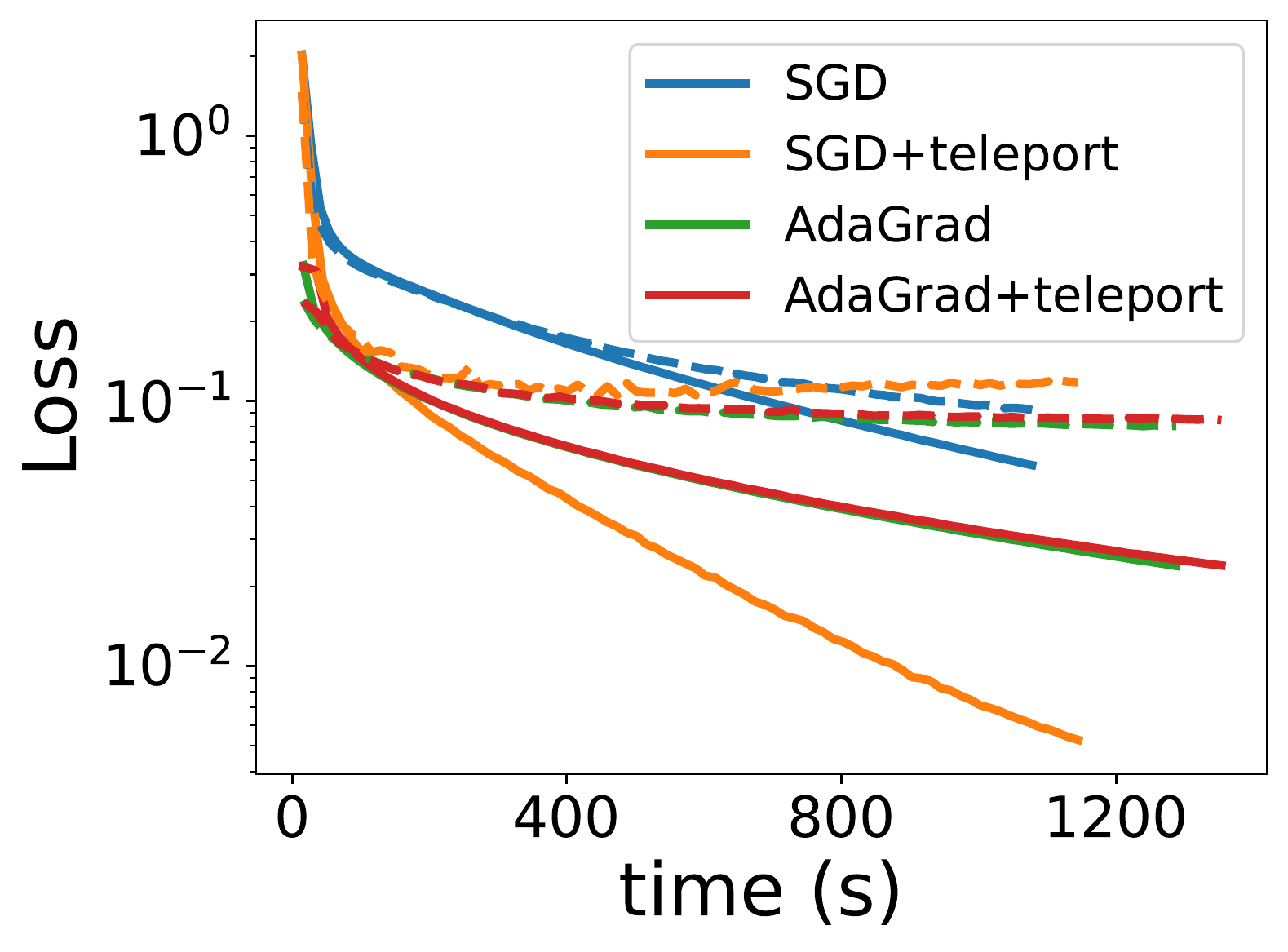}
\includegraphics[width=0.24\textwidth]{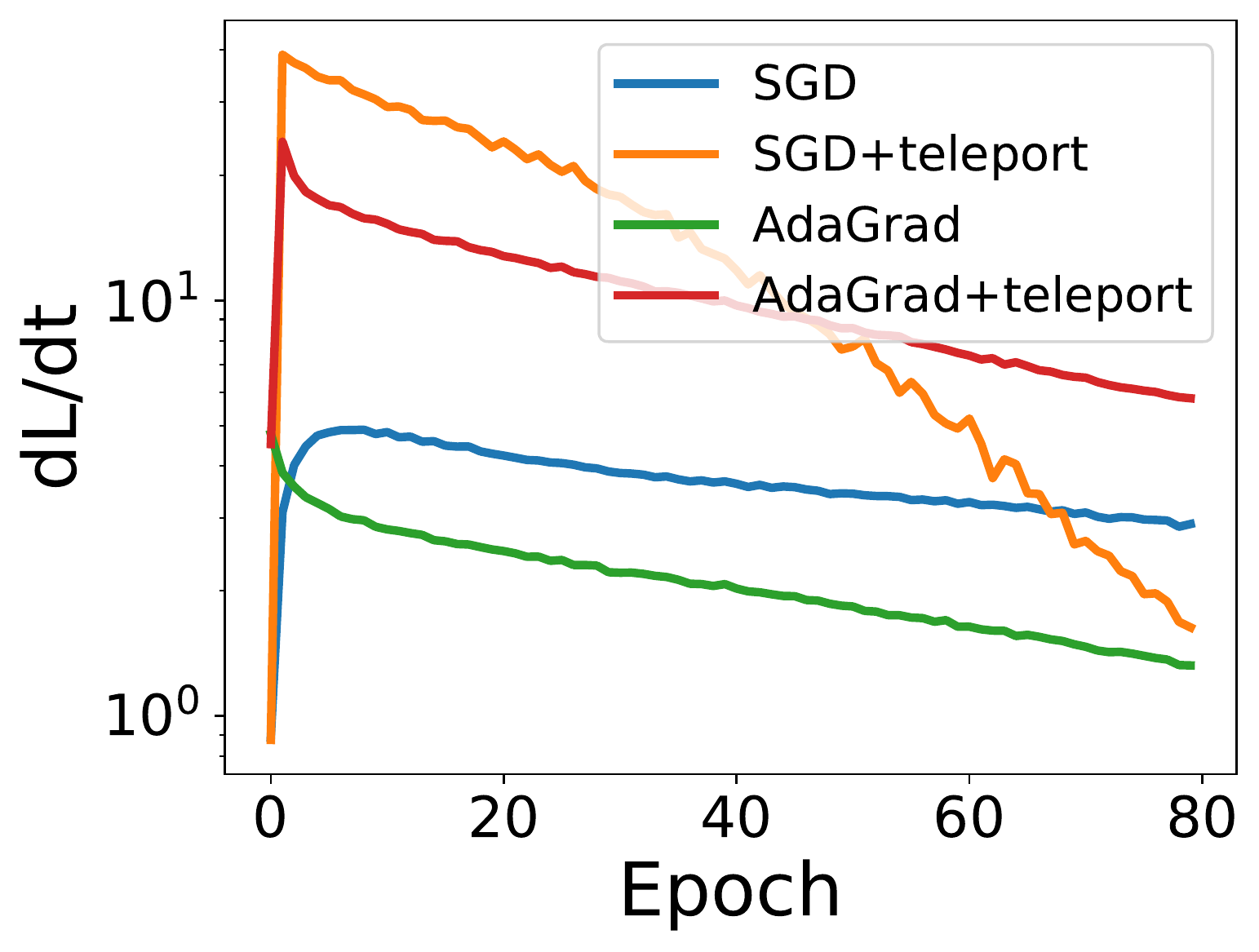}
\includegraphics[width=0.24\textwidth]{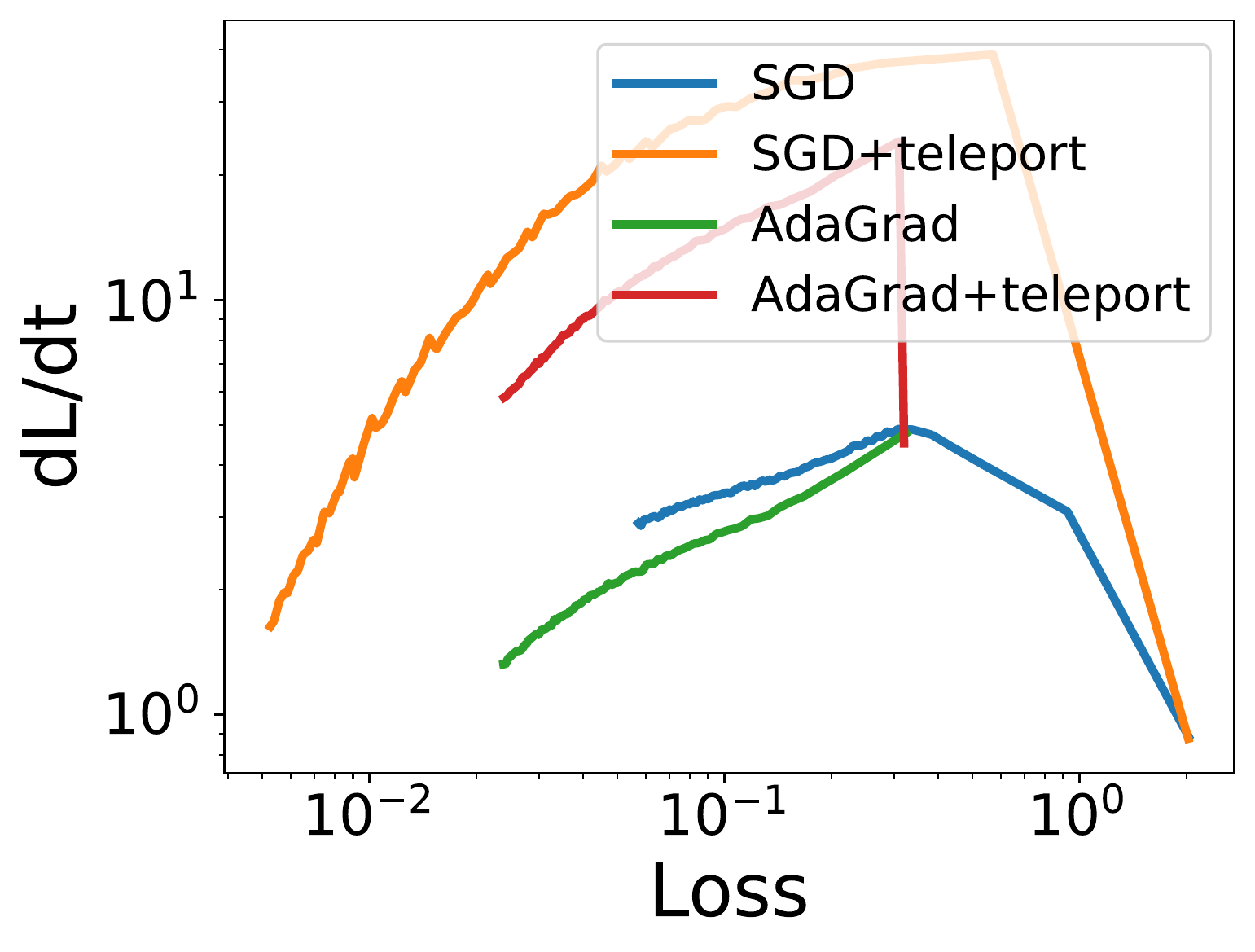}
\caption{MNIST classification using gradient descent with and without teleportation. Solid lines are training loss and dashed lines are validation loss.}
\label{fig:mnist}
\vspace{-3mm}
\end{figure}

\subsection{Teleportation schedule}
The effect of teleportation varies depending on its time and frequency. Figure \ref{fig:schedule} shows the result of teleportation on MNIST with different hyperparameters. In all experiments, we use SGD with batch size 20, learning rate $2 \times 10^{-3}$, and 10 gradient ascent steps for each teleportation. 

In Figure \ref{fig:schedule}a, we randomly select 4 different mini-batches and apply teleportation on each of them individually, but at different epochs. Teleportation before training has the worst performance. After epoch 0, the effect of teleportation is stronger when it is applied earlier. In Figure \ref{fig:schedule}b, we again apply teleportations on 4 randomly selected mini-batches, but repeat this with 5 different teleportation schedules (hyperparameter $K$ in Algorithm \ref{alg:teleport-sgd}) as shown in the legend. All schedules have the same number of teleportations. Smaller intervals between teleportations accelerate convergence more significantly. In Figure \ref{fig:schedule}c, we apply teleportation immediately after the first epoch, but use different numbers of mini-batches (hyperparameter $B$ in Algorithm \ref{alg:teleport-sgd}) and teleport using each of them individually. Using more mini-batches to teleport leads to faster decrease in training loss but is also more prone to overfitting.

\begin{figure}[t!]
\centering
\ \ \ a\hfill b \hfill c\hfill ~ \\
\includegraphics[width=0.32\textwidth]{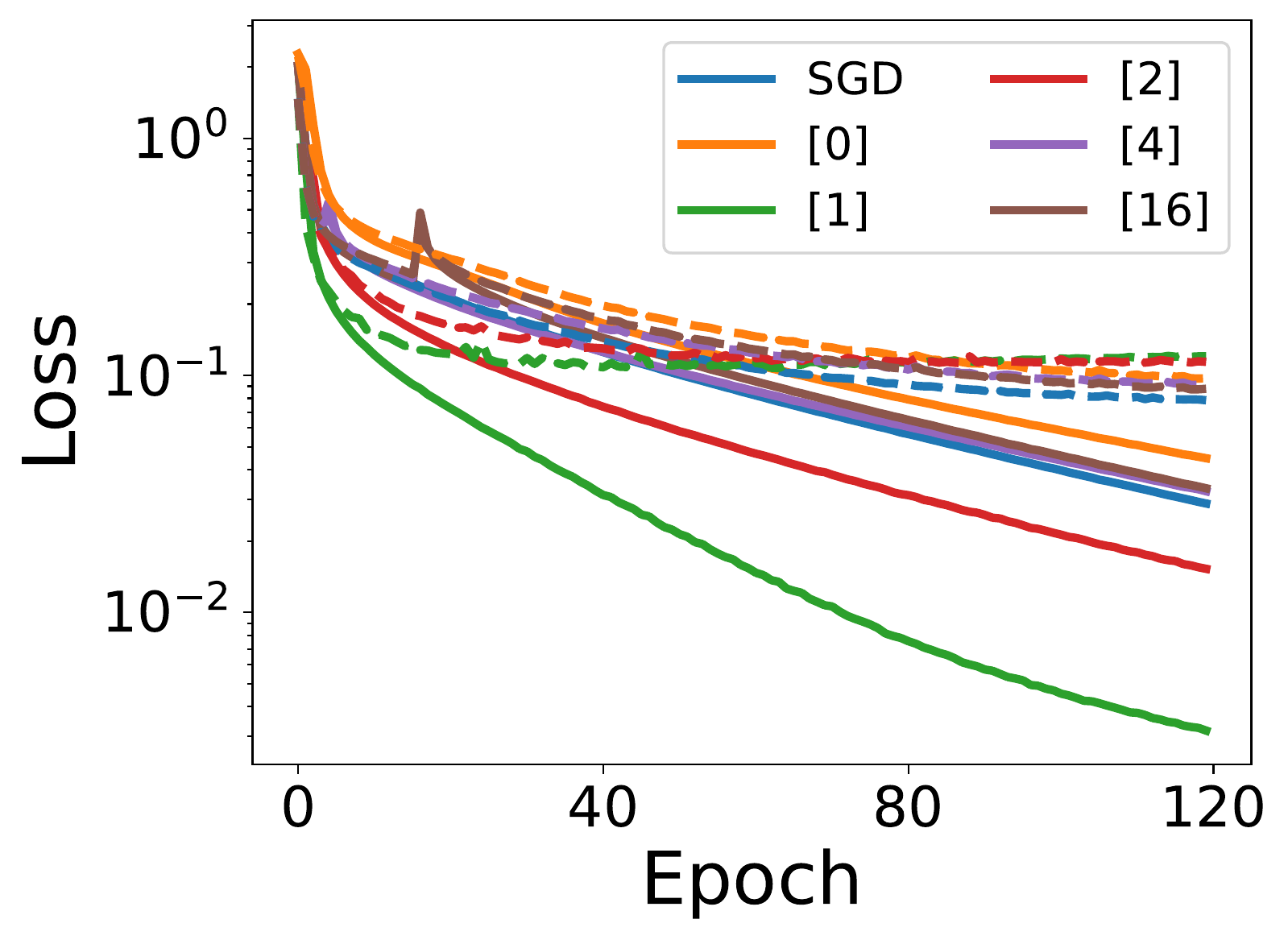}
\includegraphics[width=0.32\textwidth]{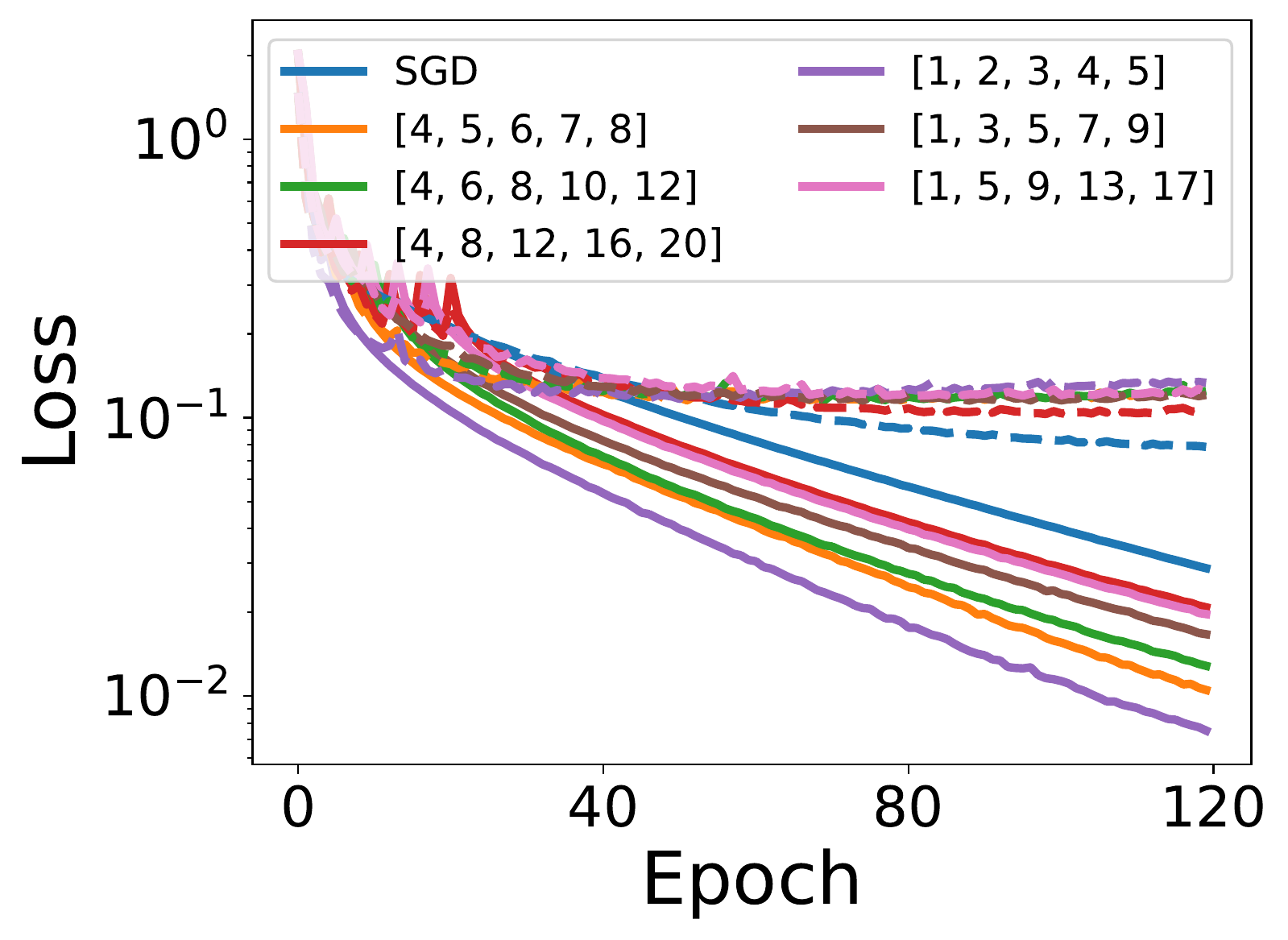}
\includegraphics[width=0.32\textwidth]{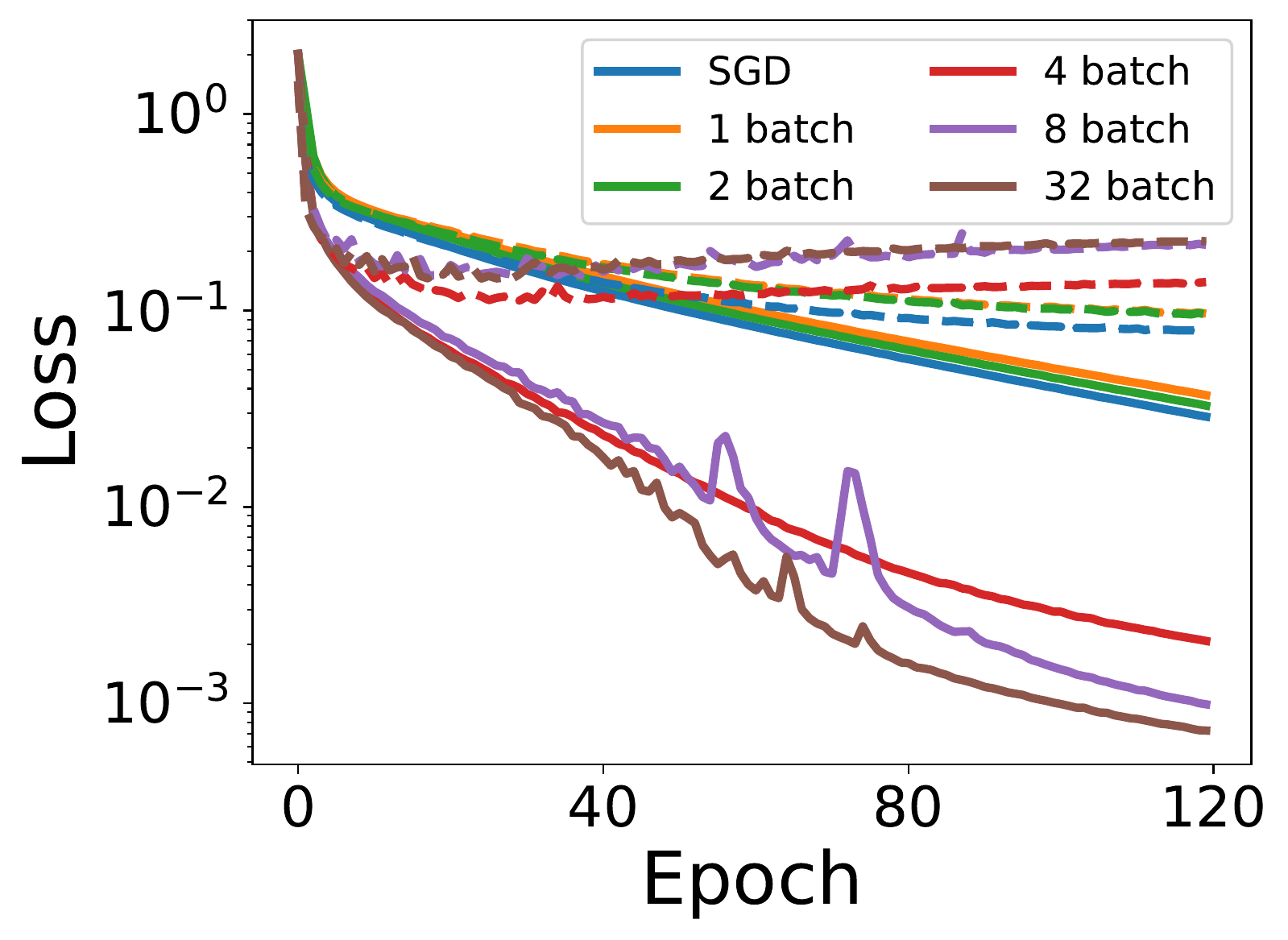}
\caption{Teleportation (a) once at different epoch, (b) 5 times with different teleportation schedules, and (c) using different number of mini-batches. The lists in the legend of (a) and (c) denote the epoch numbers in teleportation schedule where teleportation happens.}
\label{fig:schedule}
\vspace{-3mm}
\end{figure}

\subsection{Runtime analysis}
The additional amount of computation introduced by teleportation depends on the implementation of Line 2-5 of Algorithm \ref{alg:teleport}. We discuss our implementation for teleporting multi-layer neural networks as an example. Teleportating each pair of adjacent weight matrices requires computing the inverse of the output from the previous layer. 
Denote the batch size as $n$, the largest dimension of all layers as $d_{max}$, and the number of layers as $l$. Assume that $d_{max} > n$.  Computing the inverse of the output of each layer has complexity $O(d_{max}^2 n)$, and computing the pseudoinverse for all layers has complexity $O(d_{max}^2 nl)$. Note that all matrices we invert have dimensions at most $d_{max} \times n$. 

For one gradient ascent step on $g$, the forward and backward pass both have complexity $O(d_{max}^2 nl)$. This is the same as the forward and backward pass of gradient descent on $\vw$ because the architecture is the same except with approximately twice as many layers.
We perform $t$ gradient ascent steps on $g$. 
Therefore, the computation cost for one teleportation is $O(d_{max}^2 nlt)$. 

We show empirically that the runtime for teleportation scales polynomially with matrix dimensions and linearly with the number of layers. We record the runtime of gradient descent on a Leaky-ReLU neural network for 300 epochs, with a 10-step teleportation every 10 epochs. Each pair of adjacent weight matrices is transformed by a group action during teleportation. Figure \ref{fig:multi-layer-time}(a) shows the runtime for a 3-layer network using square weights and data matrices with different dimensions. Figure \ref{fig:multi-layer-time}(b) shows the runtime for 128-by-128 weight and data matrices with different number of layers. 

Although the teleportation step has the same complexity as a gradient descent step, the runtime is dominated by teleportation due to larger constants in the complexity analysis. The trade-off between teleportation time and convergence rate depends on specific problems. In our experiments, the convergence rate is improved by a small number of teleportation steps which does not add significant computational overhead.

\begin{figure}[t!]
\centering
~\hfill  \ \ a\hfill ~\hfill \ \ b \hfill ~ \hfill~ \hfill ~ \\
\includegraphics[width=0.32\textwidth]{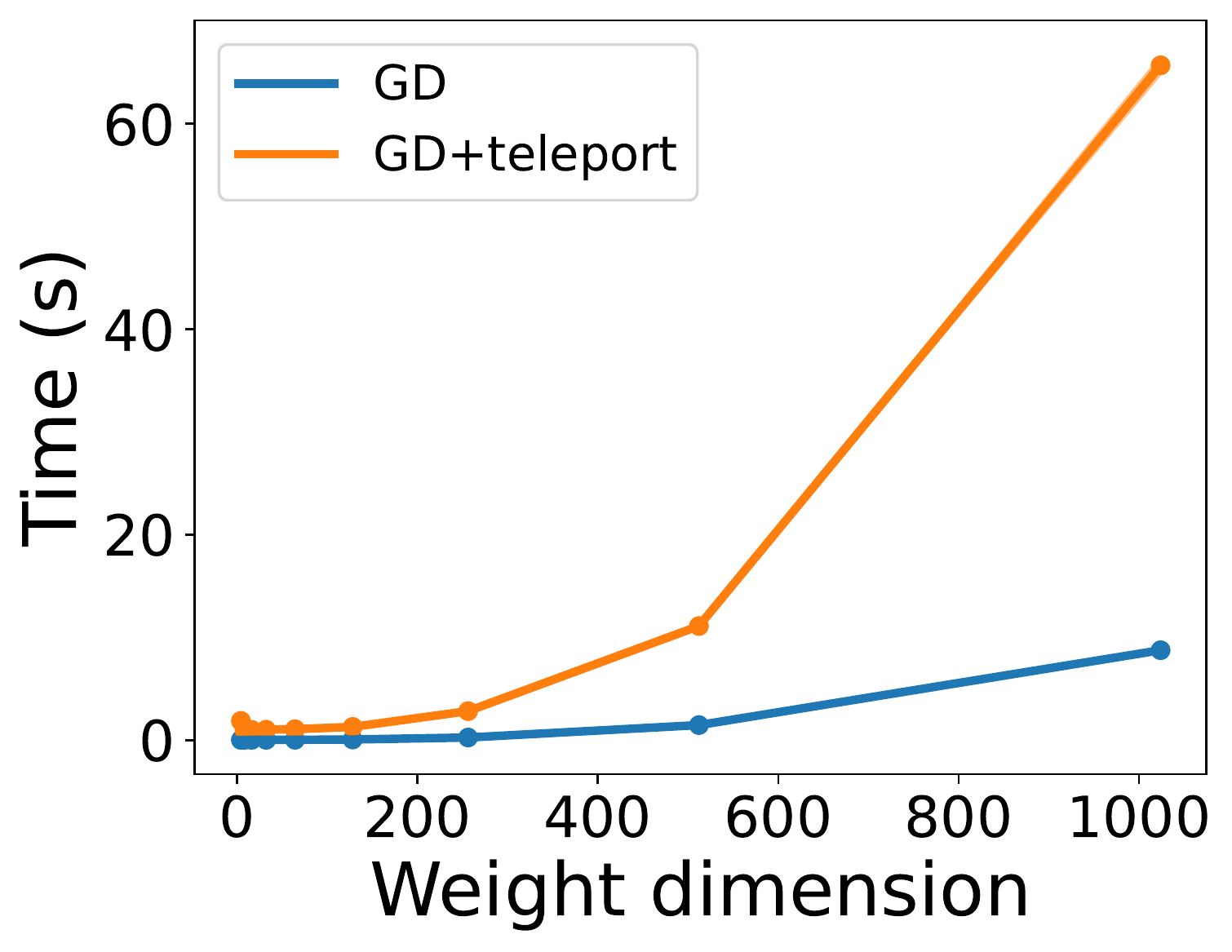}
\includegraphics[width=0.32\textwidth]{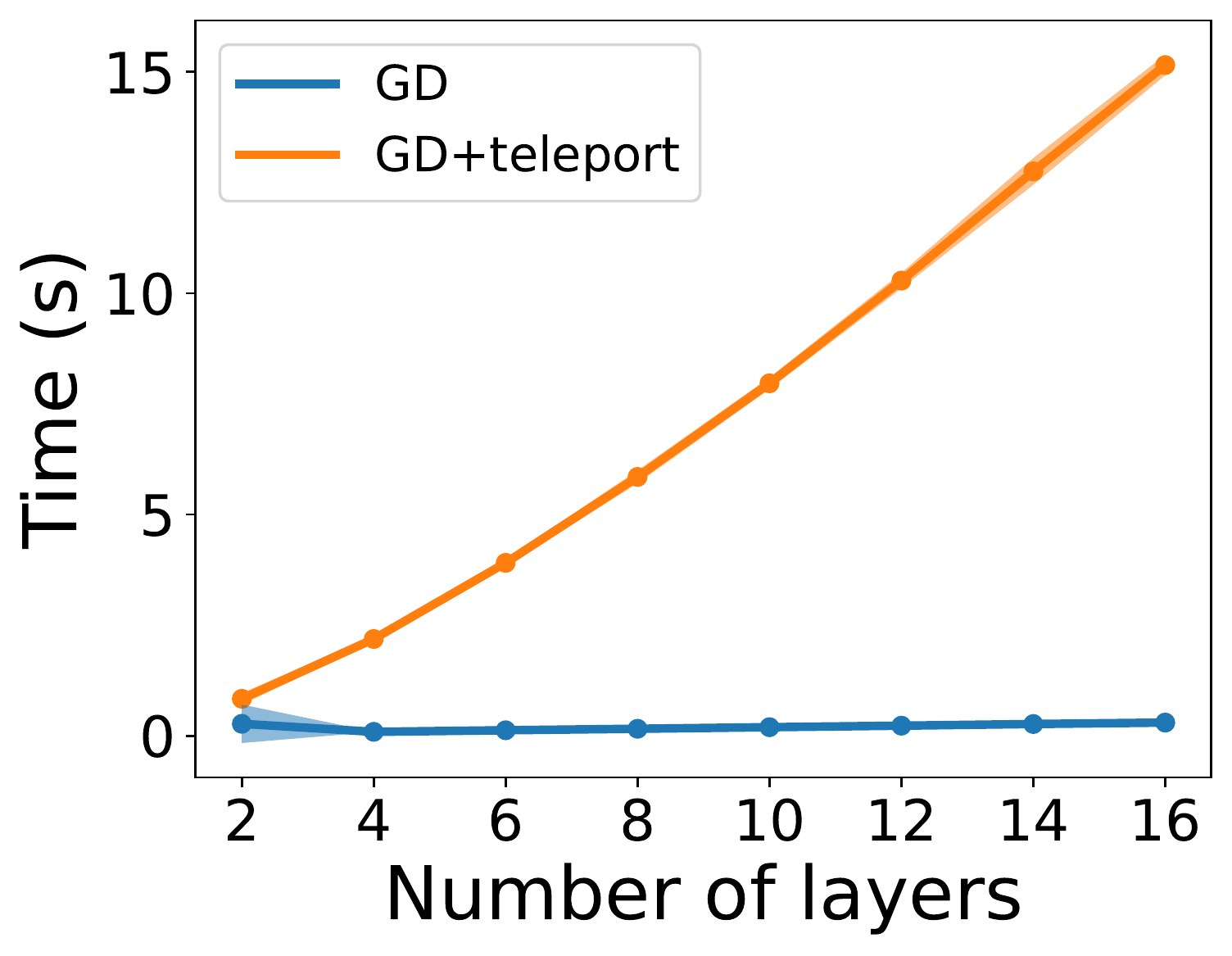}
\caption{Gradient descent training time on a Leaky-ReLU neural network for 300 epochs, with a 10-step teleportation every 10 epochs. (a) 3-layer network using square weights and data matrices, with different dimensions. (b) Weights and data are all 128 by 128 matrices, but the network has different number of layers. }
\label{fig:multi-layer-time}
\vspace{-3mm}
\end{figure}

%% file: secs/conclusion.tex
\section{Discussion and Conclusions}
We proposed a new optimization algorithm that exploits the symmetry in the loss landscape to improve convergence. It is interesting to note that optimizing $d\L/dt$ locally sometimes leads to an improved global path. One example is the ellipse function (Figure \ref{fig:sales_pic}), where teleporting once ensures that $d\L/dt$ is optimal at every $\L$ value along the trajectory.
Another example is the matrix factorization problem $\L(U, V) = \|Y-UV\|_2^2$. \cite{tarmoun2021understanding} shows that the convergence rate increases with the imbalance $U^TU - V^TV$. Consider the transformation $U, V \xrightarrow[]{} Ug, g^{-1}V$. To optimize $d\L(U,V)/dt$ locally, we would need a large $g \in \mathrm{GL}_n$, but a large $g$ also leads to a large $U^TU - V^TV$ which is positively correlated with the overall convergence rate of the entire trajectory. Hence teleportation is guaranteed to produce a better trajectory. 

A potential future direction is to derive the exact expression for how teleportation affects the loss value at a later time in gradient flow, which may lead to a closed-form solution of the optimal teleportation destination. Additionally, inspired by the landscape view of parameter space symmetry \cite{simsek2021geometry}, teleportation using discrete (permutation) symmetries may allow us to reach a better minimum. Finally, additional theory can be developed to explain the relationship between teleportation and second-order methods under the approximations we introduced, especially to quantify the improvement on the overall convergence rate, and to derive its effect on generalization bounds. Integrating teleportation with other advanced optimizers such as Adam and RMSprop would be another interesting future step.


%% file: secs/acknowledgement.tex
\begin{ack}
This work was supported in part by U.S. Department Of Energy, Office of Science, U. S. Army Research Office
under Grant W911NF-20-1-0334, Google Faculty Award, Amazon Research Award, and NSF Grants \#2134274, \#2107256 and \#2134178.
R. Walters is supported by the Roux Institute and the Harold Alfond Foundation.
We are grateful to Iordan Ganev for many helpful discussions. 

\end{ack}

%% file: secs/appendix.tex
\newpage
\section{Adaptations of Algorithm \ref{alg:teleport} for different problems}
\subsection{Stochastic gradient descent}
\label{appendix:algorithm-sgd}
We extend Algorithm \ref{alg:teleport} to stochastic gradient descent (SGD). We apply the group actions using data from a mini-batch $X_i$, and repeat for $B$ mini-batches each time. The gradient we optimize, $\Tilde{\grad} \L(X_i, g \cdot \vw_t)$, also uses single mini-batches. Algorithm \ref{alg:teleport-sgd} provides the framework for teleportation in SGD.
\begin{algorithm}
    \SetKwInput{Input}{Input}
    \SetKwInput{Output}{Output}

    \Input{Loss function $\L(\vw)$, learning rate $\eta$, number of epochs $t_{max}$, initialized parameters $\vw_0$, symmetry group $G$, teleportation schedule $K$, number of mini-batches used to teleport $B$.}
    \Output{$\vw_{t_{max}}$.}

    \For{$t \leftarrow 0$ \KwTo $t_{max}-1$}{
        \eIf{$t \in K$}{
            \For{$X_i$ in the first $B$ mini-batches}{
                $g \leftarrow \text{argmax}_{g \in G} \| \Tilde{\grad} \L(X_i, g \cdot \vw_t) \|^2$ \\
                $\vw_t \leftarrow g \cdot \vw_t$ \\
                $\vw_{t} \leftarrow \vw_t - \eta \Tilde{\grad} \L(X_i, \vw_t)$
            }
            \For{$X_i$ in the rest mini-batches}{
                $\vw_{t} \leftarrow \vw_t - \eta \Tilde{\grad} \L(X_i, \vw_t)$
            }
        }{
            \For{all mini-batches $X_i$}{
                $\vw_{t} \leftarrow \vw_t - \eta \Tilde{\grad} \L(X_i, \vw_t)$
            }
        }
       $\vw_{t+1} \leftarrow \vw_t$ \\
        
    }

    \KwRet{$\vw_{t_{max}}$}
    \caption{Symmetry Teleportation (SGD)}
    \label{alg:teleport-sgd}
\end{algorithm}

\subsection{Data transformation}
\label{appendix:algorithm-data-transform}
Algorithm \ref{alg:teleport-data} here modifies Algorithm \ref{alg:teleport} to allow transformations on both parameters and data. Denote $g_X$ as the group action on data only. The group actions on data at all teleportation steps can be precomposed as a function $f$ and applied to the input data at inference time.
\begin{algorithm}
    \SetKwInput{Input}{Input}
    \SetKwInput{Output}{Output}
    \SetKwInput{Initialize}{Initialize}

    \Input{Loss function $\L(\vw, X)$, learning rate $\eta$, number of epochs $t_{max}$, initialized parameters $\vw_0$, symmetry $G$, teleportation schedule $K$, data $X$.}
    \Output{$\vw_{t_{max}}$, data transformation $f$.}
    \Initialize{$f = $ the identity function.}
    \For{$t \leftarrow 0$ \KwTo $t_{max}-1$}{
        \If{$t \in K$}{
            $g \leftarrow \text{argmax}_{g \in G} \| \grad_{g \cdot \vw_t} \L (g \cdot(\vw_t, X)) \|^2$ \\
            $\vw_t, X \leftarrow g \cdot (\vw_t, X)$ \\
            $f \leftarrow g_X \circ f$
        }
        $\vw_{t+1} \leftarrow \vw_t - \eta \grad_{\vw_t} \L$
    }

    \KwRet{$\vw_{t_{max}}$, $f$}
    \caption{Symmetry Teleportation (with data transformation)}
    \label{alg:teleport-data}
\end{algorithm}

\section{Group actions}
In this section, we derive the group actions for the test functions and multi-layer neural networks. More details about group theory can be found in textbooks such as \cite{Lang}.

\subsection{Continuous symmetry in test functions}
\subsubsection{Ellipse}

Consider the following loss function with $a \in \R^{\geq 0}$:
\begin{align}
    \L(x_1, x_2) &= x_1^2 + ax_2^2
\end{align}

If we change the variables to $\L(u(x_1, x_2), v(x_1, x_2)) = u^2 + v^2$, 2D rotations leave $\L$ unchanged. Therefore $SO(2)$ is a symmetry of $\L(x_1,x_2)$. Let $g_\theta \in SO(2)$, and define the group action as 
\begin{align}
    g_\theta \cdot \begin{bmatrix}
        x_1 \\
        x_2
    \end{bmatrix} = \begin{bmatrix}
        1 & 0 \\
        0 & 1/\sqrt{a}
    \end{bmatrix}
    \begin{bmatrix}
        \cos \theta & -\sin \theta \\
        \sin \theta & \cos \theta
    \end{bmatrix}
    \begin{bmatrix}
        1 & 0 \\
        0 & \sqrt{a}
    \end{bmatrix}
    \begin{bmatrix}
        x_1 \\
        x_2
    \end{bmatrix}
\end{align}
Then 
\begin{align}
    \L(x_1, x_2) = \L(g \cdot (x_1, x_2))
\end{align}

\subsubsection{Rosenbrock function}
\label{app:rosenbrock_symmetry}
Consider the Rosenbrock function with 2 variables \cite{rosenbrock1960automatic}:
\begin{align}
    \L(x_1, x_2) = 100(x_1^2 - x_2)^2 + (x_1 - 1)^2
\end{align}

Let $u = 10(x_1^2 - x_2)$ and $v = x_1 - 1$. After changing the variables from $x$ and $y$ to $u$ and $v$, $\L$ has a rotational symmetry. Note that the function, $h: \R^2 \xrightarrow[]{} \R^2$, that maps $x_1, x_2$ to $u, v$ is bijective: 
\begin{align}
    (u, v) &= h(x_1, x_2) = (10(x_1^2 - x_2), x_1 - 1) \cr
    (x_1, x_2) &= h^{-1}(u, v) = (v + 1, (v + 1)^2 - 0.1u) \cr
    h(x_1, x_2) &= h(y_1, y_2) \Rightarrow (x_1, x_2) = (y_1, y_2)
\end{align}

Next, we show that $SO(2)$ is a symmetry of $\L(x_1,x_2)$. Let $\rho$ be a representation of $SO(2)$ acting on $\R^2$. For $g \in SO(2)$, define the following group action:
\begin{align}
    g \cdot (x_1, x_2) = 
    h^{-1} \left( \rho(g) h
    (x_1, x_2) \right)
\end{align}
Then 
\begin{align}
    \L(x_1, x_2) = \L(g \cdot (x_1, x_2))
\end{align}

For the Rosenbrock function with 2N parameters, we can construct a bijective function $h:\R^{2N} \xrightarrow[]{} \R^{2N}$ by transforming each of the $N$ pairs of variables as before, and $SO(2N)$ is a symmetry of $\L(x_1,...,x_{2N})$. However, we will only use the 2 variable version in the experiments.

\subsubsection{Booth function}
\label{app:booth_symmetry}
Consider the Booth function \cite{jamil2013literature}:
\begin{align*}
    \L(x_1, x_2) = (x_1 + 2x_2 - 7)^2 + (2x_1 + x_2 - 5)^2
\end{align*}
Similar to the Rosebrock function, a change of variables reveals a rotational symmetry of $\L$:
\begin{align}
    (u, v) &= h(x_1, x_2) = (x_1 + 2x_2 - 7, 2x_1 + x_2 - 5) \cr
    (x_1, x_2) &= h^{-1}(u, v) = (-\frac{1}{3} u + \frac{2}{3} v + 1, \frac{2}{3} u - \frac{1}{3} v + 3) \cr
\end{align}
The function $h: \R^2 \xrightarrow[]{} \R^2$ that maps $x_1, x_2$ to $u, v$ is bijective. Let $\rho$ be a representation of $SO(2)$ acting on $\R^2$. For $g \in SO(2)$, define the following group action:
\begin{align}
    g \cdot (x_1, x_2) = 
    h^{-1} \left( \rho(g) h
    (x_1, x_2) \right)
\end{align}
Then $\L(x_1,x_2)$ admits an $SO(2)$ symmetry:
\begin{align}
    \L(x_1, x_2) = \L(g \cdot (x_1, x_2))
\end{align}

\subsection{Continuous symmetry in multi-layer neural Networks}
\label{appendix:mult-layer}
In this and the following sections, we provide proofs for the theoretical results. We restate the propositions from the main text for readability.
\begin{manualproposition}{\ref{thm:mult-layer}}
A linear network is invariant under all groups $G_m \equiv \mathrm{GL}_{d_m}$ acting as 
\begin{align*}
    g \cdot (W_{m}, W_{m-1}) &= (W_{m}g^{-1}, g W_{m-1}), &
    g \cdot W_{k} = W_k, \quad \forall k \notin \{m,m-1\}. 
\end{align*}
\end{manualproposition}

\begin{proof} In the linear network $h_m = W_m h_{m-1}$, hence
\begin{align}
    g \cdot (W_{m}, h_{m-1})= (W_{m}g^{-1}, g h_{m-1}), \quad 
    g\cdot h_m = W_{m}g^{-1} g h_{m-1} = h_m
\end{align}
which means a $p$-layer linear network is invariant under all $G_m$ with $m\leq p$ as they keep the output $h_p$ invariant ($\forall g\in G_m$, $g\cdot h_p= h_p$). 
\out{
For $m\geq 1$ this is a symmetry of the architecture and is independent of the  input $X$, defining 
\begin{align}
    g \cdot (W_{m}, W_{m-1}) = (W_{m}g^{-1}, g W_{m-1}),\quad  
    g\cdot h_m = W_{m}g^{-1} g W_{m-1}h_{m-2} = h_m
\end{align}
We define the group action as identity on the weights of all other layers, i.e. for  $g \in \mathrm{GL}_{d_{m-1}}$,  $g \cdot W_k = W_k$ when $k\notin \{ m, m-1\}$. 
For a $p$-layer network, we have $g\cdot h_p =h_p $, meaning the entire network is invariant under this group action.
}
\end{proof}

\begin{manualproposition}{\ref{thm:mult-layer-nonlinear}}
Assume that $h_{m-2}$ is invertible.
A multi-layer network with bijective activation $\sigma$ has a $\mathrm{GL}_{d_{m-1}}$ symmetry.
For $g_m \in G_m = \mathrm{GL}_{d_{m-1}}(\R)$ the following group action keeps $h_p$ with $p\geq m$ invariant 
\begin{align}
    g_m \cdot W_k = \left\{
    \begin{array}{lc}
        W_m g_m^{-1} & k = m \\
        \sigma^{-1}\pa{g_m \sigma\pa{W_{m-1} h_{m-2}} } h_{m-2}^{-1} & k = m-1 \\
        W_k & k \not\in \{m, m-1\}
    \end{array}
    \right.
\label{eq:Wm-g-action-main-appendix}
\end{align}
\end{manualproposition}

\begin{proof}
From \eqref{eq:g-act-W-h-0}, we want to convert $g\cdot h_{m-1}$ into a transformation on $W_{m-1}$ instead of $h_{m-1}$. 
In other words, we want to find a set of transformed weights $W'_m,W'_{m-1}$ which yields the same network output $\tilde{h}_m$: 
\begin{align}
    \tilde{h}_m&= W'_m \sigma\pa{W'_{m-1} h_{m-2}} = W_m g^{-1} g \sigma\pa{W_{m-1} h_{m-2}}  \cr 
    \Rightarrow W'_m &= W_m g^{-1}, \qquad  \sigma\pa{W'_{m-1} h_{m-2}} = g \sigma\pa{W_{m-1} h_{m-2}} 
    \label{eq:Wm-g-derive-0}
\end{align}
Solving \eqref{eq:Wm-g-derive-0} 
we get
\begin{align}
    W'_{m-1} &=
    \sigma^{-1}\pa{g \sigma\pa{W_{m-1} h_{m-2}} } h_{m-2}^{-1}.
    \label{eq:W-act-nonlin}
\end{align}

\eqref{eq:Wm-g-action-main-appendix} follows from \eqref{eq:Wm-g-derive-0} and \eqref{eq:W-act-nonlin}. 

To verify that \eqref{eq:Wm-g-action-main-appendix} is a valid group action,
\begin{align}
    I \cdot W_k &= \left\{
    \begin{array}{lc}
        W_m I & k = m \\
        \sigma^{-1}\pa{I \sigma\pa{W_{m-1} h_{m-2}} } h_{m-2}^{-1} & k = m-1 \\
        W_k & k \not\in \{m, m-1\}
    \end{array}
    \right. \cr
    &= W_k
\end{align}
and
\begin{align}
    g_1 \cdot (g_2 \cdot W_k) &= \left\{
    \begin{array}{lc}
        W_m g_2^{-1} g_1^{-1} & k = m \\
        \sigma^{-1}\pa{g_1 \sigma\pa{
        \br{\sigma^{-1}\pa{g_2 \sigma\pa{W_{m-1} h_{m-2}} } h_{m-2}^{-1}}
        h_{m-2}} } h_{m-2}^{-1} & k = m-1 \\
        W_k & k \not\in \{m, m-1\}
    \end{array}
    \right. \cr
    &= \left\{
    \begin{array}{lc}
        W_m (g_1 g_2)^{-1} & k = m \\
        \sigma^{-1}\pa{(g_1 g_2) \sigma\pa{W_{m-1} h_{m-2}} } h_{m-2}^{-1} & k = m-1 \\
        W_k & k \not\in \{m, m-1\}
    \end{array}
    \right. \cr
    &= (g_1 g_2) \cdot W_k
\end{align}
\end{proof}

\section{Theoretical analysis of teleportation}
\subsection{What symmetries help accelerate optimization}
\label{appendix:speedup-cond-prop}

\begin{manualproposition}{\ref{prop:speedup-cond}}
Let $\vw'=g\cdot \vw$ be a point we teleport to. 
Let $J= \ro \vw' /\ro \vw $ be the Jacobian. 
Symmetry teleportation using $g$ accelerates the rate of decay in $\L$ if it satisfies
\begin{align*}
    \left\|\br{J^{-1}}^T\grad \L(\vw)\right\|^2_\eta > \left\|\grad \L(\vw)\right\|^2_\eta.
\end{align*}
\end{manualproposition}
\begin{proof}
 Let $\vw' = g \cdot \vw$. Denote the Jacobian as $J$, where $J_{ij} = \partial w_i' / \partial w_j$. Then the inverse of $J$ has entries $J^{-1}_{ij} = \partial w_i / \partial w_j'$. 

The gradient at $\vw'$ is
\begin{align}
    \frac{\ro \L(\vw')}{\ro \vw'} 
    &= \frac{\ro \L(\vw)}{\ro \vw'} 
    = \sum_j \frac{\ro \L(\vw)}{\ro \vw_j} \frac{\ro \vw_j}{\ro \vw_i'} 
    = \sum_j \frac{\ro \L(\vw)}{\ro \vw_j} J^{-1}_{ji} 
    = \left(\left(\frac{\ro \L(\vw)}{\ro \vw}\right)^T J^{-1} \right)^T
    = (J^{-1})^T \frac{\ro \L(\vw)}{\ro \vw}
\end{align}
The rate of change of $\L$ in gradient flow is
\begin{align}
    {d\L(\vw')\over dt}& = \bk{{\ro \L \over \ro \vw' },{d\vw' \over dt} } = - \left\| (J^{-1})^T \frac{\ro \L(\vw)}{\ro \vw} \right\|_\eta^2 
\end{align}

Thus we will have a speedup if 
\begin{align}
    \left\| (J^{-1})^T \frac{\ro \L(\vw)}{\ro \vw} \right\|_\eta^2  > \left\| \frac{\ro \L(\vw)}{\ro \vw} \right\|_\eta^2
    \label{eq:speedup-cond}
\end{align}
\end{proof}

\textbf{Proof of Corollary \ref{prop:corollary-cond}}
\begin{proof}
    Since $J = g$, 
    using \eqref{eq:dLdt-0} and the l.h.s. of \eqref{eq:dL-w-cond-0} we have 
    \begin{align}
        {d\L(g\cdot \vw)\over dt }& 
        = -\grad \L^T g^{-1}\eta \br{g^{-1}}^T \grad \L = \grad \L^T \br{g^T \eta^{-1} g }^{-1} \grad \L  = \|\grad \L\|^2_\eta 
    \end{align}
\end{proof}


\subsection{Improvement of subsequent steps }
\label{appendix:lipschitz}

\begin{manualproposition}{\ref{prop:lipschitz}}
Consider the gradient descent with a $G$-invariant loss $\L(\vw)$ and learning rate $\eta \in \R^+$. Let $\vw_t$ be the parameter at time $t$ and $\vw_t' = g \cdot \vw_t$ the parameter teleported by $g\in G$. Let $\vw_{t+T}$ and $\vw_{t+T}'$ be the parameters after $T$ more steps of gradient descent from $\vw_t$ and $\vw_t'$ respectively. Under Assumption \ref{thm:lipschitz}, if $\eta L < 1$, and 
\begin{align*}
    \frac{\left\| \partial \L / \partial \vw_{t}' \right\|_2}{\left\| \partial \L / \partial \vw_{t} \right\|_2} \geq \frac{(1 + \eta L)^T}{(1 - \eta L)^T},
\end{align*}
then 
\begin{align*}
    \left\| \frac{\partial \L}{\partial \vw_{t+T}'} \right\|_2 \geq \left\| \frac{\partial \L}{\partial \vw_{t+T}} \right\|_2.
\end{align*}
\end{manualproposition}
\begin{proof}
From the definition of Lipschitz continuity and the update rule of gradient descent, 
\begin{align}
    \left| \left\| \frac{\partial \L}{\partial \vw_{t+1}} \right\|_2 - \left\| \frac{\partial \L}{\partial \vw_{t}} \right\|_2 \right| \leq L \| \vw_{t+1} - \vw_t \|_2 = L\left\| \eta \frac{\partial \L}{\partial \vw_{t}} \right\|_2
\end{align}
Equivalently,
\begin{align}
    (1 - \eta L) \left\| \frac{\partial \L}{\partial \vw_{t}} \right\|_2 \leq \left\| \frac{\partial \L}{\partial \vw_{t+1}} \right\|_2 \leq (1 + \eta L) \left\| \frac{\partial \L}{\partial \vw_{t}} \right\|_2
\end{align}
By unrolling $T$ steps, we have
\begin{align}
    (1 - \eta L)^T \left\| \frac{\partial \L}{\partial \vw_{t}} \right\|_2 
    \leq \left\| \frac{\partial \L}{\partial \vw_{t+T}} \right\|_2 
    \leq (1 + \eta L)^T \left\| \frac{\partial \L}{\partial \vw_{t}} \right\|_2
\end{align}
Similarly, for a teleported point $\vw_t' = g \cdot \vw_t$,
\begin{align}
    (1 - \eta L)^T \left\| \frac{\partial \L}{\partial \vw_{t}'} \right\|_2 
    \leq \left\| \frac{\partial \L}{\partial \vw_{t+T}'} \right\|_2 
    \leq (1 + \eta L)^T \left\| \frac{\partial \L}{\partial \vw_{t}'} \right\|_2
\end{align}
Therefore, if
\begin{align}
     (1 - \eta L)^T \left\| \frac{\partial \L}{\partial \vw_{t}'} \right\|_2 
     \geq (1 + \eta L)^T \left\| \frac{\partial \L}{\partial \vw_{t}} \right\|_2
\end{align}
then it is guaranteed that 
\begin{align}
    \left\| \frac{\partial \L}{\partial \vw_{t+T}'} \right\|_2 \geq \left\| \frac{\partial \L}{\partial \vw_{t+T}} \right\|_2
\end{align}
\end{proof}

\subsection{Convergence analysis for convex quadratic functions}
\label{appendix:convex-quadratic}
We first show that starting from a point in $S_c$, all other points in $S_c$ can be reached with one teleportation. 
\begin{proposition}
$S_c$ contains a single orbit. That is, $G \cdot \vw \equiv \{g \cdot \vw: g \in G\} = S_c$~ for all $\vw \in S_c$.
\end{proposition}
\begin{proof}
Consider two points $\vw_1, \vw_2 \in S_c$. Then $\vw_1^T A \vw_1 = (A^{\frac{1}{2}} \vw_1)^T (A^{\frac{1}{2}} \vw_1) = c$ and $\vw_2^T A \vw_2 = (A^{\frac{1}{2}} \vw_2)^T (A^{\frac{1}{2}} \vw_2) = c$. Let $\vv_1 = \frac{A^{\frac{1}{2}} \vw_1}{\sqrt{c}}$, $\vv_2 = \frac{A^{\frac{1}{2}} \vw_2}{\sqrt{c}}$ and $\mathbf{e}_1 = [1, 0, ..., 0]^T$. Since $\|v_1\| = \|v_2\| = 1$, there exists $g_1, g_2 \in O(n)$, such that $g_1 \mathbf{e}_1 = \vv_1$ and $g_2 \mathbf{e}_1 = \vv_2$. One way to construct such $g_1$ is let the first column be equal to $\vv_1$ and other columns be the rest of the orthonormal basis. Let $g = g_2 g_1^{-1}$. Then $\vv_2 = g \vv_1$, $A^{\frac{1}{2}} \vw_2 = g A^{\frac{1}{2}} \vw_1$, and $\vw_2 = A^{-\frac{1}{2}} g A^{\frac{1}{2}} \vw_1 = g \cdot \vw_1$. 

We have shown that for any $\vw_1, \vw_2 \in S_c$, there exists a $g \in G$ such that $\vw_2 = g \cdot \vw_1$. Therefore, the group action of $G$ on $S_c$ is transitive. Equivalently, $S_c$ contains a single orbit. 
\end{proof}
The objective of teleportation is transforming parameters using a group element to maximize the norm of gradient:
\begin{align}
\label{eq:teleportation-obj}
    \max_{g \in G} \|\grad \L_A |_{g \cdot \vw}\|^2_2. 
\end{align}
Since all points on the level set are reachable, the target teleportation destination is the point with the largest gradient norm on the same level set. In other words, \eqref{eq:teleportation-obj} is equivalent to the following optimization problem:
\begin{align}
\label{eq:teleportation-obj2}
    \max_{\vw'} \|\grad \L_A |_{\vw'}\|^2_2 \cr
    \text{s.t. } \L_A(\vw') = \L_A(\vw).
\end{align}
Let $c = \L_A(\vw)$. Substitute in $\L_A$ and $\grad \L_A$, we have the following equivalent formulation:
\begin{align}
\label{eq:teleportation-obj3}
    \max_{\vw'} \|A\vw'\|^2_2 \cr
    \text{s.t. } \vw'^T A \vw' = c.
\end{align}
Next, we solve this optimization problem and show that the gradient norm is maximized on the gradient flow trajectory starting from its solution.
\begin{proposition}
\label{prop:teleportation-sol-quadratic}
The solution to \eqref{eq:teleportation-obj3} is an eigenvector of $A$ corresponding to its largest eigenvalue. 
\end{proposition}
\begin{proof}
We solve \eqref{eq:teleportation-obj3} using the method of Lagrangian multipliers. The Lagrangian of \eqref{eq:teleportation-obj3} is
\begin{align}
    \mathscr{L} = \vw'^T A^T A \vw' - \lambda(\vw'^T A \vw' - c).
\end{align}
Setting the derivative with respect of $\vw'$ to 0, we have
\begin{align}
    \frac{\ro \mathscr{L}}{\ro \vw} = 2 A^T A \vw' - 2\lambda A \vw' = 0,
\end{align}
which gives
\begin{align}
    A^T A \vw' = \lambda A \vw'.
\end{align}
Since $A$ is positive definite, $A = A^T$ and $A$ is invertible. Therefore,
\begin{align}
    A \vw' = \lambda \vw',
\end{align}
so the solution to \eqref{eq:teleportation-obj3} is an eigenvector of $A$. Then, the constraint is $\vw'^T A \vw' = \lambda \|\vw'\|^2 = c$, and the objective becomes $\max_{\vw'} \lambda^2 \|\vw'\|^2 = \max_{\vw'} c\lambda$. Therefore, we want $\lambda$ to be the largest eigenvalue of $A$. Hence the optimal $\vw'$ is an eigenvector of $A$ corresponding to its largest eigenvalue. 
\end{proof}

\begin{manualproposition}{\ref{prop:convex-quadratic-global-opt}}
If at point $\vw$, $\| \grad \L_A|_\vw \|^2$ is at a maximum in $S_{\L_A(\vw)}$, then for any point $\vw'$ on the gradient flow trajectory starting from $\vw$, $\| \grad \L_A|_{\vw'} \|^2$ is at a maximum in $S_{\L_A(\vw')}$.
\end{manualproposition}

\begin{proof}
From Proposition \ref{prop:teleportation-sol-quadratic}, $\vw$ is an eigenvector of $A$ corresponding to its largest eigenvalue. Then the gradient of $\L_A$ is $A\vw = \lambda \vw$. Therefore, on the gradient flow trajectory starting from $\vw$, every point has the same direction as $\vw$, meaning that the points are all eigenvectors of $A$ corresponding to its largest eigenvalue. Therefore, $\| \grad \L_A \|^2$ is always at a maximum on the loss level sets. 
\end{proof}

Finally, we show that maximizing the magnitude of gradient is equivalent to minimizing the distance to $\vw^*$ in a loss level set (Proposition \ref{prop:convex-quadratic-dist-minimum}).
\begin{proposition}
The solution to the following optimization problem is the same as the solution to \eqref{eq:teleportation-obj3}:
\begin{align}
\label{eq:teleportation-obj4}
    \min_{\vw'} \|\vw' - \vw^*\|^2_2 \cr
    \text{s.t. } \vw'^T A \vw' = c.
\end{align}
\end{proposition}
\begin{proof}
Similar to Proposition \ref{prop:teleportation-sol-quadratic}, we solve this optimization problem using the method of Lagrangian multipliers. Note that $\vw^* = 0$. The Lagrangian is
\begin{align}
    \mathscr{L} = \vw'^T \vw' - \lambda(\vw'^T A \vw' - c).
\end{align}
Setting the derivative with respect of $\vw'$ to 0, we have
\begin{align}
    \frac{\ro \mathscr{L}}{\ro \vw} = 2 \vw' - 2\lambda A \vw' = 0,
\end{align}
which gives
\begin{align}
    A \vw' = \lambda \vw',
\end{align}
so the solution to \eqref{eq:teleportation-obj4} is an eigenvector of $A$. Then, the constraint is $\vw'^T A \vw' = \lambda \|\vw'\|^2 = c$, and the objective becomes $\min_{\vw'} \|\vw'\|^2 = \min_{\vw'} \frac{c}{\lambda}$. Therefore, we want $\lambda$ to be the largest eigenvalue of $A$. Hence the optimal $\vw'$ is an eigenvector of $A$ corresponding to its largest eigenvalue, which is the same as the solution to \eqref{eq:teleportation-obj3}. 
\end{proof}

For a more concrete example, consider a diagonal matrix $A$ with positive diagonal elements. Then the level sets of $\L_A$ are $n$-dimensional ellipsoids centered at the origin 0, with axes in the same directions as the standard basis.
The point with largest $\|\nabla \L_A\|^2$ on a level set is in the eigendirection of $A$ corresponding to its largest eigenvalue, or equivalently, a point on the smallest semi-axes of the ellipsoid. 
Note that this point has the smallest distance to the global minimum $\vw^* = 0$ among all points in the same level set.
In addition, the gradient flow trajectory from this point always points to $\vw^*$.
Therefore, like the 2D ellipse function, one teleportation on the $n$-dimensional ellipsoid also guarantees optimal gradient norm at all points on the trajectory. 

\subsection{Relation to second-order optimization methods}
\label{appendix:newton-direction-prop}

To prove Proposition \ref{prop:newton-direction}, we first note that when the norm of the gradient is at a critical point on the level set of the loss function, the gradient is an eigenvector of the Hession. 

\begin{lemma}
If $\ro_\vv \left\|\frac{\ro \L}{\ro \vw}\right\|_2^2 = 0$ for all unit vector $\vv$ that is orthogonal to $\frac{\ro \L}{\ro \vw}$, then $\frac{\ro \L}{\ro \vw}$ is an eigenvector of the Hessian of $\L$.
\label{lemma:newton-direction}
\end{lemma}

\begin{proof}
From the definition of the directional derivative,
\begin{align}
    \ro_\vv \left\|\frac{\ro \L}{\ro \vw}\right\|_2^2 
    &= \vv \cdot \frac{\ro}{\ro \vw} \left\|\frac{\ro \L}{\ro \vw}\right\|_2^2
\end{align}
Writing the last term in indices, 
\begin{align}
    \frac{\ro}{\ro \vw_i} \left\|\frac{\ro \L}{\ro \vw}\right\|_2^2
    &= \frac{\ro}{\ro \vw_i} \sum_j \left(\frac{\ro \L}{\ro \vw_j}\right)^2 \cr
    &= \sum_j \frac{\ro}{\ro \vw_i} \left(\frac{\ro \L}{\ro \vw_j}\right)^2 \cr
    &= \sum_j 2 \frac{\ro \L}{\ro \vw_j} \frac{\ro^2 \L}{\ro \vw_i \ro \vw_j} \cr
    &= 2 \left(H \frac{\ro \L}{\ro \vw}\right)_i
\end{align}
Removing the indices,
\begin{align}
    \frac{\ro}{\ro \vw} \left\|\frac{\ro \L}{\ro \vw}\right\|_2^2
    &= 2 H \frac{\ro \L}{\ro \vw}
\end{align}
Substitute back and we have
\begin{align}
    \ro_\vv \left\|\frac{\ro \L}{\ro \vw}\right\|_2^2 
    &= \vv \cdot \left(2 H \frac{\ro \L}{\ro \vw}\right)
\end{align}
Since $\ro_\vv \left\|\frac{\ro \L}{\ro \vw}\right\|_2^2 = 0$ for all vector $\vv$ that is orthogonal to $\frac{\ro \L}{\ro \vw}$, $\vv \cdot \left(2 H \frac{\ro \L}{\ro \vw}\right) = 0$ for all vector $\vv$ that is orthogonal to $\frac{\ro \L}{\ro \vw}$.
In other words, $2 H \frac{\ro \L}{\ro \vw}$ is orthogonal to all vectors that are orthogonal to $\frac{\ro \L}{\ro \vw}$. Therefore, $2 H \frac{\ro \L}{\ro \vw}$ has the same direction of $\frac{\ro \L}{\ro \vw}$, and $\frac{\ro \L}{\ro \vw}$ is an eigenvector of the Hessian of $\L$.
\end{proof}

Proposition \ref{prop:newton-direction} is a direct consequence of Lemma \ref{lemma:newton-direction}.
\begin{manualproposition}{\ref{prop:newton-direction}}
Let $S_{\L_0} = \{\vw: \L(\vw) = \L_0\}$ be a level set of $\L$. 
If at a particular $\vw \in S_{\L_0}$ we have $ \|\grad\L(\vw)\|_2\geq  \|\grad\L(\vw')\|_2 $  for all $\vw'$ in a small neighborhood of $\vw$ in $S_{\L_0}$, then the gradient $\grad \L(\vw)$ has the same direction as the Newton's direction $H^{-1} \grad\L(\vw) $.
\end{manualproposition}

\begin{proof}
From Lemma \ref{lemma:newton-direction}, $\frac{dL}{dw}$ is an eigenvector of $H$. Therefore, it is also an eigenvector of $H^{-1}$. Hence $\frac{dL}{dw}$ has the same direction as $H^{-1} \frac{dL}{dw}$.
\end{proof}

\section{Experiment details and additional results}
\subsection{Test functions}
\label{appendix:exp-test-function}

We compare the gradient at different loss values for gradient descent with and without teleportation. Figure \ref{fig:test_functions_supp} shows that the trajectory with teleportation has a larger $d\L/dt$ value than the trajectory without teleportation at the same loss values. Therefore, the rate of change in the loss is larger in the trajectory with teleportation, which makes it favorable. 

\begin{figure}[t!]
\centering
\hspace{80pt} a \hspace{115pt} b \hfill ~ \\
\includegraphics[width=0.32\textwidth]{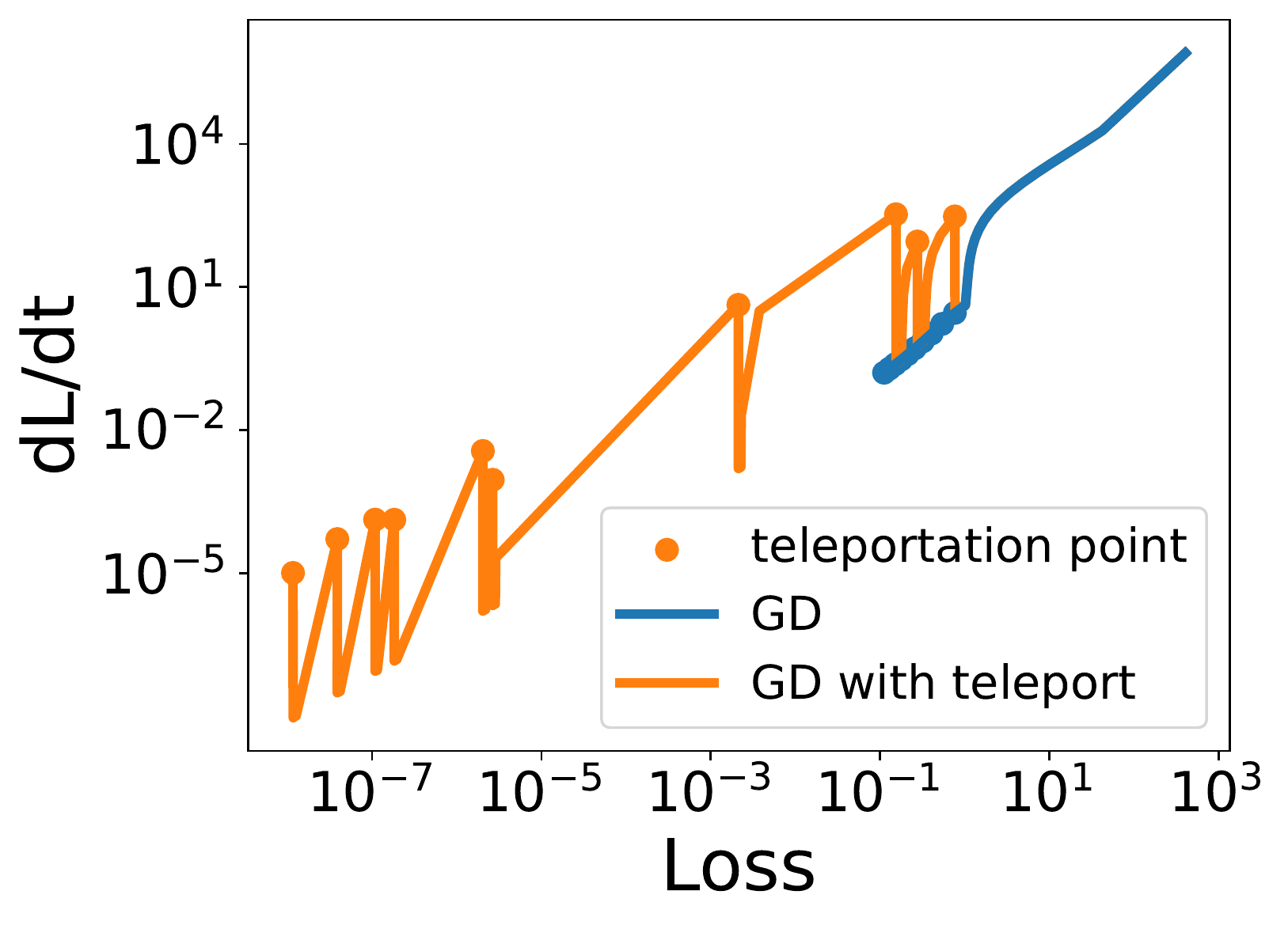}
\includegraphics[width=0.32\textwidth]{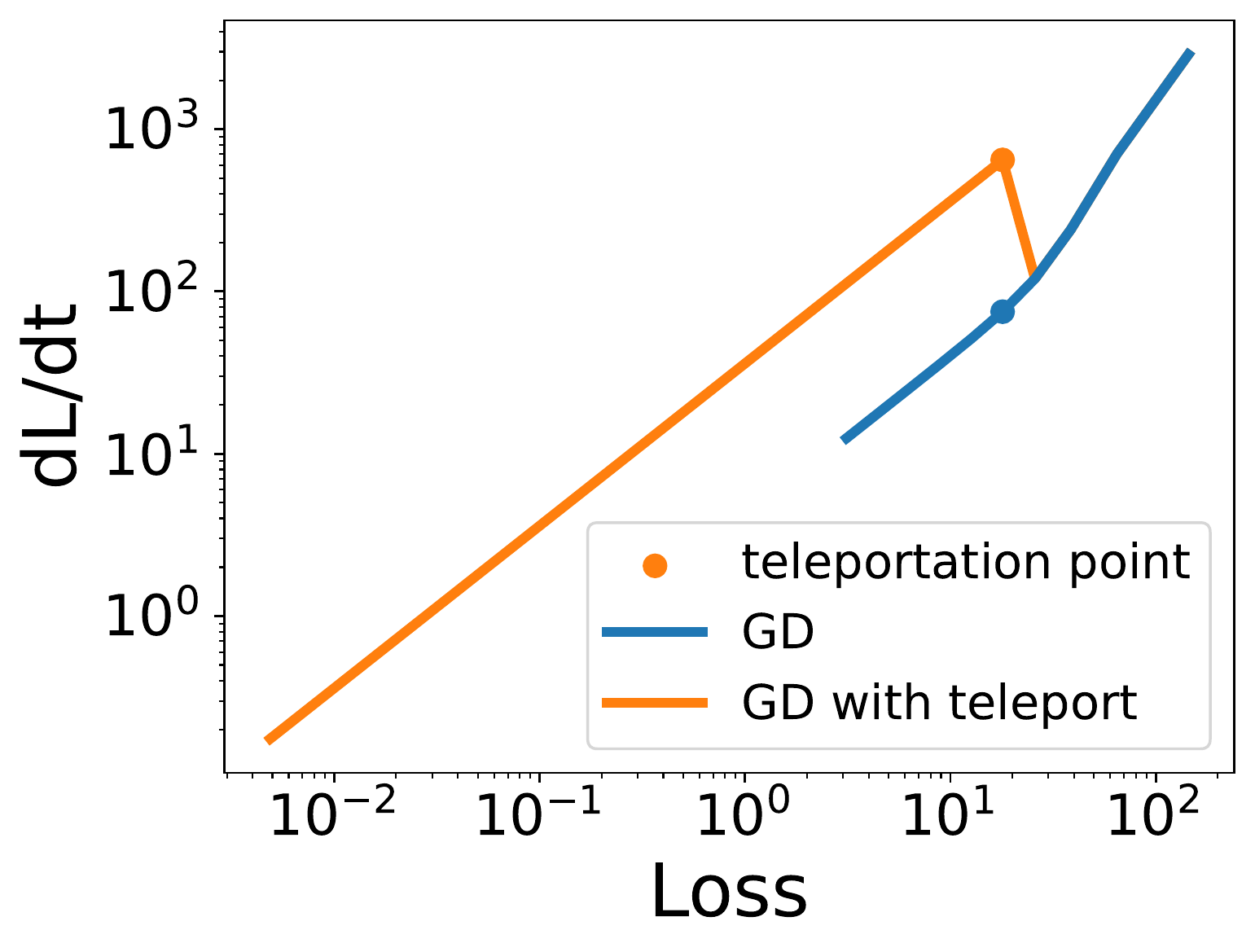}
\caption{Gradient on the trajectory of optimizing the Rosenbrock function (left) and Booth function (right). At the same loss value, the graident is larger on the trajectory with teleportation, indicating a better descent path.}
\label{fig:test_functions_supp}
\end{figure}

\subsection{Multilayer neural network}
\label{appendix:exp-multilayer}
\paragraph{Additional training details} Data $X, Y$ and initialization of parameters $W$ are set uniformly at random over $[0, 1]$. 
GD uses learning rate $10^{-4}$ and AdaGrad uses $10^{-1}$. 
Each algorithm is run 300 steps.
When using teleportation, we perform symmetry transform on the parameters once at epoch 5.
In GD, the group elements used for these transforms are found by gradient ascent on $T$ for 8 steps, with learning rate $10^{-7}$. In AdaGrad, the group elements are found by gradient ascent for 2 steps, with learning rate $10^{-5}$. The choice of hyperparameters comes from a grid search described in the next section.


\paragraph{Hyperparameter tuning} To observe the effect of hyperparameters on the speedup in computation time, we did a hyperparameter sweep on the number of steps and the learning rate used in each teleportation. The speedup of teleportation on SGD and AdaGrad is defined by $t_{sgd} / t_{sgd+teleport}$ and $t_{adagrad} / t_{adagrad+teleport}$ respectively, where $t_{sgd},t_{adagrad}$ are the wall-clock time required to reach convergence using SGD or AdaGrad, and $t_{sgd+teleport},t_{adagrad+teleport}$ are convergence time with teleportation. We consider the optimization algorithm converged if the difference between the loss of two consecutive steps is less than $10^{-3}$. This experiment is run on one CPU.

Figure \ref{fig:two-layer-sweep} shows the speedup of the same multilayer neural network regression problem defined in Section \ref{sec:experiment}, but teleporting only once at epoch 5. We did a grid search for teleportation learning rates in $[10^{-9}, 10^{-8}, 10^{-7}, 10^{-6}, 10^{-5}]$ and number of teleportation steps in $[1, 2, 4, 8, 16, 32]$. Omitted points in the figure indicate that the gradient descent fails to converge within 2000 steps or diverges.

When converged, most hyperparameter combinations improve the convergence speed in wall-clock time (speedup $> 1$). There are trade-offs in both the number of steps used for each teleportation and the teleportation learning rate. Increasing the number of steps used to optimize teleportation target allows us to find a better point in the parameter space but increases the cost of one teleportation. Increasing the learning rate of optimizing the group element improves $\|\ro \L / \ro \vw\|$ but is more likely to lead to divergence since $\|\ro \L / \ro \vw\|$ can become too large for the gradient descent learning rate. 



\begin{figure}[t!]
\centering
\ \ \ a\hfill b \hfill c\hfill d \hfill ~ \\
\includegraphics[width=0.24\textwidth]{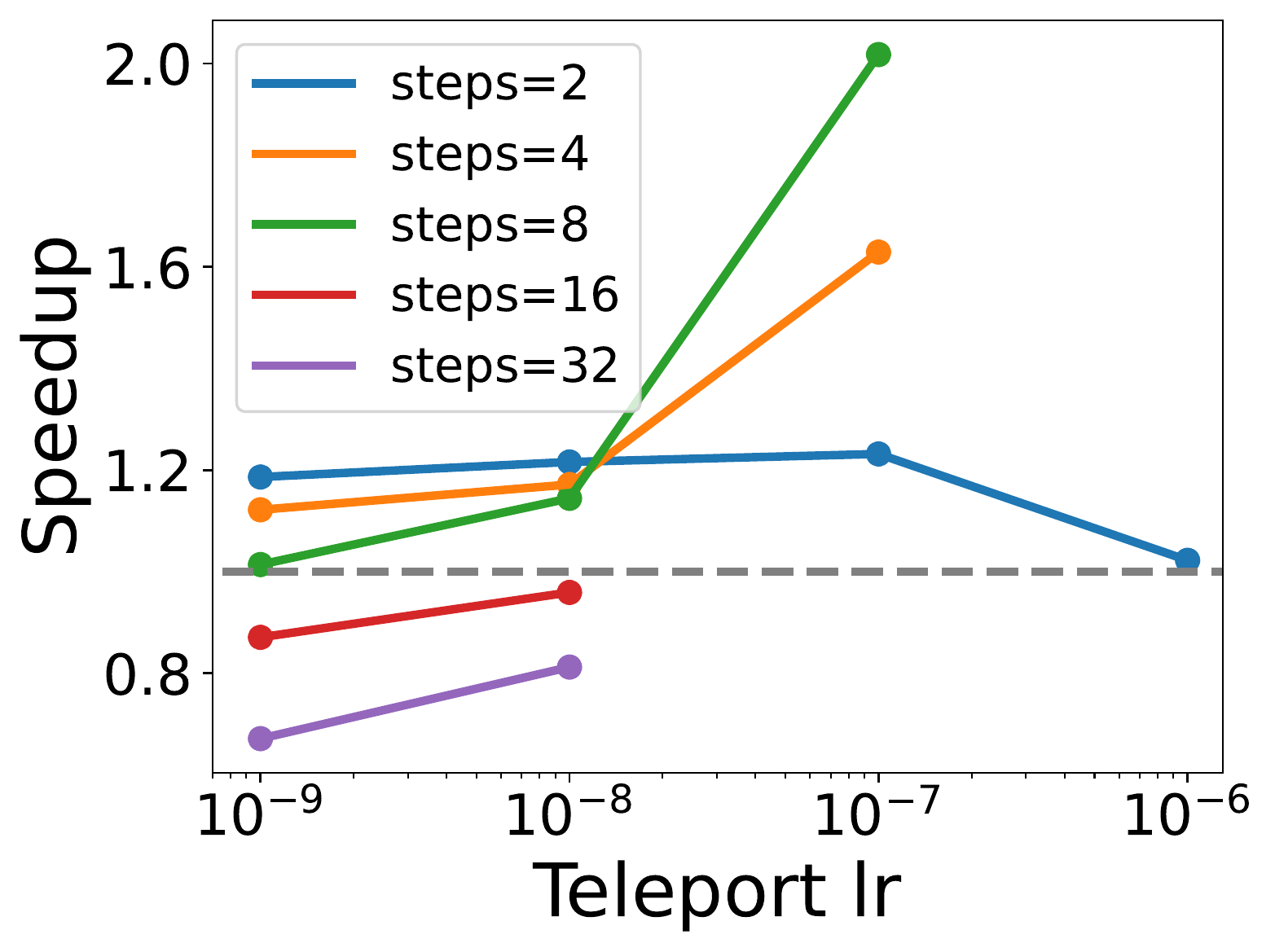}
\includegraphics[width=0.24\textwidth]{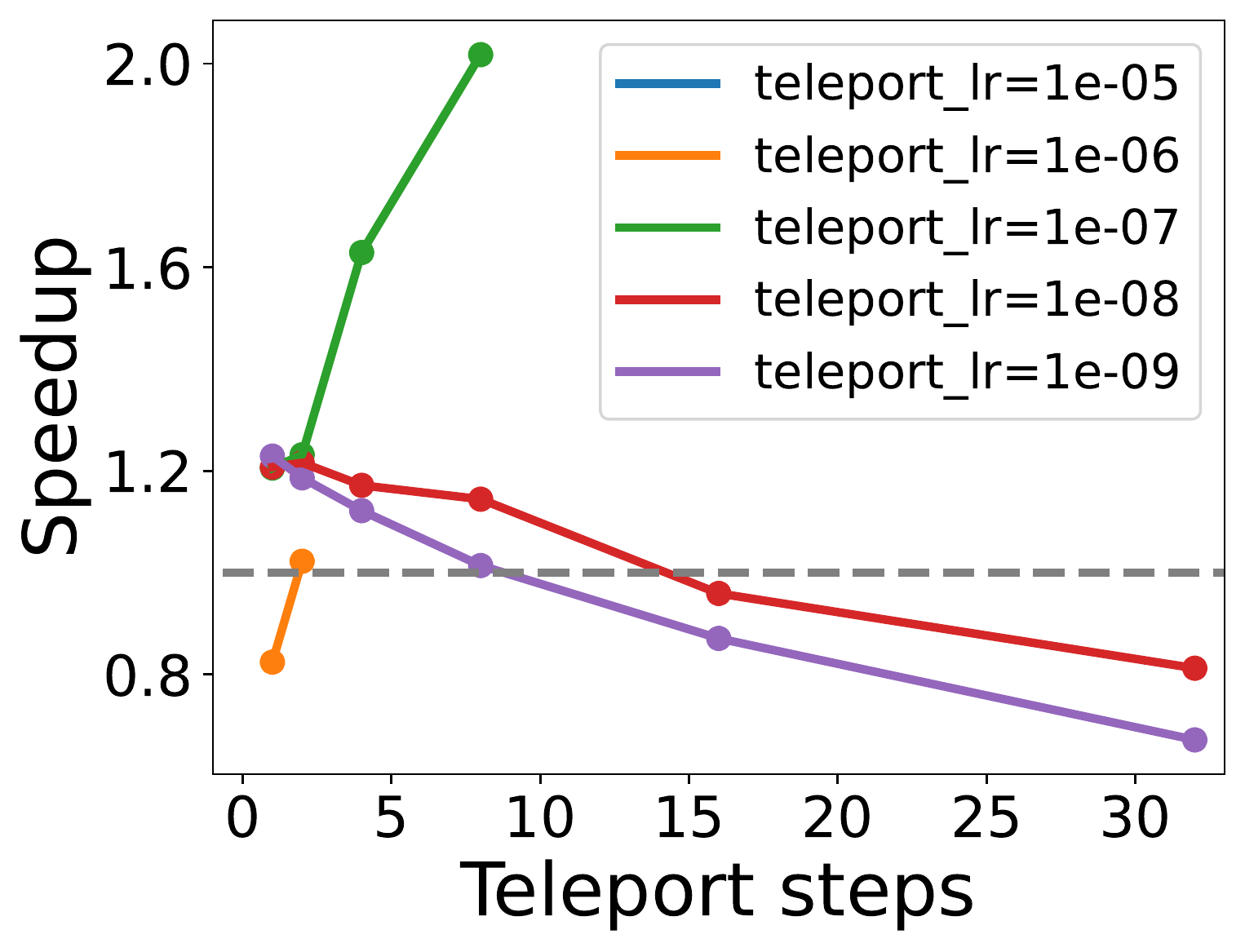}
\includegraphics[width=0.24\textwidth]{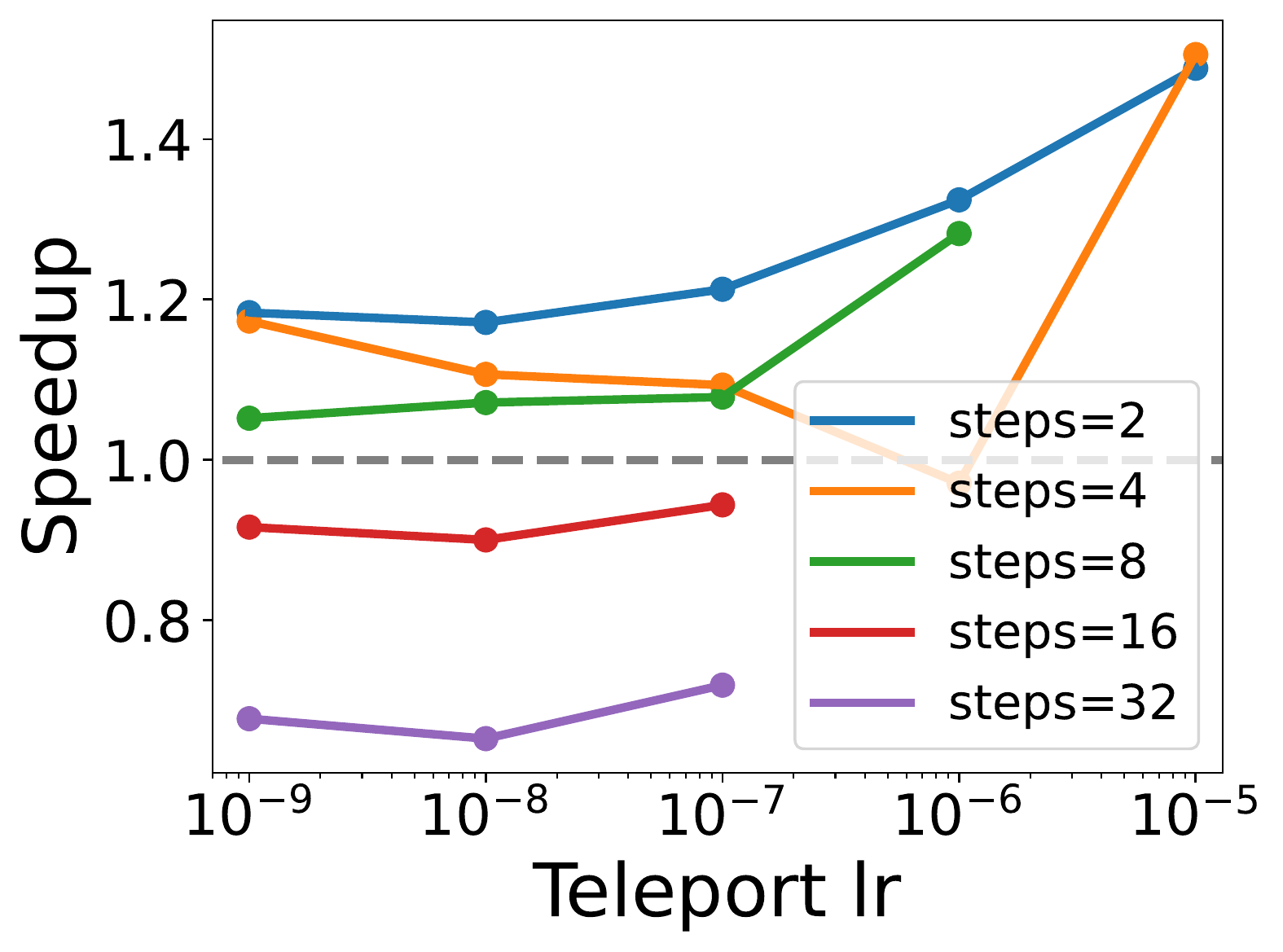}
\includegraphics[width=0.24\textwidth]{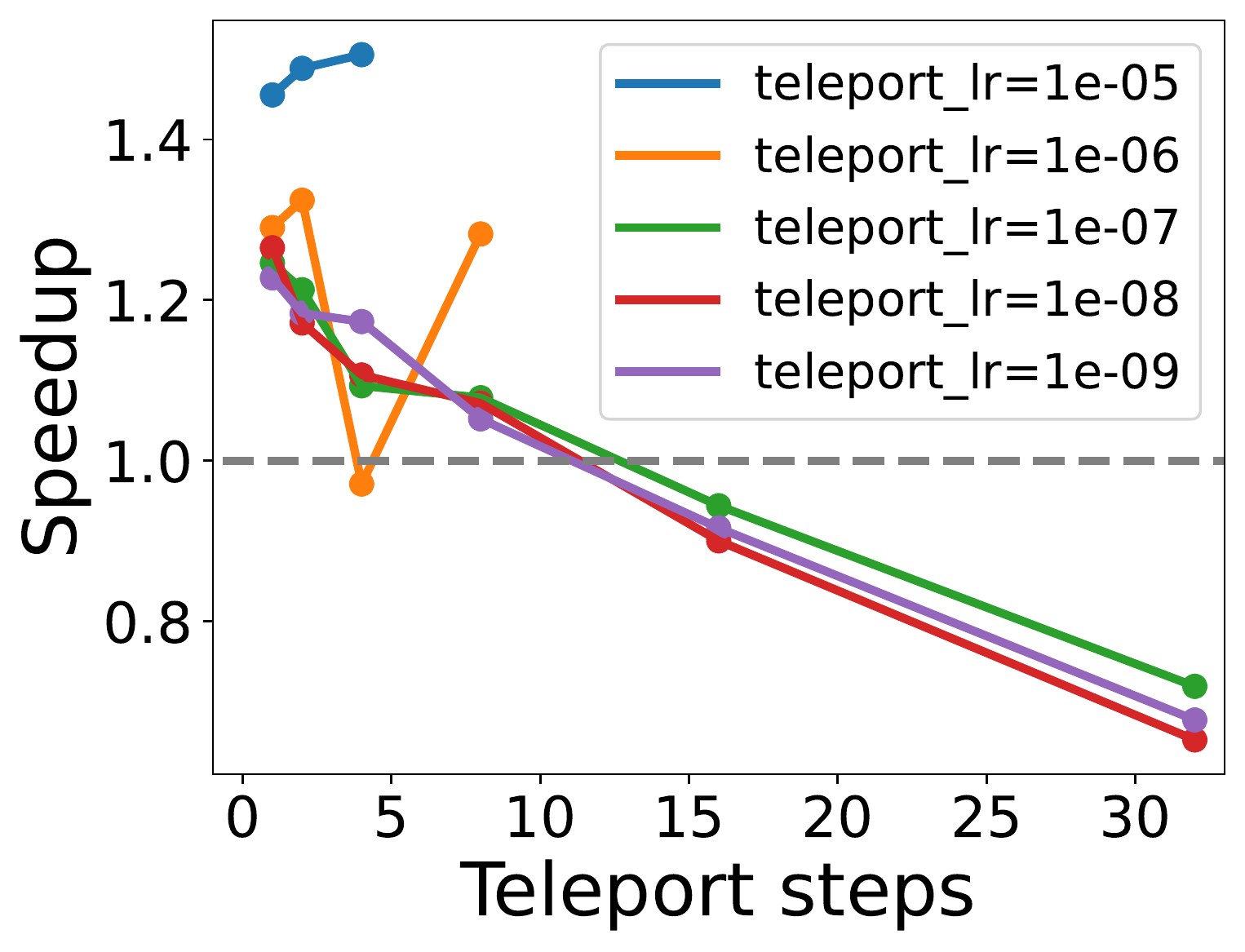}
\caption{Hyperparameter sweeps of the number of steps and the learning rate used to find the optimal group element in teleportation. The wall-clock speedup of applying teleportation is shown separately for gradient descent (a)(b) and AdaGrad (c)(d). The dashed line represents speedup $=1$.}
\label{fig:two-layer-sweep}
\end{figure}

%% file: neurips_2022.bbl
\begin{thebibliography}{29}
\providecommand{\natexlab}[1]{#1}
\providecommand{\url}[1]{\texttt{#1}}
\expandafter\ifx\csname urlstyle\endcsname\relax
  \providecommand{\doi}[1]{doi: #1}\else
  \providecommand{\doi}{doi: \begingroup \urlstyle{rm}\Url}\fi

\bibitem[Amari(1998)]{amari1998natural}
Shun-Ichi Amari.
\newblock Natural gradient works efficiently in learning.
\newblock \emph{Neural computation}, 10\penalty0 (2):\penalty0 251--276, 1998.

\bibitem[Armenta and Jodoin(2021)]{armenta2021representation}
Marco Armenta and Pierre-Marc Jodoin.
\newblock The representation theory of neural networks.
\newblock \emph{Mathematics}, 9\penalty0 (24), 2021.
\newblock ISSN 2227-7390.

\bibitem[Armenta et~al.(2020)Armenta, Judge, Painchaud, Skandarani, Lemaire,
  Sanchez, Spino, and Jodoin]{armenta2020neural}
Marco Armenta, Thierry Judge, Nathan Painchaud, Youssef Skandarani, Carl
  Lemaire, Gabriel~Gibeau Sanchez, Philippe Spino, and Pierre-Marc Jodoin.
\newblock Neural teleportation.
\newblock \emph{arXiv preprint arXiv:2012.01118}, 2020.

\bibitem[Arora et~al.(2018)Arora, Cohen, and Hazan]{arora2018optimization}
Sanjeev Arora, Nadav Cohen, and Elad Hazan.
\newblock On the optimization of deep networks: Implicit acceleration by
  overparameterization.
\newblock In \emph{International Conference on Machine Learning}, pages
  244--253. PMLR, 2018.

\bibitem[Back(1996)]{back1996evolutionary}
Thomas Back.
\newblock \emph{Evolutionary algorithms in theory and practice: evolution
  strategies, evolutionary programming, genetic algorithms}.
\newblock Oxford university press, 1996.

\bibitem[Badrinarayanan et~al.(2015)Badrinarayanan, Mishra, and
  Cipolla]{badrinarayanan2015symmetry}
Vijay Badrinarayanan, Bamdev Mishra, and Roberto Cipolla.
\newblock Symmetry-invariant optimization in deep networks.
\newblock \emph{arXiv preprint arXiv:1511.01754}, 2015.

\bibitem[Bamler and Mandt(2018)]{bamler2018improving}
Robert Bamler and Stephan Mandt.
\newblock Improving optimization for models with continuous symmetry breaking.
\newblock In \emph{International Conference on Machine Learning}, pages
  423--432. PMLR, 2018.

\bibitem[Belkin et~al.(2019)Belkin, Hsu, Ma, and Mandal]{belkin2019reconciling}
Mikhail Belkin, Daniel Hsu, Siyuan Ma, and Soumik Mandal.
\newblock Reconciling modern machine-learning practice and the classical
  bias--variance trade-off.
\newblock \emph{Proceedings of the National Academy of Sciences}, 116\penalty0
  (32):\penalty0 15849--15854, 2019.

\bibitem[Combettes and Pesquet(2011)]{combettes2011proximal}
Patrick~L Combettes and Jean-Christophe Pesquet.
\newblock Proximal splitting methods in signal processing.
\newblock In \emph{Fixed-point algorithms for inverse problems in science and
  engineering}, pages 185--212. Springer, 2011.

\bibitem[Deng(2012)]{deng2012mnist}
Li~Deng.
\newblock The {MNIST} database of handwritten digit images for machine learning
  research.
\newblock \emph{IEEE Signal Processing Magazine}, 29\penalty0 (6):\penalty0
  141--142, 2012.

\bibitem[Du et~al.(2018)Du, Hu, and Lee]{du2018algorithmic}
Simon~S Du, Wei Hu, and Jason~D Lee.
\newblock Algorithmic regularization in learning deep homogeneous models:
  Layers are automatically balanced.
\newblock \emph{Neural Information Processing Systems}, 2018.

\bibitem[Duchi et~al.(2011)Duchi, Hazan, and Singer]{duchi2011adaptive}
John Duchi, Elad Hazan, and Yoram Singer.
\newblock Adaptive subgradient methods for online learning and stochastic
  optimization.
\newblock \emph{Journal of machine learning research}, 12\penalty0 (7), 2011.

\bibitem[Ganev et~al.(2021)Ganev, van Laarhoven, and Walters]{ganev2021qr}
Iordan Ganev, Twan van Laarhoven, and Robin Walters.
\newblock Universal approximation and model compression for radial neural
  networks.
\newblock \emph{arXiv preprint arXiv:2107.02550}, 2021.

\bibitem[G{\l}uch and Urbanke(2021)]{gluch2021noether}
Grzegorz G{\l}uch and R{\"u}diger Urbanke.
\newblock Noether: The more things change, the more stay the same.
\newblock \emph{arXiv preprint arXiv:2104.05508}, 2021.

\bibitem[Gupta et~al.(2018)Gupta, Koren, and Singer]{gupta2018shampoo}
Vineet Gupta, Tomer Koren, and Yoram Singer.
\newblock Shampoo: Preconditioned stochastic tensor optimization.
\newblock In \emph{35th International Conference on Machine Learning}, 2018.

\bibitem[Hazan(2019)]{hazan2019lecture}
Elad Hazan.
\newblock Lecture notes: Optimization for machine learning.
\newblock \emph{arXiv preprint arXiv:1909.03550}, 2019.

\bibitem[Jamil and Yang(2013)]{jamil2013literature}
Momin Jamil and Xin-She Yang.
\newblock A literature survey of benchmark functions for global optimisation
  problems.
\newblock \emph{International Journal of Mathematical Modelling and Numerical
  Optimisation}, 4\penalty0 (2):\penalty0 150--194, 2013.

\bibitem[Kingma and Ba(2015)]{kingma2015adam}
Diederik~P Kingma and Jimmy Ba.
\newblock Adam: A method for stochastic optimization.
\newblock \emph{International Conference on Learning Representations}, 2015.

\bibitem[Kunin et~al.(2021)Kunin, Sagastuy-Brena, Ganguli, Yamins, and
  Tanaka]{kunin2021neural}
Daniel Kunin, Javier Sagastuy-Brena, Surya Ganguli, Daniel~LK Yamins, and
  Hidenori Tanaka.
\newblock Neural mechanics: Symmetry and broken conservation laws in deep
  learning dynamics.
\newblock In \emph{International Conference on Learning Representations}, 2021.

\bibitem[Lang(2002)]{Lang}
Serge Lang.
\newblock \emph{Algebra}.
\newblock Graduate Texts in Mathematics. Springer, 2002.
\newblock ISBN 978-0387953854.

\bibitem[Martens and Grosse(2018)]{martens2018kfac}
James Martens and Roger~B Grosse.
\newblock Optimizing neural networks with kronecker-factored approximate
  curvature.
\newblock In \emph{International Conference on International Conference on
  Machine Learning}. PMLR, 2018.

\bibitem[Meng et~al.(2019)Meng, Zheng, Zhang, Chen, Ma, and
  Liu]{meng2019mathcal}
Qi~Meng, Shuxin Zheng, Huishuai Zhang, Wei Chen, Zhi-Ming Ma, and Tie-Yan Liu.
\newblock {$\mathcal{G}$-SGD}: Optimizing relu neural networks in its
  positively scale-invariant space.
\newblock \emph{International Conference on Learning Representations}, 2019.

\bibitem[Min et~al.(2021)Min, Tarmoun, Vidal, and Mallada]{min2021explicit}
Hancheng Min, Salma Tarmoun, René Vidal, and Enrique Mallada.
\newblock On the explicit role of initialization on the convergence and
  implicit bias of overparametrized linear networks.
\newblock In \emph{International Conference on Machine Learning}. PMLR, 2021.

\bibitem[Neyshabur et~al.(2015)Neyshabur, Salakhutdinov, and
  Srebro]{neyshabur2015path-sgd}
Behnam Neyshabur, Russ~R Salakhutdinov, and Nati Srebro.
\newblock {Path-SGD}: Path-normalized optimization in deep neural networks.
\newblock In \emph{Advances in Neural Information Processing Systems}, 2015.

\bibitem[Rosenbrock(1960)]{rosenbrock1960automatic}
HoHo Rosenbrock.
\newblock An automatic method for finding the greatest or least value of a
  function.
\newblock \emph{The Computer Journal}, 3\penalty0 (3):\penalty0 175--184, 1960.

\bibitem[Saxe et~al.(2014)Saxe, McClelland, and Ganguli]{saxe2014exact}
Andrew~M. Saxe, James~L. McClelland, and Surya Ganguli.
\newblock Exact solutions to the nonlinear dynamics of learning in deep linear
  neural networks.
\newblock \emph{International Conference on Learning Representations}, 2014.

\bibitem[{\c{S}}im{\c{s}}ek et~al.(2021){\c{S}}im{\c{s}}ek, Ged, Jacot,
  Spadaro, Hongler, Gerstner, and Brea]{simsek2021geometry}
Berfin {\c{S}}im{\c{s}}ek, Fran{\c{c}}ois Ged, Arthur Jacot, Francesco Spadaro,
  Cl{\'e}ment Hongler, Wulfram Gerstner, and Johanni Brea.
\newblock Geometry of the loss landscape in overparameterized neural networks:
  Symmetries and invariances.
\newblock In \emph{International Conference on Machine Learning}, pages
  9722--9732. PMLR, 2021.

\bibitem[Tarmoun et~al.(2021)Tarmoun, Franca, Haeffele, and
  Vidal]{tarmoun2021understanding}
Salma Tarmoun, Guilherme Franca, Benjamin~D Haeffele, and Rene Vidal.
\newblock Understanding the dynamics of gradient flow in overparameterized
  linear models.
\newblock In \emph{International Conference on Machine Learning}, pages
  10153--10161. PMLR, 2021.

\bibitem[Van~Laarhoven(2017)]{van2017l2}
Twan Van~Laarhoven.
\newblock L2 regularization versus batch and weight normalization.
\newblock \emph{Advances in Neural Information Processing Systems}, 2017.

\end{thebibliography}
